\documentclass[final]{alt2021} 

\pdfoutput=1

\newcommand{\defeq}{\vcentcolon=}
\SetKwComment{Comment}{$\triangleright$}{}
\newtheorem*{theorem*}{Theorem}
\newtheorem*{lemma*}{Lemma}

\DeclareMathOperator*{\argmin}{argmin}
\DeclareMathOperator*{\argmax}{argmax}

\newcommand{\simplex}{\Delta_{|\A|}}
\newcommand{\tsimplex}{\mathcal{S}_{\A}}
\newcommand{\diamtsimplex}{\text{diam}(\mathcal{S}_{\A})}

\newcommand{\A}{\mathcal{A}}
\newcommand{\I}{\mathcal{I}}
\newcommand{\M}{\mathcal{M}}
\newcommand{\F}{\mathcal{F}}
\newcommand{\C}{\mathcal{C}}
\newcommand{\T}{\mathcal{T}}
\newcommand{\E}{\mathbb{E}}
\newcommand{\R}{\mathbb{R}}
\newcommand{\N}{\mathbb{N}}
\newcommand{\1}{\bm{1}}

\usepackage{bm}
\usepackage{mathtools}
\usepackage{tikz}
\usepackage{pifont}
\newcommand{\xmark}{\ding{55}}%

\SetAlFnt{\small}
\SetAlCapFnt{\small}
\SetAlCapNameFnt{\small}

\usepackage{bookmark}

\title[Efficient Pure Exploration for Combinatorial Semi-Bandits]{Efficient Pure Exploration for Combinatorial Bandits with Semi-Bandit Feedback}
\usepackage{times}

\altauthor{\Name{Marc Jourdan} \Email{MARC.JOURDAN@INF.ETHZ.CH}\\
 \Name{Mojm\'ir Mutn\'y} \Email{MOJMIR.MUTNY@INF.ETHZ.CH}\\
 \Name{Johannes Kirschner} \Email{JKIRSCHNER@INF.ETHZ.CH}\\
 \Name{Andreas Krause} \Email{KRAUSEA@ETHZ.CH}\\
 \addr Department of Computer Science\\
ETH Z\"urich, Switzerland}

\begin{document}

\maketitle

\begin{abstract}
  Combinatorial bandits with semi-bandit feedback generalize multi-armed bandits, where the agent chooses sets of arms and observes a noisy reward for each arm contained in the chosen set. The action set satisfies a given structure such as forming a base of a matroid or a path in a graph.
  We focus on the pure-exploration problem of identifying the best arm with fixed confidence, as well as a more general setting, where the structure of the answer set differs from the one of the action set. Using the recently popularized game framework, we interpret this problem as a sequential zero-sum game and develop a CombGame meta-algorithm whose instances are asymptotically optimal algorithms with finite time guarantees. In addition to comparing two families of learners to instantiate our meta-algorithm, the main contribution of our work is a specific oracle efficient instance for best-arm identification with combinatorial actions. Based on a projection-free online learning algorithm for convex polytopes, it is the first computationally efficient algorithm which is asymptotically optimal and has competitive empirical performance.
\end{abstract}


\section{Introduction}

The multi-armed bandit (MAB) setting is an extensively studied problem in statistics and machine learning \citep{robbins_aspects_1952, lattimore_szepesvari_2020}. The environment consists of a set of arms, each characterized by an unknown reward distribution. An agent interacts with it by playing the arms sequentially in order to identify the arm with the highest expected reward.

Combinatorial bandits \citep{cesa-bianchi_combinatorial_2012, chen_combinatorial_2013} are a natural extension of the standard framework. The agent chooses actions (or super arms) which are defined by {\em sets of arms} satisfying certain constraints. The most studied families of actions stem from matroid theory \citep{kveton_matroid_2014, perrault_exploiting_2019}. Matroids encompass the batch setting where actions are sets of size $k$ \citep{jun_top_2016, kuroki_polynomial-time_2019, rejwan_top-k_2019} and graph-based structures where arms are edges and actions are spanning trees or matching trees. This formulation can model various application-specific structures such as paths taken in routing problems \citep{talebi_stochastic_2018}. Another example is protein design, where experimental constraints force the agent to evaluate specific sequences of proteins. Instead of inducing a single mutation, a range of localized mutations are performed at once. The main challenge in combinatorial bandits is to cope with the exponential size of the action set. This renders standard approaches for the bandit setting computationally inefficient and also -- without further assumptions like linearity -- statistically inefficient. To overcome this hurdle, existing approaches assume the reward is linear over the set of arms, and leverage an efficient oracle which solves a linear optimization problem over the combinatorial set of feasible actions. Efficient combinatorial oracles are known for many constraint families such as matroid polytopes, intersections of matroids and path polytopes. Combinatorial bandit strategies vary depending on the received feedback. We consider \textit{semi-bandit} feedback where the agent \emph{observes a reward for each selected arm}. Moreover, we assume that the reward for each arm is \emph{independent}.

We focus on the {\em pure-exploration} framework, in which the agent aims at maximizing the information gathered to answer a given query and disregards the accumulated cost. Two major theoretical frameworks exist \citep{gabillon_best_2012,gabillon_improved_2016, jun_top_2016,kaufmann_complexity_2016}: the \textit{fixed-budget} setting and the \textit{fixed-confidence} setting. In the fixed-budget setting, the goal is to minimize the probability of misidentifying the correct answer given a fixed number of pulls. We consider the fixed-confidence setting where the objective is to minimize the number of pulls necessary to identify the correct answer with a given confidence $1-\delta$. The most studied problems are best-arm identification (BAI) \citep{karnin_almost_2013, jamieson_lil_2013, zaki_explicit_2020} and top-$k$ identification \citep{gabillon_multi-bandit_2011, kalyanakrishnan_pac_2012, bubeck_multiple_2013, scarlett_overlapping_2019}. 

In the spirit of transductive bandits \citep{fiez_sequential_2019} we consider a more general setting where answers are sets of arms. The set of actions and the set of answers can be different. For example, in a routing or transportation network the objective might be to identify a weak link in order to fix it. The agent evaluates a path (action) in the network and gets access to time-stamped data for each link (answer) of a played path. Similarly, in protein design, researchers often generate many mutant proteins in one experiment, but the goal is to identify the best mutant.

We adopt the recently popularized game approach of \citet{degenne_non-asymptotic_2019}. The idea is to consider a sequential zero-sum game between two players. This game approximates the optimal allocation given by the lower bound \citep{kaufmann_complexity_2016}. The objective of our work is to design asymptotically optimal algorithms with finite-time guarantees. They should have computationally efficient implementations as long as the offline combinatorial problem can be solved efficiently. 

\paragraph{Contributions} 
(1) We use the game framework for pure exploration to study combinatorial bandits with semi-bandit feedback. The action and answer sets are arbitrary and the feedback is independent across arms. Despite its increasing popularity, the game framework has not yet been used in combinatorial bandits or in the transductive setting. (2) We develop a pure-exploration CombGame meta-algorithm whose instances are asymptotically optimal algorithms with finite time guarantees. The family of algorithms directly adapts the pure-exploration meta-algorithm of \citet{degenne_non-asymptotic_2019} to the combinatorial nature of the problem allowing for tractable implementation of the game framework. (3) To overcome the limitation of prior work, we employ the projection-free algorithm over convex polyhedral sets of \citet{garber_linearly_2015}. This approach is the first computationally efficient algorithm which is asymptotically optimal and has competitive empirical performance. 

\subsection{Related Work}  \label{related_work}

Combinatorial bandits have been introduced by \citet{cesa-bianchi_combinatorial_2012} and \citet{chen_combinatorial_2013}. The emblematic examples of combinatorial actions are the basis of a matroid \citep{perrault_exploiting_2019} and the paths in a graph \citep{talebi_stochastic_2018}. Semi-bandit feedback is extensively studied \citep{kveton_tight_2015, wen_efficient_2017}. Other works have considered the \textit{bandit} feedback where the agent observes an aggregated reward \citep{combes_combinatorial_2015}. Generalizing them both, the partial linear monitoring feedback has been studied for cumulative regret minimization \citep{kirschner_pm_2020} and for pure exploration \citep{chen_combinatorial_2020}. Combinatorial bandits have also been used to denote a different setting where the agent plays arms to identify the best action \citep{chen_combinatorial_2014, chen_pure_2016, chen_nearly_2017, cao_disagreement-based_2019}. Combinatorial bandits have also been generalized to consider submodular reward functions \citep{hazan_online_2012,chen_interactive_2017}.

Before the game approach was introduced, \citet{jamieson_best-arm_2014} highlighted three important types of algorithms to solve BAI. They were based on action elimination \citep{karnin_almost_2013}, upper confidence bound (UCB) \citep{audibert_best_2010} or lower UCB \citep{kalyanakrishnan_pac_2012}. Bayesian strategies have also been proposed with Thompson sampling like algorithms \citep{russo_simple_2018, kaufmann_sequential_2018, shang_fixed-confidence_2019}. Generalizing the BAI problem to the identification of the $k$ best arms, top-$k$ identification has been studied for an agent playing arms \citep{gabillon_multi-bandit_2011, kalyanakrishnan_pac_2012, scarlett_overlapping_2019} or batches of arms \citep{jun_top_2016, kuroki_polynomial-time_2019, rejwan_top-k_2019}. The pure-exploration framework encompasses more complex queries such as maximin \citep{garivier_maximin_2016} or minimum threshold \citep{degenne_non-asymptotic_2019}. Some problems admit multiple correct answers \citep{degenne_pure_2019}.\looseness=-1

In the fixed-confidence pure-exploration setting the first known lower bounds on the sample complexity involve a characteristic time whose inverse is a complexity measure \citep{kaufmann_complexity_2016}. Those setting-dependent lower bounds have motivated the search for algorithms with matching upper bound, both in finite-time \citep{simchowitz_simulator_2017} and asymptotic regime \citep{garivier_optimal_2016}. Unfortunately, existing algorithms often require an expensive oracle to compute the optimal allocation weights which are used for sampling, such as Track-and-Stop \citep{garivier_optimal_2016} or RAGE \citep{fiez_sequential_2019}. \citet{degenne_non-asymptotic_2019} introduces the game framework which interprets the optimization problem as a zero-sum game between two players. In particular, it proposes a pure-exploration meta-algorithm which uses a cheaper best-response oracle. The game framework has inspired recent algorithms for linear bandits, such as PELEG \citep{zaki_explicit_2020} or LinGame(-C) \citep{degenne_gamification_2020}. PELEG extends the phased-elimination algorithm of \cite{fiez_sequential_2019}. The idea has also been adapted to cumulative regret in \citet{degenne_structure_2020}.

\section{Preliminaries}  \label{preliminaries}

In this section we formally define \emph{pure exploration for combinatorial bandits with semi-bandit feedback}, and prove a lower bound on the sample complexity. We then use the lower bound to determine sampling strategies for our algorithm.

\subsection{Problem Formulation}

Suppose the environment consists of $d$ arms (or base arms). Each arm $a \in [d] \defeq \{ 1, \cdots, d \}$ is associated with a probability distribution from the exponential family $\nu_a$ characterized by the unknown mean $\mu_a$. Given known $\sigma_a$, we consider two cases, in which: (a) $\nu_a$ is $\sigma_a^2$-sub-Gaussian and (b) $\nu_a$ is Gaussian  $\mathcal{N}(\mu_a, \sigma_a^2)$. An exponential family $\nu_a$ is $\sigma_a^2$-sub-Gaussian if and only if for all $(\mu_a, \lambda_a)$ the KL divergence satisfies $d_{\text{KL}}(\mu_a, \lambda_a) \geq \frac{(\mu_{a} - \lambda_a)^2}{2 \sigma_a^2}$. The independent joint distribution of the arms is denoted by $\nu$ and defined uniquely by $\mu \defeq (\mu_a)_{a \in [d]} \in \M$. The set of possible parameters $\M \subset \R^d$ is known to the agent. Similarly to earlier work on bandits, $\M$ is assumed to be bounded. As proven in Appendix \ref{appendix_unbounded_setting}, this assumption is immaterial for Gaussian distributions. The component-wise KL divergence between the true parameter $\mu$ and a different parameter $\lambda \in \M$ is denoted by a vector $d_{\text{KL}}(\mu, \lambda) \defeq \left(d_{\text{KL}}(\mu_a, \lambda_a)\right)_{a \in [d]}$.

We define the action set $\A \subset 2^{[d]}$ as a collection of sets of arms (a subset of the power set of arms). The agent can only play actions. In the literature, \textit{super arms} or \textit{multiple arms} are used to denote actions. As a special case, the action set could be the singletons (arms) $\A = \{\{a\}\}_{a \in [d]}$. Let $K \defeq \max_{A \in \A}|A|$ be the maximum size of an action. At each round $t \geq 1$, the agent chooses an action $A_t \in \A$ and observes a noisy semi-bandit feedback $Y_{t, A_t} \defeq \left(Y_{t, a} \bm{1}_{(a \in A_t)} \right)_{a \in [d]}$ where $\bm{1}_{S} \defeq \left(\bm{1}_{(a \in S)} \right)_{a \in [d]}$ is the indicator vector for $S \subset [d]$ and $Y_t \sim \nu$ is the observation vector in $\R^d$.

We define the answer set $\I \subset 2^{[d]}$ as a collection of sets of arms, possibly different from the set of actions $\A$. The setting where $\I$ and $\A$ differ is also known as the transductive bandit setting. Given a parameter $\lambda$, the reward of an answer $I \in \I$ is the sum of the rewards of each arm $\langle \lambda, \1_{I}\rangle \defeq \sum_{a \in [d]} \lambda_a \1_{(a \in I)}$. The \textit{correct answer} is given by the function $I^*: \M \mapsto \I$ defined as $I^*(\lambda) \defeq \argmax_{I \in \I } \langle \lambda, \1_{I}\rangle$. For simplicity, we assume that $I^*(\lambda)$ is unique for all $\lambda \in \M$. A more careful analysis would allow to relax this assumption to: $I^*(\mu)$ is unique for the unknown $\mu$ characterizing the bandit $\nu$. The goal of the agent is to identify the correct answer $I^*(\mu)$ by interacting with the environment. BAI is a special case where $\I = \{\{a\}\}_{a \in [d]}$. Best-action identification is obtained for $\I = \A$. 

We assume that the agent has access to efficient oracles\footnote{In practice, such an oracle may be an efficient algorithm tailored to the combinatorial constraints (e.g., Kruskal's algorithm for minimum spanning trees etc.), or a search strategy given by a Mixed Integer Programming solver.} to solve the offline linear optimization problems $\argmax_{A \in \A} \langle \1_{A}, c \rangle$ and $\argmax_{I \in \I} \langle \1_{I}, c \rangle$ for a given linear objective $c \in \R^{d}$. This assumption is commonly made for semi-bandits \citep{cao_disagreement-based_2019, kuroki_polynomial-time_2019, perrault_statistical_2020}. It is crucial, since the offline problem cannot be efficiently solved without this oracle.

\paragraph{Policies} The history $\F_t \defeq \sigma(A_{1}, Y_{1, A_{1}}, \cdots, A_{t}, Y_{t, A_{t}})$ contains all the information available to the agent at step $t+1$. In the fixed-confidence setting a strategy is described by three rules: a \textit{sampling rule} $(A_t)_{t \geq 1}$ where $A_t \in \A$ is $\F_{t-1}$-measurable, a \textit{stopping rule}, $\tau_{\delta}$ being the stopping time with respect to the filtration $(\F_t)_{t \geq 1}$, and a \textit{recommendation rule} $I_{\tau_{\delta}}$ which is $\F_{\tau_{\delta}}$-measurable.

While the sampling rule can be randomized, we consider only deterministic strategies in our work. In the fixed-confidence setting, the learner is given a confidence parameter $\delta \in (0,1)$. The strategy is said to be $\delta$-PAC if it terminates and recommends the correct answer with probability at least $1-\delta$: $\mathbb{P}_{\nu} \left[ \tau_{\delta} = \infty \lor  I_{\tau_{\delta}} \neq I^*(\mu) \right] \leq \delta$. Among $\delta$-PAC algorithms, the objective is to minimize the expected number of samples required to terminate $\E_{\nu}[\tau_{\delta}]$, also known as the \textit{sample complexity}.

\subsection{Sample Complexity Lower Bound}

Given an answer $I \in \I$, the \textit{cell $\Theta_I$} is the set of parameters for which the correct answer is $I$, $\Theta_{I} \defeq \{ \lambda \in \M: I^*(\lambda) = I\}$. The \textit{alternative to $I$} is the set of parameters for which $I$ is not the correct answer, $\Theta_{I}^{\complement}$. It is also equal to the set of parameters for which there exists an answer $J \neq I$ having a higher reward, $\Theta_{I}^{\complement} = \bigcup_{J \in \I \setminus \{I\}} \Bar{\Theta}_{J}^{I}$ where $\Bar{\Theta}_{J}^{I} \defeq \{ \lambda \in \M : \langle \1_{J} - \1_{I}, \lambda\rangle \geq 0\}$. The \textit{neighbors to I} is the set of answers whose cells' boundaries intersect the boundary of the cell $I$, $N(I) \defeq \{ J \in \I: \partial \Theta_{I} \cap \partial \Theta_{J} \neq \emptyset \}$. 

The \textit{transformed simplex} $\tsimplex := \{W_{\A}w : w \in \simplex\} \subset \mathbb{R}^d$ is the image of the $|\A|$-dimensional probability simplex $\simplex \defeq \left\{ w \in \R^{|\A|}: w\geq 0 \land \sum_{A \in \A} w_{A} =1 \right\}$ by $W_{\A} \defeq \begin{bmatrix} \bm{1}_{A_1} \hdots \bm{1}_{A_{|\A|}} \end{bmatrix}  \in \mathbb{R}^{d \times |\A|}$. The matrix $W_{\A}$ collects the action incidence vectors. For a distribution over actions $w \in \simplex$, $W_{\A}w$ represents the effect at the base arm level when sampling actions according to $w$. The probability of sampling the arm $a \in [d]$ is $\tilde w_a$, where $\tilde \cdot \defeq W_{\A} \cdot$ denotes implicitly the operator $W_{\A}$.

\paragraph{Lower bound} Given any $\delta$-PAC strategy, Theorem \ref{thm:lower_bound} gives a finite-time and asymptotic lower bound on the sample complexity, see Appendix \ref{appendix_proof_lower_bound} for a proof. This result is a technical extension of previous work, see Theorem 1 in \citet{garivier_optimal_2016}.

\begin{theorem} \label{thm:lower_bound}
For any $\delta$-PAC strategy and any bandit $\nu$ characterized by $\mu$,
\begin{align*}
    \frac{\E_{\nu}[\tau_{\delta}]}{\ln ( 1/(2.4\delta))} \geq D_{\nu}^{-1}  \quad \text{ and } \quad
    \limsup_{\delta \rightarrow 0} \frac{\E_{\nu}[\tau_{\delta}]}{\ln (1/\delta)} \geq D_{\nu}^{-1} 
\end{align*}
where the {\em complexity} $D_{\nu}$ is the inverse of the characteristic time, defined by
\begin{equation*}
    D_{\nu} \defeq \max_{\tilde{w} \in \tsimplex} \inf_{\lambda \in \Theta_{I^*(\mu)}^{\complement}} \langle \tilde{w}, d_{\text{KL}}(\mu, \lambda) \rangle
\end{equation*}
\end{theorem}

Similar bounds were already proven for other settings \citep{garivier_optimal_2016, degenne_pure_2019}. The technical difference is that we sample actions. A $\delta$-PAC strategy is said to be \textit{asymptotically optimal} if the bound is tight, meaning that for any $\nu$, $\limsup_{\delta \rightarrow 0} \frac{\E_{\nu}[\tau_{\delta}]}{\ln (1/\delta)} \leq D_{\nu}^{-1}$.

The \textit{set of optimal allocations} is $w^{*}(\mu) \defeq \left\{ w \in \simplex: \inf_{\lambda \in \Theta_{I^*(\mu)}^{\complement}} \langle W_{\A}w, d_{\text{KL}}(\mu, \lambda) \rangle = D_{\nu} \right\}$. It is non-empty since $\tilde{w} \mapsto \inf_{\lambda \in \Theta_{I^*(\mu)}^{\complement}} \langle \tilde{w}, d_{\text{KL}}(\mu, \lambda) \rangle$ is concave on the compact $\tsimplex$. Moreover, $w^{*}(\mu)$ contains multiple optimal allocations, except for specific choice of $\A$. Computing an element of $w^{*}(\mu)$ is a difficult minmax optimization even for a known $\mu$. To the best of our knowledge, there are no theoretical results on the hardness of this specific optimization problem.

\section{Algorithms}  \label{algorithms}

After introducing the game approach, we discuss two asymptotically optimal families of algorithms which instantiate our proposed pure-exploration CombGame meta-algorithm, see Algorithm \ref{algo:combgame_meta_algorithm}. The learners used to instantiating it are either on $\simplex$ or on $\tsimplex$. 

\subsection{Game Approach}

At round $t \geq 1$, the agent computes a distribution over actions $w_{t} \in \simplex$ which is converted into a deterministic action $A_t$ by tracking \citep{garivier_optimal_2016}, as explained below. Since we observe semi-bandit feedback, $w_{t}$ corresponds to $\tilde{w}_t = W_{\A} w_{t}$ at the base arms level. Importantly, due to the independence assumption and the linearity of the considered operators, all computations on $\simplex$ can be done on $\tsimplex$.

Since $\tsimplex = \text{conv} \left(\{ \1_{A} \}_{A \in \A}\right)$, the transformed simplex is a $0$-$1$ polytope in $\R^d$. A pulling proportion $w \in \simplex$ is said to be \textit{sparse} if its support is small, $\text{supp}(w) \ll |\A|$. A simple application of Carath\'eodory's theorem yields that for all $w \in \simplex$ there exists a sparse $w_{0} \in \simplex$ with $|\text{supp}(w_{0})| \leq d+1$ such that both $w$ and $w_{0}$ have the same allocation over arms, $W_{\A} w = W_{\A} w_{0}$. 

\paragraph{Two-player, minimax approach} As noted in the early work by \citet{chernoff_sequential_1959} and extended in the recent papers using gamification \citep{degenne_non-asymptotic_2019, degenne_gamification_2020}, the complexity $D_{\nu}$ is the value of a fictitious zero-sum game between two players. The agent chooses a pulling proportion over arms, $\tilde{w} \in \tsimplex$. The nature plays the most confusing alternative with respect to the KL divergence in order to fool the agent into predicting an incorrect answer, $\lambda \in \Theta_{I^*(\mu)}^{\complement}$.

Allowing nature to play distributions over alternatives and using Sion's minimax theorem, we can invert the order of the players to obtain the dual formulation of the complexity $D_{\nu}$,
\begin{equation*}
    D_{\nu} = \max_{\tilde{w} \in \tsimplex} \inf_{\lambda \in \Theta_{I^*(\mu)}^{\complement}} \langle \tilde{w}, d_{\text{KL}}(\mu, \lambda) \rangle = \inf_{q \in \mathcal{P}\left(\Theta_{I^*(\mu)}^{\complement}\right)}  \max_{A \in \A} \E_{\lambda \sim q} \left[ \langle \1_{A}, d_{\text{KL}}(\mu, \lambda) \rangle \right]
\end{equation*}
where $\mathcal{P}\left(\Theta_{I^*(\mu)}^{\complement}\right)$ denotes the set of probability distributions over $\Theta_{I^*(\mu)}^{\complement}$. 

In our work we focus on a sequential game where the agent, or $A$-player, plays first and nature, or the $\lambda$-player, is second. The $A$-player uses a learner that minimizes the cumulative regret. The $\lambda$-player has access to a best-response oracle that has no regret. This combination ensures a saddle-point property required to derive the finite-time upper bound on the sample complexity. Alternatively the order could be reversed, or they could play simultaneously \citep{degenne_non-asymptotic_2019}.

\subsection{CombGame Meta-Algorithm}

First, we briefly introduce the estimator, stopping and recommendation rules, which define the pure-exploration algorithm. Since $\mu$ (and the best answer $I^*(\mu)$) is unknown, we use the maximum likelihood estimator (MLE) $\mu_t$ as a plug-in estimator. The recommendation and the stopping rules are frequentist and use the value of $\mu_t$. Based on $\mu_t$, the sampling rule corresponds to playing an optimistic sequential game. Both the sample complexity and the computational efficiency depend on the learner used to approximate this game.

\paragraph{Estimator} 
Let $N_{t-1} \in \R^{|\A|}$ be the count of sampled actions at the beginning of round $t$ and $\tilde{N}_{t-1} = W_{\A}N_{t-1}$ its counterpart at the base arms level. The MLE, $\mu_{t-1,a} \defeq  \frac{1}{\tilde{N}_{t-1,a}} \sum_{s = 1}^{t-1} \bm{1}_{(a \in A_s)} Y_{s,a}$ for all $a \in [d]$, is associated with the confidence hyperbox for the exploration bonus $f$, $\C_t \defeq \bigtimes_{a \in [d]} [\alpha_{t,a}, \beta_{t,a}]$ where $[\alpha_{t,a}, \beta_{t,a}] \defeq \{\lambda: \tilde{N}_{t-1,a} d_{\text{KL}}(\mu_{t-1,a},\lambda) \leq f(t-1)\}$. As in \citet{degenne_non-asymptotic_2019}, the exploration bonus is chosen as $f(t) = \overline{W}((1+c)(1+b) \ln(t))$ where $c>0$, $b>0$ and $\overline{W}(x) \approx x + \ln(x)$, see Appendix \ref{proof_upper_bound_finite_time_decomposition} for an exact definition.

When $\mu_{t-1} \notin \M$, we consider $\tilde{\mu}_{t-1} \in \argmin_{\lambda \in \M \cap \C_{t}} \langle \tilde{N}_{t-1}, d_{\text{KL}}(\mu_{t-1}, \lambda) \rangle$, the projection of $\mu_{t-1}$ on $\M \cap \C_{t}$. $\tilde{\mu}_{t-1}$ is chosen randomly when $\M \cap \C_{t} = \emptyset$. When all arms are sampled an infinite number of times, we have $\lim_{\infty} \mu_t = \mu \in \M$: there exists $T_{0}$ such that for all $t \geq T_{0}$, $\mu_{t-1} \in \M$.\looseness=-1

\paragraph{Stopping and recommendation rules}
We will use the recommendation and stopping rules based on a frequentist estimator $\mu_{t-1}$. Given the feasible $\tilde{\mu}_{t-1}$, we recommend the unique best answer $I_t \defeq \argmax_{I \in \I} \langle \1_{I}, \tilde{\mu}_{t-1} \rangle$. $I_t$ can be computed with the efficient oracle. We stop as soon as the generalized likelihood ratio is above a stopping threshold $\beta(t-1, \delta)$: 
\begin{equation*}
    \tau_{\delta} \defeq \inf \left\{t \in \mathbb{N}: \min_{J \in N(I_t)} \inf_{\lambda \in \Bar{\Theta}_{J}^{I_t} } \langle \tilde{N}_{t-1}, d_{\text{KL}}(\mu_{t-1}, \lambda)\rangle > \beta(t-1,\delta)\right\}
\end{equation*}

Given any sampling rule, this pair of rules is sufficient to obtain a $\delta$-PAC strategy, see Theorem \ref{thm:delta_pac}. The proof leverages the concentration inequalities of \citet{kaufmann_mixture_2018} (Appendix \ref{appendix_proof_delta_pac}).\looseness=-1

\begin{algorithm2e}
        \caption{CombGame meta-algorithm}
        \label{algo:combgame_meta_algorithm}
        \SetAlgoLined
        \KwIn{Learner $\A^A$ with associated init, stopping threshold $\beta(t-1,\delta)$, exploration bonus $f(t)$}
        \KwOut{Answer $I_t$}
        $(w_{n_{0}}, \tilde{w}_{n_{0}}, B_{n_{0}}) =$ INIT(init)  \Comment*[r]{initialization}
         \For{$t=n_{0} +1, \cdots$}{
            $I_t = \argmax_{I \in \I}\langle \bm{1}_{I}, \tilde{\mu}_{t-1}\rangle$ \Comment*[r]{recommendation rule}
            If $\min_{J \in N(I_t)} \inf_{\lambda \in \Bar{\Theta}_{J}^{I_t} } \langle \tilde{N}_{t-1}, d_{\text{KL}}(\mu_{t-1}, \lambda)\rangle > \beta(t-1,\delta)$ then return $I_t$  \Comment*[r]{stopping rule}
            Get $(w_t, \tilde{w}_t, B_t)$ from $\A^{A}_{I_t}$\;
            $A_t \in \argmin_{A \in B_t} \frac{N_{t-1,A}}{\sum_{s=1}^{t} w_{s,A}}$  \Comment*[r]{sparse C-Tracking}
            $(\cdot,\lambda_t) \in \argmin_{J \in N(I_t), \lambda \in \Theta_{J}^{I_t}} \langle \tilde{w}_{t}, d_{\text{KL}}(\mu_{t-1},\lambda) \rangle$  \Comment*[r]{$\lambda$-player}
            $\forall a \in [d], \quad r_{t,a} = \max \left\{ \frac{f(t-1)}{\tilde{N}_{t-1,a}}, \max_{\phi \in \{\alpha_{t,a},\beta_{t,a}\}} d_{\text{KL}}(\phi, \lambda_{t,a}) \right\} $ \Comment*[r]{optimism}
            Feed $\A^{A}_{I_t}$ with the reward $r_t$\;
            Observe a sample $Y_{t,A_t}$ and update $(\mu_t, N_t, \tilde{\mu}_t)$ \;
        }
\end{algorithm2e}

\begin{theorem}  \label{thm:delta_pac}
Let $\mathcal{M}$ be bounded. Regardless of the sampling rule, a strategy using the frequentist recommendation/stopping pair with the stopping threshold:
\begin{align*}
\beta (t, \delta) &\defeq 
\begin{cases} 
    3d_{0} \ln \left(1 + \ln \left(\frac{t K}{d_{0}}\right)\right) + d_{0}\T\left(\frac{\ln \left( \frac{|\I|-1}{\delta}\right)}{d_{0}}\right) & \text{for (a)}\\
     2 d_{0}\ln \left(4 + \ln \left(\frac{t K}{d_{0}}\right)\right) + d_{0} \C^{g_G}\left(\frac{\ln \left( \frac{|\I|-1}{\delta}\right)}{d_{0}}\right) & \text{for (b)}
    \end{cases} 
\end{align*}
is $\delta$-PAC. In the above, $d_{0} \defeq \max_{I,J \in \I, J \neq I} |(I \setminus J) \cup (J \setminus I)|$, $\T$ and $\C^{g_G}$ are the functions defined in \citet{kaufmann_mixture_2018}, $\C^{g_G}(x) \approx x + \ln(x)$ and $\T(x) \approx x + 4 \ln(1+x + \sqrt{2x})$ for $x \geq 5$.\looseness=-1
\end{theorem}

\subsubsection{Sampling rule}

The challenge is to define the sampling rule in order to satisfy the stopping criterion as soon as possible. Based on the definition of $\tau_{\delta}$, $\min_{J \in N(I_t)} \inf_{\lambda \in \Bar{\Theta}_{J}^{I} } \langle \tilde{N}_{t}, d_{\text{KL}}(\mu_t, \lambda)\rangle$ should be maximized. We will achieve the desired saddle-point property by combining $|\I|$ learners for the $A$-player, one per candidate answer $\A^A_{I_t}$, and one best-response oracle for the $\lambda$-player.

We present two categories of learners, both aiming at minimizing the cumulative regret $R_{t}^{A} \defeq \max_{A \in \A} \sum_{s=1}^{t} \langle \1_{A}, r_s\rangle -  \sum_{s=1}^{t} \langle \tilde{w}_s, r_s\rangle $. $r_t$ is the optimistic reward at time $t$ as defined in the paragraph below. Learners on the simplex update $w_t$ and need a \textit{full} initialization where each action is sampled once. To overcome the computational inefficiency of those learners, we also consider learners on the transformed simplex which update $\tilde{w}_t$. By leveraging the sparse support when tracking, they only require a \textit{covering} initialization where each arm is observed at least once. The length of the initialization is denoted by $n_{0}$. We compare the different learners in Table \ref{tab:comparison_learners}.

By knowing $w_t$ used by the $A$-player, the $\lambda$-player can adopt the most confusing parameter in $\Theta_{I_t}^{\complement}$: $(\cdot, \lambda_t) \in \argmin_{J \in N(I_t), \lambda \in \Bar{\Theta}_{J}^{I_t}} \langle \tilde{w}_{t}, d_{\text{KL}}(\mu_{t-1},\lambda) \rangle$.

\paragraph{Optimism} Since the estimator is not exact but associated to a confidence region, following the exact sequential game for $\mu_t$ cannot lead to sufficient exploration. \citet{degenne_non-asymptotic_2019} overcome this hurdle by using the optimism principle. Since $\mu \in \C_t$ with high probability, the optimistic reward $r_t$ is the upper bound on the gain of the agent given the $\lambda$-player's response, $\lambda_t$: for all $a \in [d]$, $r_{t,a} \defeq \max \left\{ \frac{f(t-1)}{\tilde{N}_{t-1,a}}, \max_{\phi \in \{\alpha_{t,a},\beta_{t,a}\}} d_{\text{KL}}(\phi, \lambda_{t,a}) \right\}$ where $\frac{f(t-1)}{\tilde{N}_{t-1,a}}$ fosters exploration. The clipping is due to non-symmetric $d_{\text{KL}}$. It disappears for Gaussian as shown in Lemma \ref{lem:gaussian_optimistic_reward}. 

\paragraph{Tracking} Since a learner plays pulling proportion over actions $w_t$, we need to convert it into an action choice $A_t$. Introduced in \citet{garivier_optimal_2016}, C-Tracking and D-Tracking allow to deterministically convert weights into pulls. Due to the non-uniqueness of the optimal allocation of weights, we consider C-Tracking, which ensures $ 1 - |\A| \leq N_{t,A} - \sum_{s=1}^{t} w_{t,A} \leq 1$, see Appendix \ref{tracking_results_technicalities}. We obtain a sparse tracking procedure by limiting the choice of $A_t$ to the incremental support $B_t  \defeq \text{supp}\left( \sum_{s=1}^{t} w_{s} \right)$: $A_t \in \argmin_{A \in B_t} \frac{N_{t-1,A}}{\sum_{s=1}^{t} w_{s,A}}$. Alternatives include D-Tracking or the rounding procedure in \citet{fiez_sequential_2019}. For a non-deterministic algorithm we can directly sample the next action, $A_t \sim w_t$.

\subsection{Learners on the Simplex}  \label{learner_on_simplex}

Since we are playing pulling proportion over actions, the immediate approach is to consider Hedge-type algorithms. They constitute a family of learners on the probability simplex $\simplex$. As examples from this family, we will use Hedge \citep{cesa-bianchi_improved_2006} and the adaptive version AdaHedge \citep{de_rooij_follow_2013}. An algorithm is said to be \textit{anytime} if it is independent of the horizon $T$. Those learners require computations at the actions level to obtain a reward vector $U_t$: for all $A \in \A$, $U_{t,A} \defeq \langle \1_{A}, r_t \rangle$. For both learners the update of $w_{t}$ is: for all $A \in \A$, $w_{t,A} = \frac{w_{n_{0},A} \exp \left( - \eta_{t} L_{t-1, A}\right) }{\sum_{A' \in \A} w_{n_{0},A'} \exp \left( - \eta_t L_{t-1, A'} \right)}$ where $L_{t-1, A} = - \sum_{s=1}^{t-1} U_{s,A}$ is the cumulative loss, $\eta_t$ is the learning rate and $w_{n_{0}} = \frac{1}{|\A|} \1$ is the sampling parameter for a full initialization. In Hedge, $\eta_t$ is a constant depending on $T$. While in AdaHedge, $\eta_t$ is decreasing and defined as a function of a cumulative mixability gap. As shown in Lemmas \ref{lem:hedge_cumulative_regret} and \ref{lem:adahedge_cumulative_regret}, both Hedge and AdaHedge have optimal cumulative regret, $O\left(\ln(t)\sqrt{t}\right)$. The additional $\ln(t)$-factor originates from the unbounded losses.

Due to the potentially exponential number of actions, a closer examination of those learners reveals the \emph{computational inefficiency} of three steps. First, we initialize by sampling all the actions once. Second, at each round the update step requires the computation of $U_t \in \R^{|\A|}$ and $w_t \in \simplex$. Third, C-Tracking is equivalent to finding the minimum of $|\A|$ values since $w_t$ is dense. This motivates considering the second family of algorithms, which defines the learner directly on $\tsimplex$.

\subsection{Learners on the Transformed Simplex}   \label{learner_on_tsimplex}

To circumvent the shortcomings of the learners on $\simplex$, we introduce a second family of learners for which we update $\tilde{w}_t \in \tsimplex$ by using $r_t$. Note that the loss for the $A$-learner is linear, $f_t(x) = - \langle x, r_t \rangle$ for $x \in \tsimplex$. The online convex optimization (OCO) literature provides algorithms achieving optimal cumulative regret guarantees for adversarial linear losses. Since we want a computationally efficient algorithm, the learner should satisfy three additional requirements. First, it should be projection-free, since projections onto $\tsimplex$ require a solution to a costly quadratic optimization problem. Second, the learner should access at most one efficient linear optimization oracle per round. Third, the algorithm should maintain efficiently an incrementally sparse representation in the simplex, which is used for sparse tracking. Projection-free algorithms have been extensively studied since they are computationally efficient, as long as the linear optimization oracle is computationally efficient and increase support incrementally. They are often based on the Frank-Wolfe approach \citep{frank_algorithm_1956, jaggi_revisiting_2013, lacoste-julien_global_2015}.

The anytime Online Frank-Wolfe (OFW) \citep{hazan_projection-free_2012} and Local Linear Optimization Oracle-based OCO (LLOO) \citep{garber_linearly_2015} satisfy those requirements. LLOO is tailored to convex polyhedral sets, see Appendix \ref{appendix_implementation_details} for details. Therefore, the assumptions of LLOO are satisfied in our setting. Both use a single call per round to the linear optimization oracle in order to compute the best vertex $\1_{\tilde{A}_t}$ of the polytope with respect to the gradient of a regularized cumulative loss  $F_t$: $\tilde{A}_t \in \argmin_{A \in \A} \langle \1_{A} , \nabla F_{t}(\tilde{w}_t) \rangle$. While $F_t(x) = \frac{1}{t} \sum_{s = 1}^{t}  \frac{s^{-1/4}}{\diamtsimplex}  \|x - \tilde{w}_{n_{0}}\|_{2}^2 - \langle x, r_s \rangle$ for OFW, where $\diamtsimplex$ denotes the diameter of $\tsimplex$, we have $F_t(x) = \|x - \tilde{w}_{n_{0}}\|_{2}^2 - \eta_{\A,T}  \sum_{s = 1}^{t}  \langle x, r_s \rangle$ for LLOO. OFW simply moves on the segment connecting $\tilde{w}_t$ and $\1_{\tilde{A}_t}$, $\tilde{w}_{t+1} = \tilde{w}_t + t^{-1/4} \left( \1_{\tilde{A}_t} - \tilde{w}_t \right) \in \tsimplex$. LLOO adopts a more sophisticated strategy whose parameters $\eta_{\A,T}$, $\gamma_{\A}$ and $M_{\A, T}$ depend on the horizon $T$, see Lemma \ref{lem:lloo_cumulative_regret} for explicit formulas. LLOO simultaneously moves towards the best corner $\1_{\tilde{A}_t}$ and away from the ordered worst corners by using several pairwise Frank-Wolfe steps, see \citet{lacoste-julien_global_2015}. The corresponding update is $\tilde{w}_{t+1} = \tilde{w}_{t} + \gamma_{\A} \left( M_{\A, T}  \bm{1}_{\tilde{A}_t}  - \tilde{w}_{t,-} \right)$, where $(\tilde{w}_{t,-}, w_{t,-}) = \A^{\text{reduce}}(w_{t}, B_t, M_{\A, T}, \nabla F_{t}(\tilde{w}_{t}))$ and $\A^{\text{reduce}}$ is detailed in Algorithm \ref{algo:reduce_lloo_point}. The computations of $\A^{\text{reduce}}$ are dominated by the cost of sorting $|B_t|$ inner-products in $\R^d$, $O\left( |B_t| (d  + \ln(|B_t|))\right)$. Since $W_{\A}$ is a linear map, both variants of the convex-combination update of $\tilde{w}_{t+1}$ are propagated to the simplex to obtain $w_{t+1}$ by using $w_t$, $\delta_{\tilde{A}_t}$ (dirac function in $\tilde{A}_t$) and $w_{t,-}$ instead of $\tilde{w}_t$, $\bm{1}_{\tilde{A}_t}$ and $\tilde{w}_{t,-}$. The corresponding support is incrementally sparse, $B_{t+1} \setminus B_t \subset \{ \tilde{A}_t \}$, and unchanged when $\tilde{A}_t$ is already included in $B_t$. Lemma \ref{lem:ofw_cumulative_regret} shows that OFW has an upper bound on the cumulative regret in $O\left(\ln(t)^2 t^{3/4}\right)$, which is in general suboptimal for the online linear optimization setting. Thanks to these extra computations, Lemma \ref{lem:lloo_cumulative_regret} yields that LLOO has optimal cumulative regret, $O\left(\ln(t) \sqrt{t}\right)$. Those results are obtained by modifying existing ones \citep{hazan_projection-free_2012,garber_linearly_2015} to account for an unbounded reward and modified parameters for OFW. Since $R_{t}^{A}$ appears in the finite-time upper bound on the sample complexity (Theorem \ref{thm:finite_time_upper_bound}), optimal cumulative regret is a desirable property if we strive for low sample complexity. This is validated by our experimental results.

\begin{table}
\centering
\scalebox{0.85}{
\begin{tabular}{|c || c | c | c | c |} 
 \hline
  & Sparse support & Computational cost & Anytime & Cumulative regret \\
 \hline\hline
 Hedge &  \xmark & $O\left( |\A| \right)$ & \xmark &  $O\left(\ln(t)\sqrt{t}\right)$ \\ 
 \hline
 AdaHedge &  \xmark & $O\left( |\A| \right)$ & \checkmark & $O\left(\ln(t)\sqrt{t}\right)$  \\ 
 \hline
 OFW &  \checkmark & $O\left( |B_t| \right)$ & \checkmark &  $O\left(\ln(t)^2 t^{3/4}\right)$ \\ 
 \hline
 LLOO & \checkmark  & $O\left( |B_t| (d  + \ln(|B_t|))\right)$  & \xmark  & $O\left(\ln(t)\sqrt{t}\right)$ \\
 \hline
\end{tabular}}
\caption{Comparison of the relevant properties of the learners used to instantiate CombGame. For cumulative regret, the notation $O(\cdot)$ hides parameters independent of $t$, see Appendix \ref{appendix_learner_cumulative_regret}. For the computational cost, $O(\cdot)$ hides constant values, small compared to $|B_t|$ and $|\A|$.}
\label{tab:comparison_learners}
\end{table}

\begin{algorithm2e}
    \caption{LLOO's $\A^{\text{reduce}}$}
    \label{algo:reduce_lloo_point}
    \SetAlgoLined
    \KwIn{$w \in \Delta_{|\A|}$ with sparse support $B$, probability mass $M \in \R$ and cost vector $c \in \mathbb{R}^d$}
    $\forall A \in B, \quad l_{A} = \langle \bm{1}_{A}, c\rangle$\;
    Let $i_{1}, \cdots, i_{|B|}$ be a permutation such that $l_{A_{i_{1}}} \geq \cdots \geq l_{A_{i_{|B|}}}$\;
    Let $k$ be the smallest integer such that $\sum_{j=1}^{k} w_{A_{i_{j}}}\geq M$ \;
    $(\tilde{w}_{-}, w_{-}) = \sum_{j=1}^{k-1} w_{A_{i_{j}}}  \left(\bm{1}_{A_{i_{j}}}, \delta_{A_{i_{j}}}\right)  +  \left( M - \sum_{j=1}^{k-1} w_{A_{i_{j}}} \right) \left(\bm{1}_{A_{i_{k}}}, \delta_{A_{i_{k}}}\right)$ \;
    Return $(\tilde{w}_{-}, w_{-})$\;
\end{algorithm2e}

\section{Sample Complexity Upper Bound}  \label{sample_complexity_upper_bound}

In this section we present and sketch the proof of the finite-time upper bound on the sample complexity of our instantiated CombGame meta-algorithm. 

\subsection{Finite-time Upper Bound}

Given a learner with sub-linear cumulative regret $R_{t}^{A} = o(t)$, Theorem \ref{thm:finite_time_upper_bound} shows that the instances of Algorithm \ref{algo:combgame_meta_algorithm}, the CombGame meta-algorithm, satisfy a finite-time upper bound on the sample complexity. The upper bound involves the complexity $D_{\nu}$. The leading constant is optimal in the asymptotic regime $\delta \rightarrow 0$. Those results and their proofs are inspired from Theorem 2 in \citet{degenne_non-asymptotic_2019}. It also bares similarity with Theorem 2 of \citet{degenne_gamification_2020}.

\begin{theorem} \label{thm:finite_time_upper_bound}
    Let $\mathcal{M}$ be bounded. The sample complexity of the instantiated CombGame meta-algorithm on bandit $\mu \in \mathcal{M}$ satisfies:
    \begin{align*}
     &\E_{\nu}[\tau_{\delta}] \leq T_{0}(\delta) + \frac{2ed}{c^2} \quad 
      \text{ with } \quad  T_{0}(\delta) \defeq \max \left\{t \in \mathbb{N}: t \leq \frac{\beta(t,\delta)}{D_{\nu}} + C_{\nu}(R_{t}^{A} + h(t)) \right\}
    \end{align*}
    where $c > 0$ is the parameter of the exploration bonus $f(t)$ when taking $b=1$. The reminder terms are: the approximation error $h(t) = O\left(\sqrt{t\ln(t)}\right)$, the learner's cumulative regret $R_{t}^{A}$ and a constant $C_{\nu}$ depending on the distribution. 
    
    Moreover, the instantiated CombGame meta-algorithm is an asymptotically optimal algorithm.
\end{theorem}

Even though the upper bound in Theorem \ref{thm:finite_time_upper_bound} holds for finite-time, it is an asymptotic result by nature. The additive term, which is independent of $\delta$, can't be neglected in finite-time, and is likely to be loose due to the analysis. Therefore, we won't compare the upper bounds of different learners.

\paragraph{Proof Scheme}

Detailed in Appendix \ref{appendix_proof_upper_bound}, the proof of Theorem \ref{thm:finite_time_upper_bound} uses Lemma \ref{lem:upper_bound_finite_time_decomposition}, which is an adaptation of Lemma 1 in \citet{degenne_non-asymptotic_2019} with the same exploration bonus.

\begin{lemma} \label{lem:upper_bound_finite_time_decomposition}
Let $(\mathcal{E}_t)_{t \geq 1}$ be a sequence of concentrations events for the exploration bonus $f$ with parameters $c>0$ and $b>0$: $\mathcal{E}_t \defeq \left\{\forall s \leq t, \forall a \in [d], \quad \tilde{N}_{s,a} d_{\text{KL}}(\mu_{s,a}, \mu_a) \leq f\left(t^{1/(1+b)}\right)\right\}$ for all $t \geq 1$. Suppose that there exists $T_{0}(\delta) \in \mathbb{N}$ such that for all $t > T_{0}(\delta)$, $\mathcal{E}_t \subset \{\tau_{\delta} \leq t\}$. Then 
\begin{equation*}
    \E_{\nu}[\tau_{\delta}] \leq T_{0}(\delta) + \sum_{t > T_{0}(\delta)} \mathbb{P}_{\nu}\left[\mathcal{E}_t^{\complement}\right] \quad \text{ where }\quad \sum_{t > T_{0}(\delta)}\mathbb{P}_{\nu} \left[\mathcal{E}_t^{\complement}\right] \leq \frac{2ed}{c^2}
\end{equation*}
\end{lemma}

The challenging part of the proof is the characterization of $T_{0}(\delta)$ with an equation involving the complexity $D_{\nu}$, similarly to Appendix D in \citet{degenne_non-asymptotic_2019}. We need to exhibit an upper bound $T_{0}(\delta)$ such that for $t \geq T_{0}(\delta)$, if $\mathcal{E}_t$ holds then the algorithm has already stopped, $\tau_{\delta} \leq t$. In contrast to \citet{degenne_non-asymptotic_2019}, the particularity of our proof is to consider computations on $\tsimplex$ and not on the simplex. Even though the idea of the proof is identical, we need different technical arguments such as the tracking and concentration results in Appendices \ref{tracking_results_technicalities} and \ref{concentration_results_technicalities}. For sake of simplicity we suppose that $I_t = I^*(\mu)$ in the following informal exposition. This fails only for $o(t)$ rounds as shown in Appendix \ref{appendix_incorrect_candidate_answer}. Using C-Tacking, we obtain that as long as the stopping criterion is not satisfied, under the concentration event $\mathcal{E}_{t-1}$,
\begin{equation*}
    \beta(t-1,\delta) \geq  \inf_{\lambda \in \Theta_{I^*(\mu)}^{\complement} } \langle \tilde{N}_{t-1}, d_{\text{KL}}(\mu_{t-1}, \lambda)\rangle  \geq  \inf_{\lambda \in \Theta_{I^*(\mu)}^{\complement} } \sum_{s = 1}^{t-1} \langle \tilde{w}_{s}, d_{\text{KL}}(\mu_{s-1}, \lambda)\rangle - O\left( \sqrt{t \ln(t)}\right) 
\end{equation*}

Then, we leverage the approximate saddle-point property of the CombGame meta-algorithm. This property is obtained by combining the optimism, the no-regret $\lambda$-player and the cumulative regret of the $A$-player, see Appendix \ref{saddle_point_property}:
\begin{equation*}
    \inf_{\lambda \in \Theta_{I^*(\mu)}^{\complement} } \sum_{s = 1}^{t-1} \langle \tilde{w}_{s}, d_{\text{KL}}(\mu_{s-1}, \lambda)\rangle \geq  \max_{A \in \A} \sum_{s = 1}^{t-1} \langle \1_{A}, r_s \rangle - O\left( \sqrt{t} \right)  - R_t^A  
\end{equation*}

Under the concentration event $\mathcal{E}_{t-1}$, the optimism implies $r_{s} \geq d_{\text{KL}}(\mu, \lambda_s)$ for $s \leq t-1$. Combining the dual formulation of $D_{\nu}$ and the average of diracs, $\frac{1}{t-1} \sum_{s = 1}^{t-1} \delta_{\lambda_s} \in \mathcal{P}\left( \Theta_{I^*(\mu)}^{\complement}\right)$, yields:\looseness=-1
\begin{equation*}
	\max_{A \in \A} \sum_{s = 1}^{t-1} \langle \1_{A}, d_{\text{KL}}(\mu, \lambda_s) \rangle \geq  t \inf_{q \in \mathcal{P}\left(\Theta_{I^*(\mu)}^{\complement}\right)}  \max_{A \in \A} \E_{\lambda \sim q} \left[ \langle \1_{A}, d_{\text{KL}}(\mu, \lambda) \rangle \right] =  t D_{\nu}
\end{equation*}

Combining all inequalities justifies the definition of $T_{0}(\delta)$ as the largest time such that the following inequality is satisfied: $T_{0}(\delta) = \max \left\{t \in \mathbb{N}: t \leq \frac{\beta(t,\delta)}{D_{\nu}} + C_{\nu}(R_{t}^{A} + h(t)) \right\}$. Taking the limit $\delta \rightarrow 0$ yields that the instances of CombGame are asymptotically optimal.

\section{Experiments}  \label{experiments}

The goal of our experiments is to validate the sample  effectiveness and computational efficiency of CombGame's instances for the finite-time regime, $\delta = 0.1$. We will compare the sample complexity of our learners, the uniform sampling and GCB-PE \citep{chen_combinatorial_2020}. To our knowledge, GCB-PE is the only algorithm which can be used to solve the pure-exploration problem for combinatorial bandits with semi-bandit feedback. Other works consider bandit feedback, cumulative regret or MAB. In addition, we demonstrate that learners on $\tsimplex$ have an \emph{exponentially smaller computational cost} compared to the learners on $\simplex$. As an illustrative example, we use the best-arm identification with batch size $k$ for a Gaussian bandit, $\nu = \mathcal{N}(\mu,\sigma^2 I_{d})$. In BAI the \textit{informative} actions are the ones containing the best arm $I^*$, $\A^* \defeq \{A \in \A: I^* \subset A\}$. They provide direct feedback on the best arm, while other actions are sampled to answer indirectly to our query. The batch setting is used in real-world applications and admits an efficient oracle, the greedy algorithm. The number of actions is $|\A| = \binom{d}{k} $ and the ratio of informative actions is $\frac{|\A^*|}{|\A|} = \frac{k}{d}$. By increasing the dimension $d$, we observe the effect of an exponential increase of $|\A|$ while the ratio of informative actions $\frac{|\A^*|}{|\A|}$ is decreasing harmonically. In Appendix \ref{appendix_experimental_setting_details}, additional experiments include BAI by playing paths in a graph (Figures \ref{fig:comparison_learners_grid_net} and \ref{fig:comparison_learners_line_net}).\looseness=-1

As described in Appendix \ref{appendix_implementation_details}, the empirical results of CombGame's instances are similar in behavior if we adopt D-Tracking instead of C-Tracking, one learner $\A^{A}$ instead of $|\I|$ learners, stylized stopping threshold $\beta(t, \delta) = \ln \left( \frac{1+\ln(t)}{\delta}\right)$ and exploration bonus $f(t) = \ln (t)$ instead of the ones licensed by theory. Doubling trick is used for Hedge and LLOO. The results over $750$ runs are summarized in Figure \ref{fig:comparison_learners} by plotting the mean of the empirical stopping time $\tau_{\delta}$ and the average running time to compute the next action. The error bars correspond to the first and third quartiles. 

\begin{figure}
    \centering
    \includegraphics[scale=0.17]{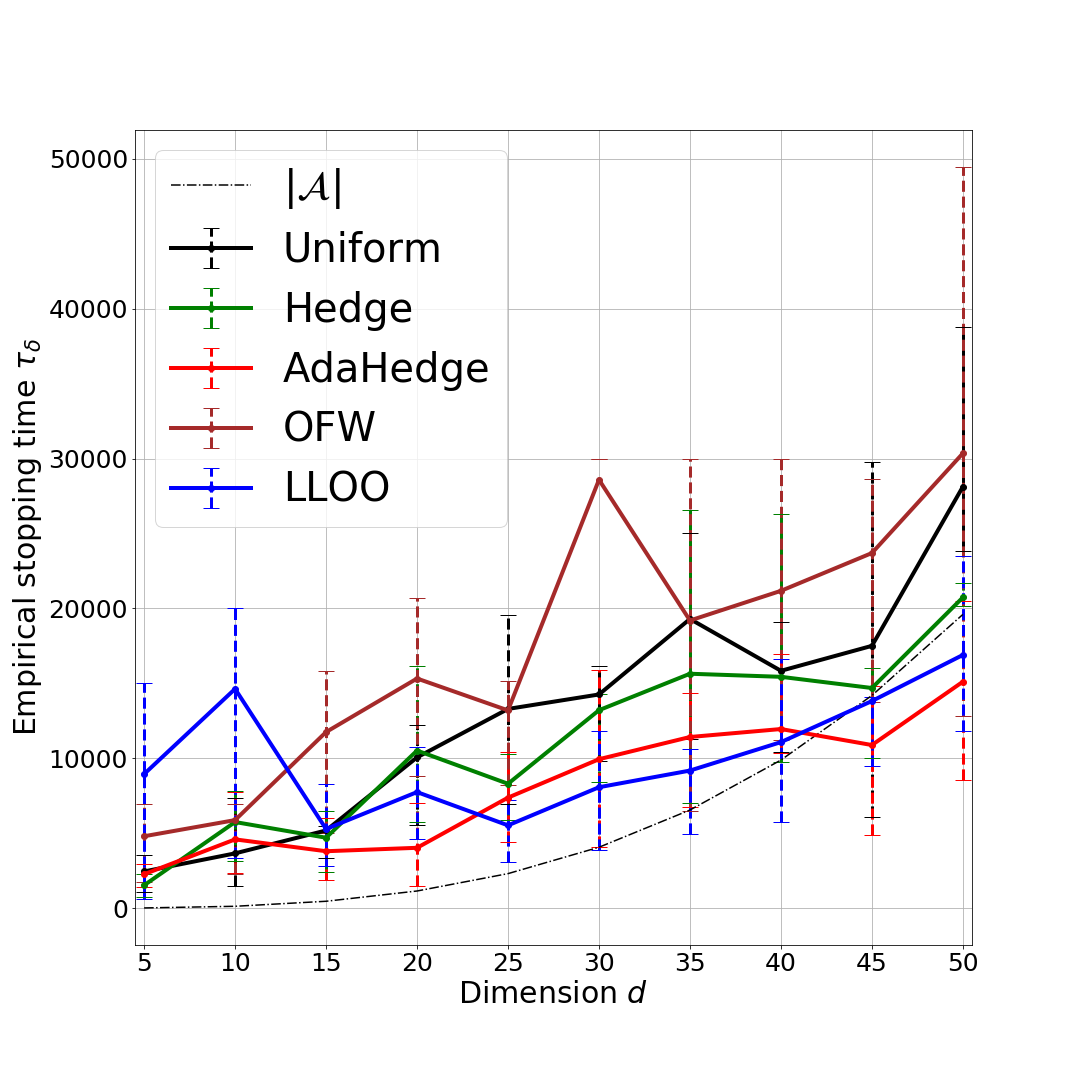}
    \includegraphics[scale=0.17]{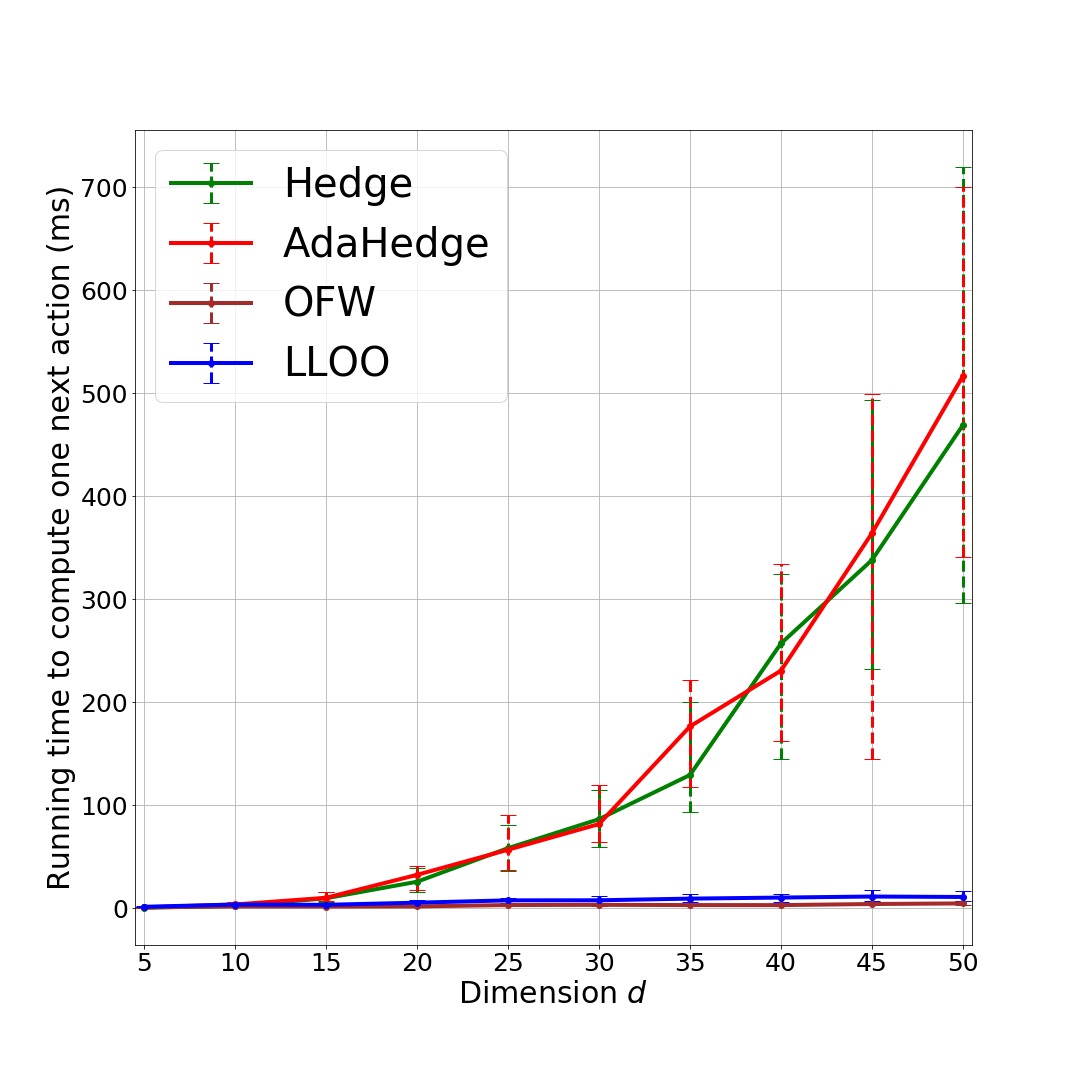}
    \caption{Uniform matroid, $k = 3$, where $\nu = \mathcal{N}(\mu,\sigma^2 I_{d})$ with $\mu$ as in Appendix \ref{appendix_experimental_setting_details_uniform_matroid} and $\sigma = 0.1$. (a) Empirical stopping time $\tau_{\delta}$. (b) Average running time to compute the next action. Note LLOO's competitive sample complexity for a low and constant computational cost.}
    \label{fig:comparison_learners}
\end{figure}

In Figure \ref{fig:comparison_learners}(a), we observe that the sample complexity of Hedge, AdaHedge and LLOO is similar and increases proportionally to the number of actions. They perform better than uniform sampling, which still works reasonably well thanks to the high number of informative actions when $k \ll d$, $|\A^*| = \binom{d-1}{k-1}$. OFW's sample complexity is significantly higher than previous algorithms. This highlights the importance of cumulative regret's guarantees in order to have competitive empirical performance. In Figure \ref{fig:comparison_learners_uniform_matroid}(c) in Appendix \ref{appendix_experimental_setting_details_uniform_matroid}, we empirically show that on this example, the sample complexity of GCB-PE is about an order of magnitude higher compared to the other sampling rules: the mean over $750$ runs of $\tau_{\delta}$ is $\{19914, 85314, 166179, 316552\}$ for $d \in \{ 5, 10, 15, 20\}$.

In Figure \ref{fig:comparison_learners}(b), the computational efficiency of the learners on the transformed simplex is striking when compared to the learners on the simplex. While the computational cost increases exponentially for Hedge and AdaHedge, it remains almost constant for OFW and LLOO. Uniform sampling has constant run time per round. As detailed in Appendix \ref{appendix_implementation_details}, the computational cost of GCB-PE is dominated by solving an NP-hard binary quadratic program in $\R^{m_{n_{0}}}$ with $m_{n_{0}} = \sum_{A \in B_{n_{0}}} |A| \geq d$. Since it has no efficient solver to our knowledge, the algorithm cannot run when $d$ is high. Therefore, we were unable to perform further experiments on GCB-PE.

Despite the fact that there is no clear-cut ranking between all algorithms in Figure \ref{fig:comparison_learners}(a), Figure \ref{fig:comparison_learners}(b) highlights that, with a greatly lower computational cost, we obtain similar sample complexity.

\section{Conclusion}   \label{conclusion}

In this paper we designed the first \emph{computationally efficient and asymptotically optimal} algorithm to solve best-arm identification with combinatorial actions and semi-bandit feedback. 

We highlight two directions to improve on our work. First, due to the learner's central role in the empirical performance, a more thorough benchmark of the existing learners should be made. An interesting choice is SFTPL from \citet{hazan_faster_2020} which meets our requirements. Second, the best-reponse oracle used by the $\lambda$-player is not computationally efficient for combinatorial \emph{answer sets}, as in best-action identification, since the computations per round scale with $|N(I_t)|$ (which is usually lower than $|\I|$). In the spirit of \citet{fiez_sequential_2019,zaki_explicit_2020}, this flaw could be mitigated by considering a phase-based algorithm discarding suboptimal answers.

Addressing a richer bandit structure where the arms are correlated is yet another avenue. Extending our approach to correlated Gaussian with known covariance matrix $\Sigma$ is straightforward. The \textit{correlated transformed simplex} is a subset of the cone of symmetric positive semi-definite matrices: $\text{conv}(\{V_{\delta_{A}}\}_{A \in \A})$ where $V_{\delta_{A}} \defeq S_{A}^{\intercal}\left(S_{A} \Sigma S_{A}^{\intercal}\right)^{-1}S_{A}$ and $S_{A} \defeq \left( \1_{(\tilde{a} = a)}\right)_{\tilde{a} \in A, a \in [d]}$. Unfortunately, the oracle has the form $\argmin_{A \in \A} \text{Tr}(V_{\delta_{A}}^{\intercal} C)$ for a cost matrix $C \in \R^{d \times d}$. To our knowledge, there is no computationally efficient oracle for this linear optimization over matrices. Therefore, it is not clear how and to what extent we can conserve the computational efficiency of our sampling rules.

Finally, as already noted in \citet{degenne_gamification_2020}, we observed that the stopping threshold is the major bottleneck in terms of finite-time empirical sample complexity. Using thresholds guarantying $\delta$-PAC algorithms is too conservative since empirical error rates are orders of magnitude below the theoretical confidence error $\delta$.

\acks{This project has received funding from the European Research Council (ERC) under the European Unions Horizon 2020 research and innovation program grant agreement No 815943. It was also supported by the Swiss National Science Foundation through the 
NCCR Catalysis.}

\bibliography{bibliography}

\clearpage
\appendix

\section{Notation}  \label{appendix_notations}

We recall some commonly used notations: the set of base arms $a \in [d] \defeq \{1, \cdots, d\}$, the indicator vector $\1_{S} \defeq (\1_{(a \in S)})_{a \in [d]}$ for a subset $S \subset [d]$, the symmetric difference of two sets $A \triangle B \defeq \left(A \setminus B \right) \cup \left( B \setminus A \right)$, the euclidean inner-product $\langle x, y \rangle \defeq \sum_{a \in [d]} x_a y_a$, the support of a vector supp$(x) \defeq \{ a \in [d]: x_a \neq 0\}$, the set of probability distributions $\mathcal{P}(\mathcal{X})$ over $\mathcal{X}$ and the $n$-dimensional probability simplex $\Delta_{n} \defeq \left\{x \in \R^{n}: x \geq 0 \land \langle 1_{d}, x \rangle = 1 \right\}$. In Table \ref{tab:notation_table}, we summarize problem-specific notations.

\begin{table}[h]
\begin{center}
     \begin{tabular}{l | l}     
     Notation & Meaning \\
     \hline
     $(\nu_a)_{a \in [d]}$ & distributions of base arms \\
     $\mu$ & bandit mean parameter, $(\mu_a)_{a \in [d]}$  \\
     $\M$ & set of possible parameters \\
     $d_{\text{KL}}(\mu, \lambda)$ & component-wise KL divergence between two parameters, $\left(d_{\text{KL}}(\mu_a, \lambda_a)\right)_{a \in [d]}$ \\
     $K$ & maximum size of an action, $\max_{A \in \A}|A|$ \\
     $A_t$ & sampled action at time $t$\\
     $Y_{t, A_t}$ & semi-bandit feedback at time $t$, $(Y_{t, a})_{a \in A_t}$ \\
     $\F_t$ & history up to time $t$, $\sigma(A_{1}, Y_{1, A_{1}}, \cdots, A_{t}, Y_{t, A_{t}})$ \\
     $\mu_t$ & maximum likelihood estimator, $\left( \frac{1}{\tilde{N}_{t-1,a}} \sum_{s = 1}^{t-1} \bm{1}_{(a \in A_s)} Y_{s,a}\right)_{a \in [d]}$ \\
     $\C_t$ & confidence region associated to $\mu_t$ \\
     $\tilde{\mu}_t$ & projection of $\mu_t$ onto $\M$ \\
     $I^*(\lambda)$ & unique correct answer for parameter $\lambda$, $\argmax_{I \in \I } \langle \lambda, \1_{I}\rangle$\\
     $I_t$ & recommended answer at time $t$\\
     $\Theta_{I}$ & cell $I$, $\{ \lambda \in \M: I^*(\lambda) = I\}$ \\
     $\Bar{\Theta}_{J}^{I}$ & set of parameters for which $J$ outperforms $I$, $\{ \lambda \in \M : \langle \1_{J} - \1_{I}, \lambda\rangle \geq 0\}$ \\
     $N(I)$ & neighbors to $I$, $\{ J \in \I: \partial \Theta_{I} \cap \partial \Theta_{J} \neq \emptyset \}$ \\
     $\tau_{\delta}$ & stopping time for confidence $\delta$\\
     $\tsimplex$ & transformed simplex\\
     $D_{\nu}$ & complexity for the bandit $\nu$\\
     $N_t, \tilde{N}_t$ & empirical count of sampled actions and its equivalent for arms, $W_{\A} N_t$\\
     $w_t, \tilde{w}_t$ & pulling distribution over actions and its equivalent for arms, $W_{\A} w_t$\\
     $\beta(t, \delta)$ & stopping threshold at time $t$ for confidence $\delta$ \\
     $f(t)$ & exploration bonus at time $t$ \\
     $r_t$ & optimistic reward \\
     $U_t$ & extended optimistic reward \\
     $R_{t}^{A}$ & cumulative regret of the learner of the $A$-player \\
     \hline
    \end{tabular}
\end{center}
	\caption{Table of notations}
    \label{tab:notation_table}
\end{table}

\section{Outline}

The appendices are organized as follows:
\begin{itemize}
    \item The proof of Theorem \ref{thm:lower_bound} is detailed in Appendix \ref{appendix_proof_lower_bound}.
    \item The proof of Theorem \ref{thm:delta_pac} is detailed in Appendix \ref{appendix_proof_delta_pac}.
    \item The results concerning the optimistic reward are detailed in Appendix \ref{appendix_optimistic_reward}: bounds on $\|r_t\|_{\infty}$ and explicit formulas for Gaussian bandit.
    \item The upper bounds on the learners' cumulative regret are proven in Appendix \ref{appendix_learner_cumulative_regret}.
    \item The full proof of Theorem \ref{thm:finite_time_upper_bound} is detailed in Appendix \ref{appendix_proof_upper_bound}.
    \item In Appendix \ref{appendix_unbounded_setting}, we sketch why the boundedness assumption is immaterial for Gaussian bandit.
    \item The implementation details for the experiments are presented in Appendix \ref{appendix_implementation_details}. Additional empirical results are also displayed.
\end{itemize}

\section{Proof of Theorem \ref{thm:lower_bound}} \label{appendix_proof_lower_bound}

Let kl$(x,y)$ be the KL divergence of a Bernoulli distribution. Let $\nu$ and $\nu'$ be two bandit models such that for all $a \in [d]$ the distributions $\nu_a$ and $\nu_a'$ are mutually absolutely continuous. The associated density are denoted $f_{\nu_a}$ and $f_{\nu_a'}$. Given the history up to time $t$, the log-likelihood ratio of the independent observations is:
\begin{equation*}
    L_t = L_t(A_1, Y_{1,A_1}, \cdots, A_{t}, Y_{t,A_t}) \defeq \sum_{a \in [d]} \sum_{s \in [t]: a \in A_s} \ln \left( \frac{f_{\nu_a}(Y_{s,a})}{f_{\nu_a'}(Y_{s,a})}\right)
\end{equation*}

The proof of Theorem \ref{thm:lower_bound} is an adaptation of the proof of Theorem 1 in \citet{garivier_optimal_2016} to our setting. We use Lemma 19 of \citet{kaufmann_complexity_2016}, which shows a lower bound on the expectation of the log-likelihood ratio. Combined with Wald's lemma, we obtain the transportation inequality of Lemma \ref{lem:transportation_lemma}, which replaces the Lemma 1 in \citet{kaufmann_complexity_2016}.

\begin{lemma*}{(Lemma 19 in \citet{kaufmann_complexity_2016})} \label{lem:19_kaufmann}
    Let $\tau$ be the almost-surely finite stopping time with respect to the filtration $(\F_t)_{t \geq 1}$. For every event $\mathcal{E} \in \F_{\tau}$,
    \begin{equation*}
        \E_{\nu}\left[L_{\tau}\right] \geq \text{kl}(\mathbb{P}_{\nu}(\mathcal{E}),\mathbb{P}_{\nu'}(\mathcal{E}))
    \end{equation*}
\end{lemma*}

\begin{lemma} \label{lem:transportation_lemma}
Let $\nu$ and $\nu'$ be two bandit models with independent arms. For any almost-surely finite stopping time $\tau$ with respect to the filtration $(\mathcal{F}_t)_{t \geq 1}$,
\begin{equation*}
    \sum_{a \in [d]} \E_{\nu}[\tilde{N}_{\tau,a}] d_{\text{KL}}(\nu_{a},\nu_{a}') \geq \sup_{\mathcal{E} \in \mathcal{F}_{\tau}} \text{kl}(\mathbb{P}_{\nu}(\mathcal{E}),\mathbb{P}_{\nu'}(\mathcal{E}))
\end{equation*}
\end{lemma}
\begin{proof}
For all $a \in [d]$, we denote $(Y_{s,a})_{s \in [t]: a \in A_s}$ the sequence of $\tilde{N}_{t,a}$ i.i.d. samples observed for the arm $a$. By definition of $L_{\tau}$, the fact that $d_{\text{KL}}(\nu_{a},\nu_{a}') \defeq \E_{\nu} \left[ \ln \left( \frac{f_{\nu_a}(Y_{s,a})}{f_{\nu_a'}(Y_{s,a})}\right) \right]$ and applying Wald's lemma to $L_{\tau}$, we obtain that: for all $a \in [d]$, $\E_{\nu}\left[L_{\tau}\right] = \sum_{a \in [d]} \E_{\nu}[\tilde{N}_{\tau,a}] d_{\text{KL}}(\nu_{a},\nu_{a}')$. Combining this equation with Lemma 19 of \citet{kaufmann_complexity_2016} yields the desired result.
\end{proof}

\begin{theorem*} 
For any $\delta$-PAC strategy and any bandit $\nu$,
\begin{align*}
    \frac{\E_{\nu}[\tau_{\delta}]}{\ln ( 1/(2.4\delta))} \geq D_{\nu}^{-1}  \quad \text{ and } \quad
    \limsup_{\delta \rightarrow 0} \frac{\E_{\nu}[\tau_{\delta}]}{\ln (1/\delta)} \geq D_{\nu}^{-1} 
\end{align*}
where the {\em complexity} $D_{\nu}$, inverse of a characteristic time, is defined by
\begin{equation*}
    D_{\nu} \defeq \max_{\tilde{w} \in \tsimplex} \inf_{\lambda \in \Theta_{I^*(\mu)}^{\complement}} \langle \tilde{w}, d_{\text{KL}}(\mu, \lambda) \rangle
\end{equation*}
\end{theorem*}
\begin{proof}
Let $\delta \in (0,1)$, $\nu$ a bandit with parameter $\mu \in \M$ and consider a $\delta$-PAC strategy. Let $\lambda \in \Theta_{I^*(\mu)}^{\complement}$ be the parameter of a bandit $\nu'$ with a unique correct answer $J \neq I^*(\mu)$. Let $\mathcal{E}_{J} \defeq \{ I_{\tau_{\delta}} = J\} \in \F_{\tau_{\delta}}$ be the event in which we recommend $J$ instead of $I^*(\mu)$ at the stopping time. Since the strategy is $\delta$-PAC, we have: $\mathbb{P}_{\nu}(\mathcal{E}_{J}) \leq \delta$ and $\mathbb{P}_{\nu'}(\mathcal{E}_{J}) \geq 1 - \delta$. Therefore, we have:
\begin{equation*}
    \sup_{\mathcal{E} \in \F_{\tau_{\delta}}} \text{kl}(\mathbb{P}_{\nu}(\mathcal{E}),\mathbb{P}_{\nu'}(\mathcal{E})) \geq \text{kl}(\mathbb{P}_{\nu}(\mathcal{E}_{J}),\mathbb{P}_{\nu'}(\mathcal{E}_{J})) \geq \text{kl}(\delta,1-\delta) \geq \ln ( 1/(2.4\delta))
\end{equation*}

The last inequality $\text{kl}(\delta,1-\delta) = \delta \ln \left( \frac{\delta}{1- \delta} \right) + (1-\delta) \ln  \left( \frac{1- \delta}{\delta} \right) \geq \ln ( 1/(2.4\delta))$ was shown in \citet{kaufmann_complexity_2016}. Combined with Lemma \ref{lem:transportation_lemma}, we obtain: $\sum_{a \in [d]} \E_{\nu}[\tilde{N}_{\tau_{\delta},a}] d_{\text{KL}}(\mu_{a},\lambda_{a}) \geq \ln ( 1/(2.4\delta))$ for all $\lambda \in \Theta_{I^*(\mu)}^{\complement}$. By construction, we have $\frac{\E_{\nu}[\tilde{N}_{\tau_{\delta},a}]}{\E_{\nu}[\tau_{\delta}]} \in \tsimplex$. Instead of considering a specific alternative bandit minimizing the lower bound, we combine all the inequalities. Taking the infimum:
\begin{align*}
    \ln ( 1/(2.4\delta)) &\leq \E_{\nu}[\tau_{\delta}] \inf_{\lambda \in \Theta_{I^*(\mu)}^{\complement}} \sum_{a \in [d]} \frac{\E_{\nu}[\tilde{N}_{\tau_{\delta},a}]}{\E_{\nu}[\tau_{\delta}]} d_{\text{KL}}(\mu_{a},\lambda_{a})  \\
    &\leq \E_{\nu}[\tau_{\delta}] \sup_{\tilde{w} \in \tsimplex} \inf_{\lambda \in \Theta_{I^*(\mu)}^{\complement}} \sum_{a \in [d]} \tilde{w}_{a} d_{\text{KL}}(\mu_{a},\lambda_{a}) = \E_{\nu}[\tau_{\delta}] D_{\nu}
\end{align*}

This concludes the proof of the finite-time lower bound. Taking the limit $\delta \rightarrow 0$ in the previous lower bound yields directly the asymptotic lower bound.
\end{proof}

\section{Proof of Theorem \ref{thm:delta_pac}}  \label{appendix_proof_delta_pac}

The proof of Theorem \ref{thm:delta_pac} uses the deviation inequalities of \citet{kaufmann_mixture_2018}, see Appendix \ref{deviation_inequalities}. The idea of the proof is similar to the proof of Proposition 21 in \citet{kaufmann_mixture_2018}, as well as Theorem 2 in \citet{shang_fixed-confidence_2019}. 

\begin{theorem*}  
Let $\mathcal{M}$ be bounded. Regardless of the sampling rule, a strategy using the frequentist recommendation/stopping pair with the stopping threshold:
\begin{align*}
\beta (t, \delta) &\defeq 
\begin{cases} 
    3d_{0} \ln \left(1 + \ln \left(\frac{t K}{d_{0}}\right)\right) + d_{0}\T\left(\frac{\ln \left( \frac{|\I|-1}{\delta}\right)}{d_{0}}\right) & \text{for (a)}\\
     2 d_{0}\ln \left(4 + \ln \left(\frac{t K}{d_{0}}\right)\right) + d_{0} \C^{g_G}\left(\frac{\ln \left( \frac{|\I|-1}{\delta}\right)}{d_{0}}\right) & \text{for (b)}
    \end{cases} 
\end{align*}
is $\delta$-PAC. In the above, $d_{0} \defeq \max_{I,J \in \I, J \neq I} |(I \setminus J) \cup (J \setminus I)|$, $\T$ and $\C^{g_G}$ are the functions defined in \citet{kaufmann_mixture_2018}, $\C^{g_G}(x) \approx x + \ln(x)$ and $\T(x) \approx x + 4 \ln(1+x + \sqrt{2x})$ for $x \geq 5$.\looseness=-1
\end{theorem*}

\begin{proof}
First let's show that $\tau_{\delta} < \infty$. We recall the following expressions: $I_t = \argmax_{I \in \I} \langle \1_{I}, \tilde{\mu}_{t-1} \rangle$ where $\tilde{\mu}_{t-1} \in \argmin_{\lambda \in \M \cap \C_{t}} \langle \tilde{N}_{t-1}, d_{\text{KL}}(\mu_{t-1}, \lambda) \rangle$ and
\begin{align*}
    \tau_{\delta} &= \inf \left\{t \in \mathbb{N}: \min_{J \in N(I_t)} \inf_{\lambda \in \Bar{\Theta}_{J}^{I_t} } \langle \tilde{N}_{t-1}, d_{\text{KL}}(\mu_{t-1}, \lambda)\rangle > \beta(t-1,\delta)\right\}
\end{align*}

Let an arbitrary sampling rule, the set of arms sampled only a finite time, $\mathcal{U} \defeq \{a \in [d]: \lim_{t \rightarrow \infty}\tilde{N}_{t,a} < + \infty\}$, and the limit of the empirical sampling rate, $\tilde{w}_{\infty} \defeq \left(\lim_{\infty} \frac{\tilde{N}_{t,a}}{t} \right)_{a \in [d]}$. For all $a \in \mathcal{U}^{\complement}$, the law of large number proves that $\mu_{t} \rightarrow_{\infty} \mu_a$, while for all $a \in \mathcal{U}$, $\mu_{t} \rightarrow_{\infty} \tilde{\mu}_a \neq \mu_a$. Since it is a basic requirement for a sampling rule to predict the unique correct answer, we consider only sampling rules satisfying $\lim_{\infty} I_t = I^*(\mu)$. Since $\mu_t \rightarrow \mu$ implies $\tilde{\mu}_t \rightarrow \mu \in \M$ and $\lim_{\infty} I_t = I^*(\mu)$, this condition is weaker than assuming the convergence of the parameter $\mu_t$ towards the true parameter, which happens if $\mathcal{U} = \emptyset$. Let $T_{0} \in \N$ such that for all $t \geq T_{0}$, $I_t = I^*(\mu)$. For $t \geq T_{0}$, the stopping condition rewrites as: $\inf_{\lambda \in \Theta_{I^*(\mu)}^{\complement} } \langle \frac{\tilde{N}_{t-1}}{t}, d_{\text{KL}}(\mu_{t-1}, \lambda)\rangle > \frac{\beta(t-1,\delta)}{t}$. By continuity, dominated convergence (to invert $\lim$ and $\inf$ for $\M$ bounded), and using that $\beta(t, \cdot) \sim_{\infty} c_{d} \ln ( \ln (t) )$, taking the limit on both side yields: $\inf_{\lambda \in \Theta_{I^*(\mu)}^{\complement} }  \sum_{a \in \mathcal{U}^{\complement}} \tilde{w}_{\infty, a} d_{\text{KL}}(\mu_{a}, \lambda_a) \geq 0$.

By construction, we have $\tilde{w}_{\infty} \in \tsimplex$, hence the left term is strictly positive if: for all $a \in [d]$ such that $\tilde{w}_{\infty, a} \neq 0$, we have $d_{\text{KL}}(\mu_{a}, \lambda_a) \neq 0$. Since $d_{\text{KL}}(\mu_{a}, \lambda_a) = 0$ if and only if $\mu_{a}=\lambda_a$, the fact that $\lambda \in \Theta_{I^*(\mu)}^{\complement}$ allows us to conclude that the inequality is strict. Therefore, there exists a finite time such that the stopping condition is met: $\tau_{\delta} < \infty$.

Second, let's show that $\mathbb{P}_{\nu} \left[ I_{\tau_{\delta}} \neq I^*(\mu) \right] \leq \delta$. Let $\beta(t,\delta)$ be an arbitrary stopping threshold. Since $\{I_{\tau_{\delta}} \neq I^*(\mu)\} = \bigcup_{I \neq  I^{*}(\mu)} \{\exists t \in \mathbb{N}: I_{t+1} = I \land \inf_{\lambda \in \Theta_{I_{t+1}}^{\complement}} \langle \tilde{N}_{t}, d_{\text{KL}}(\mu_{t}, \lambda)\rangle > \beta(t,\delta)\}$, the union bound yields:
\begin{align*}
    \mathbb{P}_{\nu} \left[ I_{\tau_{\delta}} \neq I^*(\mu) \right] & \leq \sum_{I \neq  I^{*}(\mu)} \mathbb{P}_{\nu} \left[ \exists t \in \mathbb{N}: \inf_{\lambda \in \Theta_{I}^{\complement} } \langle \tilde{N}_{t}, d_{\text{KL}}(\mu_{t}, \lambda)\rangle > \beta(t,\delta)\right]
\end{align*}

Let $I \neq I^*(\mu)$. Since $\Theta_{I}^{\complement} = \bigcup_{J \neq I} \Bar{\Theta}_{J}^{I}$, we have $\inf_{\lambda \in \Theta_{I}^{\complement}} h(\lambda) = \min_{J \neq I} \inf_{\lambda \in \Bar{\Theta}_{J}^{I}} h(\lambda) \leq \inf_{\lambda \in \Bar{\Theta}_{I^*(\mu)}^{I}} h(\lambda)$. Using $h(\lambda) = \langle \tilde{N}_{t}, d_{\text{KL}}(\mu_{t}, \lambda)\rangle$, we obtain:
\begin{align*}
    \inf_{\lambda \in \Theta_{I}^{\complement} } \langle \tilde{N}_{t}, d_{\text{KL}}(\mu_{t}, \lambda)\rangle \leq \inf_{\lambda \in \Bar{\Theta}_{I^*(\mu)}^{I} }\langle \tilde{N}_{t}, d_{\text{KL}}(\mu_{t}, \lambda)\rangle \leq \sum_{a \in I^*(\mu) \triangle I} \tilde{N}_{t,a} d_{\text{KL}}(\mu_{t,a}, \mu_{a})
\end{align*}

The last inequality is obtained by considering a parameter $\lambda$ defined as: $\lambda_a = \mu_a$ when $a \in I^*(\mu) \triangle I$ and $\lambda_a = \mu_{t, a}$ else. Since $\mu \in \Bar{\Theta}_{I^*(\mu)}^{I}$, we have that $\langle \1_{I^*(\mu)} - \1_{I}, \lambda_a \rangle = \langle \1_{I^*(\mu)} - \1_{I}, \mu_a \rangle \geq 0$. Therefore $\lambda \in \Bar{\Theta}_{I^*(\mu)}^{I}$, hence it is a valid parameter. The upper bound rewrites as:
\begin{align*}
    \mathbb{P}_{\nu} \left[ I_{\tau_{\delta}} \neq I^*(\mu) \right] & \leq \sum_{I \neq  I^{*}(\mu)} \mathbb{P}_{\nu} \left[ \exists t \in \mathbb{N}: \sum_{a \in I^*(\mu) \triangle I} \tilde{N}_{t,a} d_{\text{KL}}(\mu_{t,a}, \mu_{a}) > \beta(t,\delta)\right]
\end{align*}

To conclude, we need to control the deviation of the \textit{self-normalized sums}, $\sum_{a \in S} \tilde{N}_{t,a} d_{\text{KL}}(\mu_{t,a}, \mu_{a})$ for $S \subset [d]$, thanks to concentration inequalities, which are uniform in time. The concentration inequality depends on the setting (a) sub-Gaussian bandit or (b) Gaussian bandit. Moreover, we want an expression for $\beta(t, \delta)$ which doesn't depend on the answer $I$ or the empirical count $\tilde{N}_{t}$. Let $d_{0} \defeq \max_{I,J \in \I, J \neq I} |I \triangle J|$. Combining the concavity of $x \mapsto \ln(c + \ln(x))$ and the fact that $\sum_{a \in I^*(\mu) \triangle I*} \tilde{N}_{t,a} \leq \sum_{a \in [d]} \tilde{N}_{t,a} \leq t K$, we obtain:
\begin{align*}
     \sum_{a \in I^*(\mu) \triangle I} \ln(c + \ln(\tilde{N}_{t,a})) &\leq |I^*(\mu) \triangle I|\ln \left(c + \ln \left(\frac{\sum_{a \in I^*(\mu) \triangle I} \tilde{N}_{t,a}}{|I^*(\mu) \triangle I|}\right)\right) \\
     &\leq |I^*(\mu) \triangle I|\ln \left(c + \ln \left(\frac{t K}{|I^*(\mu) \triangle I|}\right)\right)\leq d_{0} \ln \left(c + \ln \left(\frac{t K}{d_{0}}\right)\right)
\end{align*}

The last inequality is due to the fact that $h_{t}(x)=x \ln \left(c_{0} + \ln \left(\frac{t}{x}\right)\right)$ is increasing on $]0,d]$ when $t \gg d$. The higher $t$ is, the longer $h_{t}$ is increasing. Numerically, $h_1$ is increasing till $424$. Let $\T$ and $\C^{g_G}$ the functions defined in \citet{kaufmann_mixture_2018}. Since $x \mapsto x \T\left(\frac{c}{x}\right)$ and $x \mapsto x \C^{g_G}\left(\frac{c}{x}\right)$ are increasing (Appendix \ref{deviation_inequalities}), we obtain:
\begin{align*}
    |I^*(\mu) \triangle I| \T\left(\frac{\ln \left( \frac{|\I|-1}{\delta}\right)}{|I^*(\mu) \triangle I|}\right) &\leq d_{0} \T\left(\frac{\ln((|\I|-1)/\delta)}{d_{0}}\right) \\
    |I^*(\mu) \triangle I| \C^{g_{G}}\left(\frac{\ln \left( \frac{|\I|-1}{\delta}\right)}{|I^*(\mu) \triangle I|}\right) &\leq d_{0} \C^{g_G}\left(\frac{\ln((|\I|-1)/\delta)}{d_{0}}\right) 
\end{align*}

Combining those inequalities with $c=1$ for (a) and $c=4$ for (b), we obtain that:
\begin{align*}
    3\sum_{a \in I^*(\mu) \triangle I} \ln(1 + \ln(\tilde{N}_{t,a})) + |I^*(\mu) \triangle I| \T\left(\frac{\ln \left( \frac{|\I|-1}{\delta}\right)}{|I^*(\mu) \triangle I|}\right) &\leq 3d_{0} \ln \left(1 + \ln \left(\frac{tK}{d_{0}}\right)\right) \\&\quad \quad +  d_{0} \T\left(\frac{\ln((|\I|-1)/\delta)}{d_{0}}\right) \\
    2\sum_{a \in I^*(\mu) \triangle I} \ln(4 + \ln(\tilde{N}_{t,a})) + |I^*(\mu) \triangle I| \C^{g_{G}}\left(\frac{\ln \left( \frac{|\I|-1}{\delta}\right)}{|I^*(\mu) \triangle I|}\right) &\leq 2d_{0} \ln \left(4 + \ln \left(\frac{tK}{d_{0}}\right)\right)  \\&\quad \quad + d_{0} \C^{g_G}\left(\frac{\ln((|\I|-1)/\delta)}{d_{0}}\right) 
\end{align*}

Let $\beta (t, \delta)$ be the stopping threshold defined as: 
\begin{align*}
\beta (t, \delta) &\defeq 
\begin{cases} 
    3d_{0} \ln \left(1 + \ln \left(\frac{tK}{d_{0}}\right)\right) + d_{0}\T\left(\frac{\ln \left( \frac{|\I|-1}{\delta}\right)}{d_{0}}\right) & \text{for (a)}\\
     2 d_{0}\ln \left(4 + \ln \left(\frac{tK}{d_{0}}\right)\right) + d_{0} \C^{g_G}\left(\frac{\ln \left( \frac{|\I|-1}{\delta}\right)}{d_{0}}\right) & \text{for (b)}
    \end{cases} 
\end{align*}

Since $x \leq y$ implies $\mathbb{P}[X > y] \leq \mathbb{P}[X > x] $, using Theorem 14 in \citet{kaufmann_mixture_2018} and Corollary 10 in \citet{kaufmann_mixture_2018} (Appendix \ref{deviation_inequalities}) yields the result: 
\begin{align*}
    \mathbb{P}_{\nu} \left[ I_{\tau_{\delta}} \neq I^*(\mu) \right] & \leq \sum_{I \neq  I^{*}(\mu)} \frac{\delta}{|\I|-1} = \delta
\end{align*}

Therefore, we conclude that a strategy using the frequentist recommendation/stopping pair is $\delta$-PAC.
\end{proof}

\subsection{Deviation Inequalities}  \label{deviation_inequalities}

The deviation inequality for sub-Gaussian bandit is rewritten in Appendix \ref{deviation_inequalities_subgaussian}, while the deviation inequality for Gaussian bandit is presented in Appendix \ref{deviation_inequalities_gaussian}.

\subsubsection{Sub-Gaussian Bandit}   \label{deviation_inequalities_subgaussian}

Theorem 14 in \citet{kaufmann_mixture_2018} holds for sub-Gaussian bandits.

\begin{lemma*}[Theorem 14 in \citet{kaufmann_mixture_2018}] \label{lem:theorem_14_mixture_martingale}
Let $\delta >0$, $\nu$ be independent one-parameter exponential families with mean $\mu$ and $S \subset [d]$. Then we have,
\begin{equation*}
    \mathbb{P}_{\nu}\left[\exists t \in \mathbb{N}: \sum_{a \in S} \tilde{N}_{t,a} d_{\text{KL}}(\mu_{t,a}, \mu_a) \geq \sum_{a \in S} 3 \ln(1 + \ln(\tilde{N}_{t,a})) + |S|\T\left(\frac{\ln \left( \frac{1}{\delta}\right)}{|S|} \right) \right]  \leq  \delta
\end{equation*}
where $\T : \mathbb{R}^+ \rightarrow \mathbb{R}^+$ is such that $\T(x) = 2\tilde{h}_{3/2}\left(\frac{h^{-1}(1+x) + \ln\left(\frac{\pi^2}{3}\right)}{2}\right)$ with:
\begin{align*}
    \forall u \geq 1, \quad h(u) &= u - \ln(u)\\
    \forall z \in [1,e], \forall x \geq 0, \quad \tilde{h}_{z}(x) &= \begin{cases}
        \exp\left(\frac{1}{h^{-1}(x)}\right) h^{-1}(x)& \text{if } x \geq h^{-1}\left(\frac{1}{\ln(z)}\right)\\
       z(x-\ln(\ln(z)))&\text{else}
    \end{cases}
\end{align*}
\end{lemma*}

\subsubsection{Gaussian Bandit}  \label{deviation_inequalities_gaussian}

Corollary 10 in \citet{kaufmann_mixture_2018} holds for Gaussian bandits. Lemma \ref{lem:properties_of_c_g_gaussian} gathers some properties of $\C^{g_{G}}$. 

\begin{lemma*}[Corollary 10 in \citet{kaufmann_mixture_2018}] \label{lem:corollary_10_mixture_martingale}
Let $\delta >0$, $\nu$ be a family of independent Gaussian with mean $\mu$ and $S \subset [d]$. Then we have,
\begin{equation*}
    \mathbb{P}_{\nu}\left[\exists t \in \mathbb{N}: \sum_{a \in S} \tilde{N}_{t,a} d_{\text{KL}}(\mu_{t,a}, \mu_a) \geq \sum_{a \in S} 2 \ln(4 + \ln(\tilde{N}_{t,a})) + |S|\C^{g_{G}}\left(\frac{\ln \left( \frac{1}{\delta}\right)}{|S|} \right) \right]  \leq  \delta
\end{equation*}
where $\C^{g_{G}}(x) = \min_{y \in ]1/2,1]} \frac{g_{G}(y) + x}{y} $ with $g_{G} : ]1/2,1] \rightarrow \mathbb{R}$ such that $g_{G}(y) = 2y - 2y\ln(4y) + \ln(\zeta(2y)) - \frac{1}{2}\ln(1-y)$
\end{lemma*}

\begin{lemma} \label{lem:properties_of_c_g_gaussian}
The $g_{G}$ is positive on $]1/2,1[$ and satisfies $g_{G}(y) \rightarrow_{y \rightarrow \{1/2,1\}} + \infty$. The function $x \mapsto x C^{g_G}\left(\frac{c}{x}\right)$ is increasing.
\end{lemma}
\begin{proof}
Since $\zeta(1) = \lim_{n \rightarrow \infty} \sum_{s=1}^{n} \frac{1}{s} = + \infty$ and $g_{G}(y) \sim_{1/2} \ln(\zeta(2y))$, we have $g_{G}(y) \rightarrow_{y \rightarrow 1/2} + \infty$. Since $\zeta(2) = \frac{\pi^2}{6}$ and $g_{G}(y) \sim_{1} -\frac{1}{2}\ln(1-y)$, we have $g_{G}(y) \rightarrow_{y \rightarrow 1} + \infty$. 

Let $y \in ]1/2,1[$ and $h(y) = 2y - 2y\ln(4y) - \frac{1}{2}\ln(1-y)$. We have $h'(y) =  \frac{1}{2} \frac{1}{1-y} - 2 \ln(4y)$, hence $h'(y) \geq 0$ if and only if $1 \geq 4 (1-y) \ln(4y)$. Numerically, this condition is always true, hence $h$ is increasing. Since $h(y) \geq h(1/2)= 1$, we obtain $g_{G}(y) \geq 1 + \ln(\zeta(2y))$. Using that $\ln(\zeta(2)) = \ln\frac{\pi^2}{6} > 0$ and $x \mapsto \zeta(x)$ decreasing on $]1/2,1]$, we obtain that $\ln(\zeta(2y)) \geq \ln(\zeta(2))$. Therefore we can conclude that $\forall y \in [1/2,1[, g_{G}(y) \geq 1 \geq 0$.

Since $x C^{g_G}\left(\frac{c}{x}\right) = \min_{y \in ]1/2,1]} \frac{xg_{G}(y) + c}{y}$ and $g_{G}$ is positive on $]1/2,1[$, we obtain directly that $x \mapsto x C^{g_G}\left(\frac{c}{x}\right)$ is increasing.
\end{proof}

\section{Optimistic Reward} \label{appendix_optimistic_reward}

In Appendix \ref{proof_loss_inf_norm_upper_lower_bound}, we prove an upper and lower bound on the optimistic reward (Lemma \ref{lem:loss_inf_norm_upper_lower_bound}). The properties of $r_t$ for Gaussian bandit are studied in Appendix \ref{appendix_gaussian_optimistic_reward}

\subsection{Bounds on \texorpdfstring{$\|r_t\|_{\infty}$}-} \label{proof_loss_inf_norm_upper_lower_bound}

Due to the boundedness assumption, Lemma \ref{lem:loss_inf_norm_upper_lower_bound} below shows that the optimistic reward $r_t$ is almost bounded. When an arm $a$ is sampled less than a logarithmic number of times, $r_{t,a}$ becomes large enough to stir the sampling towards actions containing it. 

\begin{lemma} \label{lem:loss_inf_norm_upper_lower_bound}
Let $\mathcal{M}$ bounded. Under the event $\{\mu \in \C_{t}\}$, we have:
\begin{align*}
\max \left\{ \frac{f(t-1)}{\min_{a \in [d]} \tilde{N}_{t-1,a}}, \epsilon_{\nu}\right\} \leq \|r_t\|_{\infty} \leq \max \left\{ \frac{f(t-1)}{\min_{a \in [d]} \tilde{N}_{t-1,a}}, D_{\M}\right\} 
\end{align*}
\end{lemma}
The upper bound is a consequence of the boundedness assumption, $D_{\M} \defeq \sup_{(\phi, \lambda) \in \M^2} \|d_{\text{KL}}(\phi, \lambda)\|_{1} $. Since $\mu$ has a unique correct answer, the lower bound stems from the Chernoff information lower bound $\epsilon_{\nu}$ which holds for both (a) and (b): there exists $\epsilon_{\nu} > 0$ such that,
\begin{equation*}
    \forall \lambda \in \Theta_{I^*(\mu)}^{\complement}, \exists  a \in [d], \quad \text{ch}(\lambda_a, \mu_a) \geq \epsilon_{\nu}
\end{equation*}
where ch$(x,y) \defeq \inf_{u \in \Theta}\left( d_{\text{KL}}(u,x) + d_{\text{KL}}(u,y) \right)$.

Before proving Lemma \ref{lem:loss_inf_norm_upper_lower_bound}, let's first prove that $\epsilon_{\nu}$ exists. For (a) sub-Gaussian, we have $d_{\text{KL}}(u,x) \geq \frac{(u-x)^2}{2\sigma_{a}^2}$, hence the chernoff information of the setting (a) is greater than the one for setting (b). For (b) Gaussian, we have: $\text{ch}(x, y) = \frac{1}{2\sigma_{a}^2} \inf_{u \in \Theta}((u-x)^2 + (u-y)^2) =  \frac{(x-y)^{2}}{8\sigma_{a}^2}$.

Let $\mu \in \M$ and $\lambda \in \Theta_{I^*(\mu)}^{\complement}$. Since $\mu \in \Theta_{I^*(\mu)}$, which is an open set, the euclidean distance to $\Theta_{I^*(\mu)}^{\complement}$ is strictly positive: there exists $a \in [d]$ such that $|\lambda_a - \mu_a| \geq \epsilon > 0$. Since $\text{ch}(x, y) \geq \frac{(x-y)^{2}}{8\sigma_{a}^2}$, we can conclude for both (a) and (b) that there exists $\epsilon_{\nu} > 0$ as defined above.

Next, we prove the Lemma \ref{lem:loss_inf_norm_upper_lower_bound} itself.

\begin{proof}
For all $a \in [d]$, let $\phi_{t,a} \in \argmax_{\phi \in \{\alpha_{t,a}, \beta_{t,a}\}}  d_{\text{KL}}(\phi, \lambda_{t,a})$, the optimistic mean parameter. By convexity of $x \mapsto d_{\text{KL}}(x, y)$, we have: $\max_{\phi \in [\alpha_{t,a}, \beta_{t,a}]} d_{\text{KL}}(\phi, \lambda_{t,a}) = d_{\text{KL}}(\phi_{t,a}, \lambda_{t,a})$. The optimistic reward rewrites as: $r_{t,a} = \max \left\{\frac{f(t-1)}{ \tilde{N}_{t-1,a}}, d_{\text{KL}}(\phi_{t,a}, \lambda_{t,a})  \right\} $ for all $a \in [d]$. Using that $\max_{a \in [d]} \max\{x_a,y_a\} = \max\{\max_{a \in [d]} x_a, \max_{a \in [d]} y_a\}$, we obtain:
\begin{equation*}
    \|r_t\|_{\infty} = \max \left\{ \frac{f(t-1)}{\min_{a \in [d]} \tilde{N}_{t-1,a}}, \max_{a \in [d]}d_{\text{KL}}(\phi_{t,a}, \lambda_{t,a})  \right\}
\end{equation*}

Since $\M$ is bounded, $D_{\M} = \sup_{(\phi, \lambda) \in \M^2} \|d_{\text{KL}}(\phi, \lambda)\|_{1}$ and $\|\cdot\|_{\infty} \leq \|\cdot\|_{1}$, we obtain the desired upper bound.

Due to concentration events, with high probability we have $\mu \in \mathcal{C}_t = \bigtimes_{a \in [d]}[\alpha_{t,a}, \beta_{t,a}]$. Assume $\{\mu \in \mathcal{C}_t\}$ holds. Combining the definition of $\phi_{t,a}$ and $\lambda_t \in \partial \Theta_{I_t}$, we obtain: $ d_{\text{KL}}(\phi_{t,a}, \lambda_{t,a}) \geq  d_{\text{KL}}(\mu_{a}, \lambda_{t, a}) \geq \min_{\lambda \in \partial \Theta_{I_t}} d_{\text{KL}}(\mu_a, \lambda_a)\geq \min_{I \in \mathcal{I},\lambda \in \partial \Theta_{I}} d_{\text{KL}}(\mu_a, \lambda_a)$. The function $y \mapsto d_{KL}(x, y)$ is not convex in general and is minimized in $x=y$. Hence, the geometry of the cells yields: 
\begin{equation*}
    \min_{I \in \mathcal{I},\lambda \in \partial \Theta_{I}} d_{\text{KL}}(\mu_a, \lambda_a) = \min_{\lambda \in \partial \Theta_{I^*(\mu)}} d_{\text{KL}}(\mu_a, \lambda_a) \geq \min_{\lambda \in  \Theta_{I^*(\mu)}^{\complement}} d_{\text{KL}}(\mu_a, \lambda_a)
\end{equation*}

Since $\min\{d_{\text{KL}}(y,x), d_{\text{KL}}(x,y)\} \geq \text{ch}(x,y)$, the previous inequalities yield: $d_{\text{KL}}(\phi_{t,a}, \lambda_{t,a})  \geq \min_{\lambda \in  \Theta_{I^*(\mu)}^{\complement}} \text{ch}(\lambda_a, \mu_a)$ for all $a \in [d]$. Using the chernoff information lower bound and taking the maximum over $a \in [d]$, we conclude that: $\|r_t\|_{\infty} \geq \max \left\{ \frac{f(t-1)}{\min_{a \in [d]} \tilde{N}_{t-1,a}}, \epsilon_{\nu}\right\}$. 
\end{proof}

\subsection{Gaussian Bandit} \label{appendix_gaussian_optimistic_reward}

For Gaussian bandit, we have $d_{\text{KL}}(x, y) = \frac{(x - y)^{2}}{2 \sigma_{a}^2} $. Direct computations yield: $\alpha_{t,a} = \mu_{t-1,a} - \sqrt{\frac{2 f(t-1) \sigma_a^{2}}{\tilde{N}_{t-1,a}}}$ and $\beta_{t,a} = \mu_{t-1,a} + \sqrt{\frac{2 f(t-1) \sigma_a^{2}}{\tilde{N}_{t-1,a}}}$. As a consequence of $d_{\text{KL}}$ and $[\alpha_{t,a}, \beta_{t,a}]$ being symmetric, Lemma \ref{lem:gaussian_optimistic_reward} shows that the clipping $\frac{f(t-1)}{\tilde{N}_{t-1,a}}$ is superfluous.

\begin{lemma} \label{lem:gaussian_optimistic_reward}
Let $\nu$ be independent Gaussian and $\lambda \in \mathcal{M}$. Then, for all $a \in [d]$,
\begin{align*}
    \phi_{t,a} &= \alpha_{t,a} \bm{1}_{\lambda_{a} \geq \mu_{t-1,a}} + \beta_{t,a} (1 - \bm{1}_{\lambda_{a} \geq \mu_{t-1,a}}) \\
    d_{\text{KL}}(\phi_{t,a}, \lambda_{a}) &= \frac{(\mu_{t-1,a} - \lambda_{a})^{2}}{2\sigma_{a}^2} + \frac{f(t-1)}{\tilde{N}_{t-1,a}} + \sqrt{\frac{2f(t-1)}{\sigma_{a}^2\tilde{N}_{t-1,a}}}|\mu_{t-1,a} - \lambda_{a}| \geq \frac{f(t-1)}{\tilde{N}_{t-1,a}}
\end{align*}
where $\phi_{t,a} = \argmax_{\phi \in \{\alpha_{t,a}, \beta_{t,a}\}}\frac{(\phi - \lambda_{a})^2}{2 \sigma_a^2}$.
\end{lemma}
\begin{proof}
Let $\lambda \in \mathcal{M}$. Let $\phi_{t,a} = \argmax_{\phi \in \{\alpha_{t,a}, \beta_{t,a}\}}\frac{(\phi - \lambda_{a})^2}{2 \sigma_a^2}$ for all $a \in [d]$. Assume $\lambda_{a} \geq \mu_{t-1,a}$. Since $\lambda_{a} \geq \mu_{t-1,a} \geq \alpha_{t,a}$ and $\mu_{t-1,a} \leq \beta_{t,a}$, we have $(\alpha_{t,a} - \lambda_{a})^2  \geq (\beta_{t,a} - \lambda_{a})^2$. Therefore $\phi_{t,a} = \alpha_{t,a}$. Assume $\lambda_{a} < \mu_{t-1,a}$. Since $\lambda_{a} < \mu_{t-1,a} \leq \beta_{t,a}$ and $\mu_{t-1,a} \geq \alpha_{t,a}$, we have $(\alpha_{t,a} - \lambda_{a})^2  \leq (\beta_{t,a} - \lambda_{a})^2$. This concludes the first statement.

Due to the explicit formulas for $\phi_{t,a}$, $\alpha_{t,a}$ and $\beta_{t,a}$, we have $(\phi_{t,a} - \mu_{t-1,a})(\mu_{t-1,a} - \lambda_{a}) = \sqrt{\frac{2 f(t-1) \sigma_a^{2}}{\tilde{N}_{t-1,a}}} |\mu_{t-1,a} - \lambda_{a}|  \geq 0$, hence we conclude:
\begin{align*}
	\frac{(\phi_{t,a} - \lambda_{a})^{2}}{2\sigma_{a}^2} = \frac{(\mu_{t-1,a} - \lambda_{a})^{2}}{2\sigma_{a}^2} + \frac{f(t-1)}{\tilde{N}_{t-1,a}} + \sqrt{\frac{2f(t-1)}{\sigma_{a}^2\tilde{N}_{t-1,a}}}|\mu_{t-1,a} - \lambda_{t,a}| \geq \frac{f(t-1)}{\tilde{N}_{t-1,a}}
\end{align*}
\end{proof}

\section{Learner's Cumulative Regret} \label{appendix_learner_cumulative_regret}

In the Appendix \ref{appendix_learner_cumulative_regret}, we show upper bounds on the cumulative regret for the different learners: Hedge in Lemma \ref{lem:hedge_cumulative_regret}, AdaHedge in Lemma \ref{lem:adahedge_cumulative_regret}, OFW in Lemma \ref{lem:ofw_cumulative_regret} and LLOO in Lemma \ref{lem:lloo_cumulative_regret}.

\subsection{Learner on the Simplex}   \label{appendix_learner_simplex_cumulative_regret}

The extended optimistic reward is defined as: $U_{t,A} \defeq \langle \1_{A}, r_t \rangle$ for all $A \in \A$. The cumulative regret rewrites as:
\begin{align*}
    R_{t}^A &=  \sum_{s=1}^{t} b_s  \langle w_s, l_s\rangle - \min_{A \in \A} \sum_{s=1}^{t} b_s l_{s,A} \leq \left( \sum_{s=1}^{t} \langle w_s, l_s\rangle - \min_{A \in \A} \sum_{s=1}^{t} l_{s,A} \right) \max_{s \leq t} b_{s}
\end{align*}
where $l_{t, A} = \frac{\|U_t\|_{\infty} - U_{t,A}}{b_t} \in [0,1]$ and $b_t = \|U_t\|_{\infty} - \min_{A \in \A} U_{t,A}$ is the scale of the loss at time $t$. Since $U_t$ is positive, $\tilde{N}_{t-1,a} \geq 1$ and $f(t) = \Omega(\ln(t))$, we obtain that: $b_t \leq \|U_t\|_{\infty} \leq d \|r_t\|_{\infty}$ and $\frac{f(t-1)}{\min_{a \in [d]} \tilde{N}_{t-1,a}}\leq  f(t-1) = O(\ln(t))$. Therefore, Lemma \ref{lem:loss_inf_norm_upper_lower_bound} yields that: $\max_{s \leq t} b_{s} = O\left(\ln(t)\right)$.

\paragraph{Hedge} Using Corollary 3 in \citet{cesa-bianchi_improved_2006}, we obtain that Hedge has optimal cumulative regret (Lemma \ref{lem:hedge_cumulative_regret}).

\begin{lemma} \label{lem:hedge_cumulative_regret}
Hedge satisfies
\begin{align*}
	R_{t}^{Hedge} &\leq  \left( 4 \sqrt{\frac{L_{t}^*(t - L_{t}^*)}{t}\ln(|\A|)} + 39 \max\{1, \ln(|\A|)\} \right) \max_{s \leq t} b_{s} = O\left(\ln(t)\sqrt{t}\right)
\end{align*}
where $L_{t}^* = \min_{A \in \A} \sum_{s=1}^{t} l_{s,A}$.
\end{lemma}
\begin{proof}
For scaled losses in $[0,1]$, Corollary 3 in \citet{cesa-bianchi_improved_2006} yields that Hedge's cumulative regret $R_{t}'$ satisfies: $R_{t}' \leq 4 \sqrt{\frac{L_{t}^*(\sigma t - L_{t}^*)}{t}\ln(|\mathcal{A}|)} + 39\sigma \max\{1, \ln(|\mathcal{A}|)\}$ where $\sigma$ is the range of observed loss. By definition of $l_t$, we have $\sigma=1$. Factorizing the maximum of the scale of the loss $\max_{s \leq t} b_{s}$, we obtain the upper bound on $R_{t}^{Hedge}$. In the worst case this algorithm has a regret of order $O(\sqrt{t})$, but it performs much better when the loss of the best expert $L_{t}^*$ is close to either $0$ or $t$. Combined with $\max_{s \leq t} b_{s} = O\left(\ln(t)\right)$, this concludes the proof.
\end{proof}

\paragraph{AdaHedge} Using the results of \citet{de_rooij_follow_2013}, we obtain that AdaHedge has optimal cumulative regret (Lemma \ref{lem:adahedge_cumulative_regret}).

\begin{lemma} \label{lem:adahedge_cumulative_regret}
AdaHedge satisfies
\begin{align*}
	R_{t}^{AdaHedge} &\leq \sqrt{\sum_{s \leq t} b_s^2 \ln(|\A|)} +\left(\frac{4}{3} \ln(|\A|) + 2\right)  \max_{s \leq t} b_{s} = O\left(\ln(t)\sqrt{t}\right)
\end{align*}
\end{lemma}
\begin{proof}
For scaled losses in $[0,1]$, Theorem 6 in \citet{de_rooij_follow_2013} yields that AdaHedge's cumulative regret $R_{t}' $ satisfies: $R_{t}' \leq 2 \sqrt{V_{t}\ln(|\mathcal{A}|)} + \frac{4}{3} \ln(|\mathcal{A}|) + 2$ where $V_{t} =\sum_{s \in [t]} v_s$ with $v_s = \sum_{A \in \mathcal{A}} w_{s,A} (l_{s,A} - \langle w_{s},l_{s} \rangle)^2$. We have $v_s \leq \|l_{s} - \langle w_{s},l_{s} \rangle\|_{\infty}^2  = \frac{\|\langle w_s, U_s\rangle - U_{s}\|_{\infty}^2}{b_s^{2}} \leq \frac{b_s^2}{\sigma^2}$ where $\sigma = \max_{s \leq t} b_s$. The upper bound on $R_{t}'$ rewrites as: $R_{t}'\leq \frac{1}{\sigma} \sqrt{\sum_{s \leq t} b_s^2 \ln(|\mathcal{A}|)} + \frac{4}{3} \ln(|\mathcal{A}|) + 2$. Theorem 16 in \citet{de_rooij_follow_2013} yields that $R_{t}^{A} = \sigma R_{t}'$. Therefore, we conclude that:
\begin{align*}
    R_{t}^{AdaHedge} &\leq \sqrt{\sum_{s \leq t} b_s^2 \ln(|\A|)} +\left(\frac{4}{3} \ln(|\A|) + 2\right)  \max_{s \leq t} b_{s} \\
    &\leq  \left( \sqrt{t\ln(|\A|)} +\left(\frac{4}{3} \ln(|\A|) + 2\right) \right) \max_{s \leq t} b_{s}
\end{align*}

Combined with $\max_{s \leq t} b_{s} = O\left(\ln(t)\right)$, this concludes the proof.
\end{proof}

\subsection{Learner on the Transformed Simplex}  \label{appendix_learner_tsimplex_cumulative_regret}

We recall the cumulative regret is defined as: $R_{t}^{A} = \max_{A \in \A} \sum_{s=1}^{t} \langle \1_{A}, r_s\rangle -  \sum_{s=1}^{t} \langle \tilde{w}_s, r_s\rangle$. For the same reasons as in Appendix \ref{appendix_learner_simplex_cumulative_regret} Lemma \ref{lem:loss_inf_norm_upper_lower_bound} yields: $\max_{s \leq t} \|r_s\|_2 = O(\ln(t))$.

\paragraph{OFW} Slightly adapting the results of \citet{hazan_projection-free_2012}, we obtain that OFW has an upper bound on the cumulative regret in $O\left(\ln(t)^2 t^{3/4}\right)$ (Lemma \ref{lem:ofw_cumulative_regret}). This is in general suboptimal for the online linear optimization setting.

\begin{lemma} \label{lem:ofw_cumulative_regret}
OFW satisfies
\begin{equation*}
	R_{t}^{OFW} \leq 18 (2 + \max_{s \leq t} \|r_s\|_2)^2 \diamtsimplex t^{3/4} + 3 \diamtsimplex t^{3/4} = O\left(\ln(t)^2 t^{3/4}\right) 
\end{equation*}
\end{lemma}
\begin{proof}
OFW described in Section \ref{learner_on_tsimplex} is exactly the algorithm used in the proof of Theorem 4.4 in \citet{hazan_projection-free_2012}, which is a result for adversarial cost functions. In their notations, the Lipschitz constant $L$ satisfies: $L = \|r_{t}\|_2$. In order to conserve the anytime property of OFW, we use a different $\sigma_t$ which is independent of $L$, $\sigma_t = \frac{1}{\diamtsimplex} t^{-1/4}$. The decrease in $t^{-1/4}$ is optimal. This modification doesn't change the idea of the proof and impact only the final bound by a multiplicative factor, $\max_{s \leq t} \|r_{s}\|_2$. A close examination of its proof shows that Theorem 3.1 in \citet{hazan_projection-free_2012} still holds for time dependent Lipschitz constant $L_t$. Therefore, we follow the proof of Theorem 4.4 and apply Theorem 3.1 for $\hat{f}_{t}(x) = \langle l_t, x \rangle + \frac{1}{\diamtsimplex} t^{-1/4} \|x-x_1\|_{2}^{2}$. The exact same steps and using that $L_s \leq \max_{s \leq t} \|r_{s}\|_2$ for all $s\in [t]$ yield that:
\begin{align*}
    R_{t}^{OFW} &\leq 18 (\max_{s \leq t} \|r_s\|_2 + 2) \diamtsimplex t^{3/4}\max_{s \leq t} \|r_s\|_2 + 3 \diamtsimplex t^{3/4} \\
    &\leq 18 (2 + \max_{s \leq t} \|r_s\|_2)^2 \diamtsimplex t^{3/4} + 3 \diamtsimplex t^{3/4}
\end{align*}

Combined with $\max_{s \leq t} \|r_s\|_2 = O(\ln(t))$, this concludes the proof.
\end{proof}

\paragraph{LLOO} Using Theorem 3 in \citet{garber_linearly_2015}, we obtain that OFW has optimal cumulative regret (Lemma \ref{lem:lloo_cumulative_regret}).

\begin{lemma} \label{lem:lloo_cumulative_regret}
    Let $T$ be the horizon and $\mu_{\A}$, defined in \citet{garber_linearly_2015}. With $\gamma_{\A} = (3 d \mu_{A}^2)^{-1}$, $\eta_{\A, T} = \frac{\diamtsimplex}{18 \mu_{\A}\sqrt{d T} \max_{t \leq T} \|r_t\|_2}$  and $M_{\A, T} = \min \left\{\mu_{\A}^2 \frac{d}{\sqrt{T}}\left( 1 + \frac{1}{18 d  \mu_{\A}^2 } \right), 1\right\}$, LLOO satisfies:
    \begin{equation*}
        R_{t}^{LLOO} = O\left( \diamtsimplex \mu_{\A} \sqrt{d t} \max_{s \leq t} \|r_s\|_2 \right) = O\left(\ln(t)\sqrt{t}\right)
    \end{equation*}
\end{lemma}
\begin{proof}
Let $T$ be the horizon and $\mu_{\A}$ as defined in \citet{garber_linearly_2015} (see Appendix \ref{appendix_experimental_setting_details} for an explicit formula), which depends on $\tsimplex$. For a non strongly convex function $\sigma = 0$, LLOO described in Section \ref{learner_on_tsimplex} is exactly the combination of Algorithm 5 and Algorithm 4 in \citet{garber_linearly_2015}. The re-organization highlights the similarities with OFW. As parameters for the algorithm, we use the theoretically licensed: $\gamma_{\A} = (3 d \mu_{A}^2)^{-1}$, $\eta_{\A, T} = \frac{\diamtsimplex}{18 \mu_{\A}\sqrt{d T} \max_{t \leq T} \|r_t\|_2}$  and $M_{\A, T} = \min \left\{\mu_{\A}^2 \frac{d}{\sqrt{T}}\left( 1 + \frac{1}{18 d  \mu_{\A}^2 } \right), 1\right\}$. Theorem 3 in \citet{garber_linearly_2015} yields that: $R_{t}^{LLOO} = O\left( \diamtsimplex \mu_{\A} \sqrt{d t} \max_{s \leq t} \|r_s\|_2 \right) = O\left(\ln(t)\sqrt{t}\right)$. Combined with $\max_{s \leq t} \|r_s\|_2 = O(\ln(t))$, this concludes the proof.
\end{proof}

\section{Proof of Theorem \ref{thm:finite_time_upper_bound}} \label{appendix_proof_upper_bound}

In Appendix \ref{proof_upper_bound_finite_time_decomposition}, we prove the preliminary Lemma \ref{lem:upper_bound_finite_time_decomposition}. The saddle-point property of the algorithm $\A_{I}^{A}$ associated to $I$ is proven in Appendix \ref{saddle_point_property}. In Appendix \ref{appendix_candidate_answer}, we lower and upper bound the number of times when the candidate answer is correct. Combining them yields the definition of $T_{0}(\delta)$ and concludes the proof of Theorem \ref{thm:finite_time_upper_bound}. Technical arguments with respect to C-Tracking and concentration events are proven in Appendix \ref{technicalities}.

\subsection{Proof of Lemma \ref{lem:upper_bound_finite_time_decomposition}} \label{proof_upper_bound_finite_time_decomposition}

Let $t_b \defeq t^{1/(1+b)}$, with $b > 0$, and $(\mathcal{E}_t)_{t \geq 1}$ be a sequence of concentrations events for the exploration bonus $f$ with parameters $b$ and $c>0$: for all $t \geq 1$,
\begin{equation} \label{eq:concentration_events}
    \mathcal{E}_t \defeq \left\{\forall s \leq t, \forall a \in [d], \quad \tilde{N}_{s,a} d_{\text{KL}}(\mu_{s,a}, \mu_a) \leq f\left(t_b\right)\right\}
\end{equation}
where $f(t) = \overline{W}((1+c)(1+b) \ln(t))$ with $c>0$, $b>0$ and $\overline{W}(x) \approx x + \ln(x)$. More precisely, for $x \geq 1$, $\overline{W}(x) = - W_{-1}(- e^{-x})$ where $W_{-1}$ denotes the negative branch of the Lambert $W$ function. This sequence is theoretically validated due to Lemmas 5 and 6 in \citet{degenne_non-asymptotic_2019}.

\begin{lemma*}[Lemmas 5 and 6 in \citet{degenne_non-asymptotic_2019}] \label{lem:independent_concentration_events_finite_sum}
Let $(Y_{s,a})_{s \in [t]}$ be i.i.d random variables in a canonical one-parameter exponential family with mean $\mu_a$. Then, for $\alpha > 0$,
    \begin{equation*}
        \mathbb{P}_{\nu}\left[\exists s \leq t , d_{\text{KL}}\left(\frac{1}{s} \sum_{r=1}^{s} Y_{s,a},\mu_a\right) \geq \frac{\alpha}{s}\right] \leq 2e\ln(t)e^{-(\alpha - \ln(\alpha))}
    \end{equation*}
For independent $\nu$ and $(\mathcal{E}_t)_{t \geq 1}$ defined in Equation \ref{eq:concentration_events}, we obtain:
\begin{equation*}
    \forall t \geq 3, \mathbb{P}_{\nu}\left[\mathcal{E}_t^{\complement}\right] \leq 2 e d \frac{\ln(t)}{t^{1+c}} \quad \text{and} \quad \sum_{t > T_{0}(\delta)} \mathbb{P}_{\nu}\left[\mathcal{E}_t^{\complement}\right] \leq \frac{2ed}{c^2}
\end{equation*}
\end{lemma*}

\begin{lemma*} 
Let $(\mathcal{E}_t)_{t \geq 1}$ be a sequence of concentrations events for the exploration bonus $f$ with parameters $c>0$ and $b>0$: for all $t \geq 1$,
\begin{align*}
    \mathcal{E}_t &\defeq \left\{\forall s \leq t, \forall a \in [d], \quad \tilde{N}_{s,a} d_{\text{KL}}(\mu_{s,a}, \mu_a) \leq f\left(t^{1/(1+b)}\right)\right\}
\end{align*}
Suppose that there exists $T_{0}(\delta) \in \mathbb{N}$ such that for all $t > T_{0}(\delta)$, $\mathcal{E}_t \subset \{\tau_{\delta} \leq t\}$. Then 
\begin{equation*}
    \E_{\nu}[\tau_{\delta}] \leq T_{0}(\delta) + \sum_{t > T_{0}(\delta)} \mathbb{P}_{\nu}\left[\mathcal{E}_t^{\complement}\right] \quad \text{ where }\quad \sum_{t > T_{0}(\delta)} \mathbb{P}_{\nu}\left[\mathcal{E}_t^{\complement}\right] \leq \frac{2ed}{c^2}
\end{equation*}
\end{lemma*}
\begin{proof}
First, let's prove the upper bound for an arbitrary sequence of concentrations events $(\mathcal{E}_t)_{t \geq 1}$ satisfying: there exists $T_{0}(\delta) \in \mathbb{N}$ such that for $t > T_{0}(\delta)$, $\mathcal{E}_t \subseteq \{\tau_{\delta} \leq t\}$. Since the stopping time is a positive random variable, we have: $\E_{\nu}[\tau_{\delta}] = \sum_{t=1}^{\infty} \mathbb{P}_{\nu}(\tau_{\delta} > t)$. For $t > T_{0}(\delta)$, $\{\tau_{\delta} > t\} \subseteq \mathcal{E}_t^{\complement}$, hence $\mathbb{P}_{\nu}(\tau_{\delta} > t) \leq \mathbb{P}_{\nu}\left[\mathcal{E}_t^{\complement}\right]$. For $t \leq T_{0}(\delta)$, we have $\mathbb{P}_{\nu}(\tau_{\delta} > t) \leq 1$. Combining those yields: $\E_{\nu}[\tau_{\delta}] \leq T_{0}(\delta) + \sum_{t > T_{0}(\delta)} \mathbb{P}_{\nu}\left[\mathcal{E}_t^{\complement}\right]$. Second, let's prove that $\sum_{t > T_{0}(\delta)} \mathbb{P}_{\nu}\left[\mathcal{E}_t^{\complement}\right] \leq \frac{2ed}{c^2}$ for $\mathcal{E}_t$ defined in Equation \ref{eq:concentration_events}. Combining Lemma 5 and Lemma 6 from \citet{degenne_non-asymptotic_2019} yields the desired result.
\end{proof}

\subsection{Saddle-point Property}  \label{saddle_point_property}

Let $T_{t,I} \defeq \{s \in [t]: I_s = I\}$ for all $I \in \I$. Let $I \in \I$. Similarly to \citet{degenne_non-asymptotic_2019}, we prove the saddle-point property of the algorithm $\A_{I}^{A}$ associated to $I$.
\begin{definition}
An algorithm playing sequences $(\tilde{w}_{s}, \lambda_{s})_{s \in T_{t,I}} \in \left(\tsimplex \times \Theta_{I}^{\complement}\right)^{|T_{t,I}|}$ is an approximate optimistic saddle-point algorithm with slack $x_t$ if:
\begin{equation*}
    \inf_{\lambda \in \Theta_{I}^{\complement}} \sum_{s \in T_{t,I}} \langle \tilde{w}_{s}, d_{\text{KL}}(\mu_{s-1}, \lambda)\rangle \geq \max_{A \in \mathcal{A}} \sum_{s \in T_{t,I}} \langle \bm{1}_{A}, r_{s} \rangle - x_t
\end{equation*}
\end{definition}

Using the standard result that $\min_{x}(f(x)+g(x)) \geq \min_{x}(f(x))  + \min_{x}(g(x)) $ and the explicit definition of $\lambda_s \in \inf_{\lambda \in \Theta_{I}^{\complement}} \langle \tilde{w}_{s}, d_{\text{KL}}(\mu_{s-1}, \lambda) \rangle$, which is a best-response oracle without regret, we obtain that:
\begin{equation*}
     \inf_{\lambda \in \Theta_{I}^{\complement}} \sum_{s \in T_{t,I}} \langle \tilde{w}_{s}, d_{\text{KL}}(\mu_{s-1}, \lambda)\rangle \geq \sum_{s \in T_{t,I}} \inf_{\lambda \in \Theta_{I}^{\complement}} \langle \tilde{w}_{s}, d_{\text{KL}}(\mu_{s-1}, \lambda) \rangle =  \sum_{s \in T_{t,I}} \langle \tilde{w}_{s}, d_{\text{KL}}(\mu_{s-1}, \lambda_s) \rangle
\end{equation*}

Let $C_{s,a} = r_{s,a} - d_{\text{KL}}(\mu_{s-1,a}, \lambda_{s,a})$ and $C_{s} = (C_{s,a})_{a \in [d]}$ be the slack between the optimistic reward and the reward for the parameter $\mu_{s-1}$. We have:
\begin{align*}
	\sum_{s \in T_{t,I}} \langle \tilde{w}_{s}, d_{\text{KL}}(\mu_{s-1}, \lambda_s) \rangle &\geq  \sum_{s \in T_{t,I}} \langle \tilde{w}_{s}, r_s \rangle -   \sum_{s \in T_{t,I}} \langle \tilde{w}_{s}, C_s \rangle
\end{align*}

Since $\sum_{s \in T_{t,I}} \langle \tilde{w}_{s}, r_s \rangle = \sum_{s \in T_{t,I}} \langle w_{s}, U_s \rangle$ is The cumulative reward of the $A$-player with a learner on $\simplex$ or a learner on $\tsimplex$, introducing the cumulative regret $R_{t}^{A}$ yields: $\sum_{s \in T_{t,I}} \langle \tilde{w}_{s}, r_s \rangle \geq \max_{A \in \mathcal{A}} \sum_{s \in T_{t,I}} \langle \bm{1}_{A}, r_s \rangle - R_{|T_{t,I}|}^A$. Combining these inequalities yield:
\begin{equation*}
    \inf_{\lambda \in \Theta_{I}^{\complement}} \sum_{s \in T_{t,I}} \langle \tilde{w}_{s}, d_{\text{KL}}(\mu_{s-1}, \lambda)\rangle \geq \max_{A \in \mathcal{A}} \sum_{s \in T_{t,I}} \langle \bm{1}_{A}, r_{s} \rangle - x_t
\end{equation*}
where $x_t = \sum_{s \in T_{t,I}} \langle \tilde{w}_{s}, C_s \rangle + R_{|T_{t,I}|}^A$ is the slack of the optimistic saddle-point algorithm.

\subsection{Candidate Answer} \label{appendix_candidate_answer}

The MLE $\mu_{t-1}$ summarizes the observations seen at the beginning of round $t$. Since $\M$ can have a peculiar geometry, $\mu_{t-1} \notin \M$ might happen. Due to concentration results, we have $\mu_{t-1} \in \M$ after a certain time. We consider $\tilde{\mu}_{t-1} \in \M \cap \C_t$, so that for all $a \in [d]$, $d_{\text{KL}}(\mu_{t-1,a},\mu_{t-1,a}^{\mathcal{M}}) \leq \frac{f(t-1)}{\tilde{N}_{t-1,a}}$. A more elaborate choice, but not necessary, would be: $\tilde{\mu}_{t-1} \in \argmin_{\lambda \in \M \cap \C_{t}} \langle \tilde{N}_{t-1}, d_{\text{KL}}(\mu_{t-1}, \lambda) \rangle$. When $\M \cap \C_{t} = \emptyset$, $\tilde{\mu}_{t-1}$ is chosen randomly. The candidate answer is defined as: $I_t = I^*(\tilde{\mu}_{t-1})$.

In Appendix \ref{appendix_incorrect_candidate_answer}, we show that $I_t$ is not the correct answer for only $o(t)$ rounds. This provides a lower bound on the number of times the candidate answer is correct. An upper bound on the number of times the candidate answer is correct is proved in Appendix \ref{appendix_correct_candidate_answer}. 

\subsubsection{Incorrect Answer} \label{appendix_incorrect_candidate_answer}

Let $I^*= I^*(\mu)$, $t<\tau_{\delta}$, $T_{t,I} \defeq \{s \in [t]: I_s = I\}$ for all $I \in \I$ and $t_b \defeq t^{1/(1+b)}$. The number of time the recommended answer is not correct is $o(t)$ as a consequence of the following fact: when $I_t \neq I^*$ a quantity, denoted $\epsilon_t$, is increasing linearly while being $O(\sqrt{t})$ due to concentration arguments. The proof of this fact uses a consequence of the chernoff information lower bound $\epsilon_{\nu}$ (Appendix \ref{proof_loss_inf_norm_upper_lower_bound}): the Lemma 18 of \citet{degenne_non-asymptotic_2019}.

\begin{lemma*}[Lemma 18 in \citet{degenne_non-asymptotic_2019}] \label{lem:lemma_18_degenne_non_asymptotic}
    For (a) sub-Gaussian or (b) Gaussian bandit, if $d_{\text{KL}}(\mu_{t-1,a},\mu_a) \leq \frac{f(t-1)}{\tilde{N}_{t-1,a}}$ for all $a \in [d]$, then: $I_t \neq I^*(\mu)$ implies there exists $a_{0} \in [d]$ such that $\frac{f(t-1)}{\tilde{N}_{t-1,a_{0}}} \geq \frac{\epsilon_{\nu}}{2}$.
\end{lemma*} 
Let $s \in [t]$, such that $I_s \neq I^*$, and $\epsilon_t \defeq \sum_{s \leq t, I_s \neq I^*} \langle \tilde{w}_s, d_{\text{KL}}(\mu_{s-1},\mu)\rangle$. Since $\mu \in \Theta_{I_s}^{\complement}$, we have: $\epsilon_t \geq \sum_{I \in \mathcal{I} \setminus \{I^*\}} \inf_{\lambda \in \Theta_{I_s}^{\complement}}  \sum_{s \leq t, I_s = I} \langle \tilde{w}_s, d_{\text{KL}}(\mu_{s-1},\lambda)\rangle$. For each $I \neq I^*$, the approximate optimistic saddle-point property of the learners with slack $x_t = R^{A}_{|T_{t,I}|} + \sum_{s \in T_{t,I}} \langle \tilde{w}_s, C_s\rangle$ (Appendix \ref{saddle_point_property}) yields:
\begin{align*}
    \epsilon_t \geq \sum_{I \in \mathcal{I} \setminus \{I^*\}} \max_{A \in \mathcal{A}} \sum_{s \in T_{t,I}} \langle \bm{1}_{A}, r_s\rangle - \sum_{I \in \mathcal{I} \setminus \{I^*\}} R^{A}_{|T_{t,I}|} - \sum_{s \leq t, I_s \neq I^*} \langle \tilde{w}_s, C_s\rangle
\end{align*}
Since $f(s-1) \leq f(t-1)$, the condition of Lemma 18 in \citet{degenne_non-asymptotic_2019} is validated under the event $\mathcal{E}_t$. Hence, we obtain that: for all $s \in [t_b, t]$ such that $I_s \neq I^*$, there exists $a_{0} \in [d]$ such that $\frac{f(s-1)}{\tilde{N}_{s-1,a_{0}}} \geq \frac{\epsilon_{\nu}}{2}$. Let $t' \defeq  \max \left\{s \in [t]: I_s \neq I^* \right\}$. We suppose $t' > t_b$, which is possible since $b>0$. The higher $b$ is, the weaker this assumption is. Let $a_{0} \in [d]$ such that $\frac{f(t'-1)}{\tilde{N}_{t'-1,a_{0}}} \geq \frac{\epsilon_{\nu}}{2}$. Since $f$ is increasing and, for $t>e$, $\frac{f(t_b)}{f(t)} \geq C_b = \frac{1}{3(1+b)}$, we have that: for all $s \in [t_b,t']$, 
\begin{equation*}
    \frac{f(s-1)}{\tilde{N}_{s-1,a_{0}}} \geq \frac{f(s-1)}{\tilde{N}_{t'-1,a_{0}}}  =\frac{f(s-1)}{f(t'-1)} \frac{f(t'-1)}{\tilde{N}_{t'-1,a_{0}}}  \geq \frac{f(t_b)}{f(t)} \frac{\epsilon_{\nu}}{2} \geq C_{b} \frac{\epsilon_{\nu}}{2} 
\end{equation*}

Let $A_{0} \in \A$ such that $a_{0} \in A_{0}$. By definition of $r_s$, we have $r_{s,a_{0}} \geq \frac{f(s-1)}{\tilde{N}_{s-1,a_{0}}}$ and $r_s \geq 0$. Combining these inequalities and dropping the time $s < t_b$ yield:
\begin{equation*}
	\max_{A \in \mathcal{A}} \sum_{s \in T_{t,I}} \langle \bm{1}_{A}, r_s\rangle \geq  \sum_{s \in T_{t,I}} \langle \bm{1}_{A_0}, r_s\rangle  \geq \sum_{s \in T_{t,I}: s \geq t_b}  r_{s,a_{0}}  \geq \sum_{s \in T_{t,I}: s \geq t_b}  \frac{f(s-1)}{\tilde{N}_{s-1,a_{0}}} \geq C_{b} \frac{\epsilon_{\nu}}{2} \left(|T_{t,I}| - |T_{t_b,I}| \right)
\end{equation*} 

Since $R^{A}_{|T_{t,I}|} \leq R^{A}_{t}$, we have $ \sum_{I \in \mathcal{I} \setminus \{I^*\}} R^{A}_{|T_{t,I}|} \leq (|\mathcal{I}| - 1) R^{A}_{t}$. For a concave cumulative regret such as $t \mapsto \sqrt{t}$, we would have $\sum_{I \in \mathcal{I} \setminus \{I^*\}} R^{A}_{|T_{t,I}|} \leq (|\mathcal{I}| - 1)R^{A}_{\frac{t - |T_{t,I^*}|}{|\mathcal{I}| - 1}}$. Since $\sum_{I \in \mathcal{I} \setminus \{I^*\}} (|T_{t,I}| - |T_{t_b,I}|) = t - t_b - |T_{t,I^*}|$, summing these inequalities yields:
\begin{equation*}
    \epsilon_t \geq \frac{C_b \epsilon_{\nu}}{2} (t - t_b - |T_{t,I^*}|) - (|\mathcal{I}| - 1) R^{A}_{t} - \sum_{s \leq t, I_s \neq I^*} \langle \tilde{w}_s, C_s\rangle
\end{equation*}

Under event $\mathcal{E}_t$ we have: for all $s \leq t$ and all $a \in [d]$, $d_{\text{KL}}(\mu_{s-1},\mu) \leq \frac{f(t_b)}{\tilde{N}_{s-1,a}}$. Since $f(t_b) \leq f(t)$, by definition of $\epsilon_t$ and Lemma \ref{lem:summation_ratio_w_and_empirical_counts}, we obtain:
\begin{equation*}
	\epsilon_t \leq f(t) \sum_{s \leq t} \sum_{a \in [d]} \frac{\tilde{w}_{s,a}}{\tilde{N}_{s-1,a}} \leq f(t) \left( d |\mathcal{A}|^2 + 2d \ln \left( \frac{t K}{d} \right) \right)
\end{equation*}

Combining these inequalities, we obtain a lower bound on $|T_{t, I^*}|$: for $t<\tau_{\delta}$, under $\mathcal{E}_t$,
\begin{align*}
    |T_{t,I^*}| \geq t - t_b - \frac{2}{C_b \epsilon_{\nu}} \left((|\mathcal{I}| - 1) R^{A}_{t} + \sum_{s \leq t, I_s \neq I^*} \langle \tilde{w}_s, C_s\rangle + f(t) \left( d |\mathcal{A}|^2 + 2d \ln \left( \frac{t K}{d} \right) \right) \right)
\end{align*}

\subsubsection{Correct Answer}  \label{appendix_correct_candidate_answer}

Let $I^*= I^*(\mu)$, $t<\tau_{\delta}$ and $T_{t,I} \defeq \{s \in [t]: I_s = I\}$ for all $I \in \I$. Let $t' \defeq  \max \left\{ s \leq  t: I_{s} = I^* \right\}$ be the last round in which we recommend the correct answer before the algorithm stops. Since $s \mapsto \beta(s,\delta)$ is increasing, we have $\beta(t, \delta) \geq \beta(t' - 1,\delta)$. By definition of $t'$, we have $T_{t,I^*} = T_{t',I^*}$. The non-satisfied stopping criterion rewrites as: $\beta(t' - 1,\delta) \geq \inf_{\lambda \in \Theta_{I^*}^{\complement}} \langle \tilde{N}_{t'-1}, d_{\text{KL}}(\mu_{t'-1}, \lambda)\rangle$. Lemma \ref{lem:concentration_events_summed_kl_diff_upper_bound} and $t \mapsto t f(t)$ increasing yield that:
\begin{align*}
    \inf_{\lambda \in \Theta_{I^*}^{\complement}} \langle \tilde{N}_{t'-1}, d_{\text{KL}}(\mu_{t'-1}, \lambda)\rangle &\geq \inf_{\lambda \in \Theta_{I^*}^{\complement}} \langle \tilde{N}_{t'-1}, d_{\text{KL}}(\mu, \lambda)\rangle -  L_{\M}\sqrt{2 (t'-1) f(t'-1) d \|\sigma^2\|_{\infty}  K}\\
    &\geq \inf_{\lambda \in \Theta_{I^*}^{\complement}} \langle \tilde{N}_{t'-1}, d_{\text{KL}}(\mu, \lambda)\rangle -  L_{\M} \sqrt{2 t f(t) d \|\sigma^2\|_{\infty}  K}
\end{align*}

Combining C-Tracking, Lemma \ref{lem:c_tracking_control_empirical_counts} and $\langle \1, d_{\text{KL}}(\mu, \lambda) \rangle = \|d_{\text{KL}}(\mu, \lambda)\|_{1}\leq D_{\M}$ ($\M$ bounded), we obtain:
\begin{align*}
    \inf_{\lambda \in \Theta_{I^*}^{\complement}} \langle \tilde{N}_{t'-1}, d_{\text{KL}}(\mu, \lambda)\rangle &\geq \inf_{\lambda \in \Theta_{I^*}^{\complement}} \sum_{s = 1}^{t'-1} \langle \tilde{w}_{s}, d_{\text{KL}}(\mu, \lambda)\rangle - |\mathcal{A}|^2 D_{\M}
\end{align*}

Lemma 14 in \citet{degenne_non-asymptotic_2019} (Appendix \ref{lem:concentration_events_kl_diff_upper_bound}) yields:
\begin{equation*}
    \inf_{\lambda \in \Theta_{I^*}^{\complement}} \sum_{s = 1}^{t'-1} \langle \tilde{w}_{s}, d_{\text{KL}}(\mu, \lambda)\rangle \geq \inf_{\lambda \in \Theta_{I^*}^{\complement}} \sum_{s = n_{0} + 1}^{t'-1} \langle \tilde{w}_{s}, d_{\text{KL}}(\mu_{s-1}, \lambda)\rangle - L_{\M} \sqrt{2\|\sigma^2\|_{\infty}f(t)} \sum_{s = n_{0} + 1}^{t'- 1}\sum_{a \in [d]} \frac{\tilde{w}_{s,a}}{\sqrt{\tilde{N}_{s-1,a}}}
\end{equation*}

Lemma \ref{lem:summation_ratio_w_and_empirical_counts} shows: 
\begin{equation*}
    \sum_{a \in [d]} \sum_{s=n_{0}+1}^{t'-1} \frac{\tilde{w}_{s,a}}{\sqrt{\tilde{N}_{s-1,a}}} \leq d  |\mathcal{A}|^2 + 2\sqrt{2 d (t'-1) K} \leq d  |\mathcal{A}|^2 + 2\sqrt{2 d t K}
\end{equation*}

Since $\langle \tilde{w}_{s}, d_{\text{KL}}(\mu_{s-1}, \lambda)\rangle \geq 0$, dropping all the rounds for which $I_s \neq I^*$ yields:
\begin{equation*}
    \inf_{\lambda \in \Theta_{I^*}^{\complement}} \sum_{s = n_{0} + 1}^{t'-1} \langle \tilde{w}_{s}, d_{\text{KL}}(\mu_{s-1}, \lambda)\rangle \geq \inf_{\lambda \in \Theta_{I^*}^{\complement}} \sum_{n_{0} + 1 \leq s \leq t'-1, I_s = I^*} \langle \tilde{w}_{s}, d_{\text{KL}}(\mu_{s-1}, \lambda)\rangle
\end{equation*}

Combining the saddle-point property of $\A^{A}_{I^*}$ and $R_{t'-1}^{A} \leq R_{t}^{A}$, we obtain:
\begin{align*}
	\inf_{\lambda \in \Theta_{I^*}^{\complement}} \sum_{n_{0} + 1 \leq s \leq t'-1, I_s = I^*} \langle \tilde{w}_{s}, d_{\text{KL}}(\mu_{s-1}, \lambda)\rangle &\geq \max_{A \in \mathcal{A}} \sum_{n_{0} + 1 \leq s \leq t'-1, I_s = I^*} \langle \bm{1}_{A}, r_s \rangle  - R_{t}^{A} \\&\quad \quad \quad-   \sum_{n_{0} + 1 \leq s \leq t'-1, I_s = I^*} \langle \tilde{w}_{s}, C_s \rangle
\end{align*} 

Under the concentration event $\mathcal{E}_t$, we have $d_{\text{KL}}(\mu_{s,a}, \mu_a) \leq \frac{f(t_b)}{\tilde{N}_{s,a}} \leq \frac{f(s)}{\tilde{N}_{s,a}}$ for all $s \geq t_{b}$. Hence, we obtain that: $\mu \in [\alpha_{s,a}, \beta_{s,a}]$ for all $s \geq t_{b}$. Combined with the definition of $r_s$, this implies that: for all $a \in [d]$, $r_{s,a} \geq d_{\text{KL}}(\mu_a, \lambda_{s,a})$. Dropping all the the rounds for which $s < t_b$ yields:
\begin{align*}
    \max_{A \in \mathcal{A}} \sum_{n_{0} + 1 \leq s \leq t'-1, I_s = I^*} \langle \bm{1}_{A}, r_s \rangle  &\geq \max_{A \in \mathcal{A}} \sum_{t_b \leq s \leq t'-1, I_s = I^*} \langle \bm{1}_{A}, d_{\text{KL}}(\mu, \lambda_{s}) \rangle 
\end{align*}

Combining $\frac{1}{|T_{t'-1, I^*}| - |T_{t_b, I^*}|} \sum_{t_b \leq s \leq t'-1, I_s = I^*} \delta_{\lambda_s} \in \mathcal{P}\left(\Theta_{I^*}^{\complement}\right)$ (average of diracs in $(\lambda_s)_{s}$), the dual formulation of $D_{\nu}$ and the fact that $|T_{t'-1, I^*}| \geq |T_{t, I^*}|-1$, we obtain that:
\begin{align*}
     \max_{A \in \mathcal{A}} \sum_{t_b \leq s \leq t'-1, I_s = I^*} \langle \bm{1}_{A}, d_{\text{KL}}(\mu, \lambda_{s}) \rangle &\geq (|T_{t'-1, I^*}| - |T_{t_b, I^*}|) \inf_{q \in \mathcal{P}\left(\Theta_{I^*}^{\complement}\right)}\max_{A \in \mathcal{A}}  \E_{\lambda \sim q} \left[\langle \bm{1}_{A}, d_{\text{KL}}(\mu, \lambda)\right]\rangle\\
    &\geq (|T_{t, I^*}| - 1 - t_b) D_{\nu}
\end{align*}

Combining these inequalities, we obtain an upper bound on $|T_{t, I^*}|$: for $t<\tau_{\delta}$, under $\mathcal{E}_t$,
\begin{align*}
	&\frac{\beta(t,\delta) + R_{t}^{A} + c_{t}}{D_{\nu}} \geq |T_{t, I^*}| - 1 - t_b  \\
	\text{where }\quad& c_t = L_{\M} \sqrt{2\|\sigma^2\|_{\infty}f(t)} \left(d  |\mathcal{A}|^2 + 2\sqrt{2 d t K}\right) + |\mathcal{A}|^2 D_{\M} \\
	&\quad \quad \quad \quad +  L_{\M} \sqrt{2 t f(t) d \|\sigma^2\|_{\infty}  K} + \sum_{n_{0} + 1 \leq s \leq t'-1, I_s = I^*} \langle \tilde{w}_{s}, C_s \rangle
\end{align*}

\subsection{Stopping Time Upper Bound} \label{stopping_time_upper_bound}

Combining the upper and lower bounds on $|T_{t, I^*}|$ (Appendices \ref{appendix_incorrect_candidate_answer} and \ref{appendix_correct_candidate_answer}) yields: for $t<\tau_{\delta}$, under $\mathcal{E}_t$,
\begin{equation} \label{eq:solution_t_delta}
    t \leq \frac{\beta(t,\delta)}{D_{\nu}} + C_{\nu}\left(R_{t}^{A} + h(t)\right)
\end{equation}
where $C_{\nu} \defeq  \frac{1}{D_{\nu}} + \frac{2(|\I|-1)}{C_b \epsilon_{\nu}}$ and 
\begin{align*}
h(t) &\defeq \frac{1}{C_{\nu}} \left( \frac{c_t}{D_{\nu}} + 2t_b + 1 + \frac{2}{C_b \epsilon} \left( \sum_{s \leq t, I_s \neq I^*} \langle \tilde{w}_s, C_s\rangle + f(t) \left( d |\mathcal{A}|^2 + 2d \ln \left( \frac{t K}{d} \right) \right)  \right) \right)
\end{align*}

Lemma \ref{lem:lemma_15_16_degenne} yields:
\begin{align*}
        \sum_{s \geq n_{0}+1}^{t} \langle \tilde{w}_s, C_{s} \rangle &\leq  f(t) \left( d |\mathcal{A}|^2 + 2d \ln \left( \frac{t K}{d} \right) \right) \\
        &\quad \quad + 2L_{\M} \sqrt{2\|\sigma^2\|_{\infty}f(t)}  \left( d  |\mathcal{A}|^2 + 2\sqrt{2 d t K} \right)
\end{align*}

Let $b=1$. Using that $f(t) = \Omega(\ln(t))$ and $t_b = \sqrt{t}$, we obtain that: $h(t) = O\left( \sqrt{t \ln(t)}\right)$. Let $T_{0}(\delta)\defeq \max \left\{t \in \mathbb{N}: t \leq \frac{\beta(t,\delta)}{D_{\nu}} + C_{\nu}(R_{t}^{A} + h(t)) \right\}$ be the upper bound on time such that the Equation \ref{eq:solution_t_delta} is satisfied. Since the set is non empty and bounded, $T_{0}(\delta) \in \mathbb{N}$. The set is bounded since the learner has sublinear cumulative regret, $R_{t}^{A} = o(t)$, and $\frac{\beta(t,\delta)}{D_{\nu}} + C_{\nu}(R_{t}^{A} + h(t)) = O\left(R_{t}^{A} + \sqrt{t \ln(t)}\right)$.

\begin{theorem*} 
    Let $\mathcal{M}$ bounded. The sample complexity of the instantiated CombGame meta-algorithm on bandit $\mu \in \mathcal{M}$ satisfies:
    \begin{align*}
     \E_{\nu}[\tau_{\delta}] \leq T_{0}(\delta) + \frac{2ed}{c^2} \quad  \text{ with } \quad  T_{0}(\delta) \defeq \max \left\{t \in \mathbb{N}: t \leq \frac{\beta(t,\delta)}{D_{\nu}} + C_{\nu}(R_{t}^{A} + h(t)) \right\}
    \end{align*}
    where $c > 0$ is the parameter of the exploration bonus $f(t)$ when taking $b=1$. The reminder terms are: the approximation error $h(t) = O\left(\sqrt{t\ln(t)}\right)$, the learner's cumulative regret $R_{t}^{A}$ and a constant $C_{\nu}$ depending on the distribution. 
    
    The instantiated CombGame meta-algorithm is an asymptotically optimal algorithm. 
\end{theorem*}
\begin{proof}
By the absurd, we assume there exists $t > T_{0}(\delta)$ such that $\mathcal{E}_t \cap \{t < \tau_{\delta}\} \neq \emptyset$. Under event $\mathcal{E}_t$ combining $t > T_{0}(\delta)$ and $t < \tau_{\delta}$ yields the following contradiction:
\begin{align*}
    \frac{\beta(t,\delta)}{D_{\nu}} + C_{\nu}(R_{t}^{A} + h(t))< t &\leq \frac{\beta(t,\delta)}{D_{\nu}} + C_{\nu}(R_{t}^{A} + h(t))
\end{align*}

Therefore, we have $\mathcal{E}_t \cap \{t < \tau_{\delta}\} = \emptyset$. Hence, for all $t> T_{0}(\delta)$, $\mathcal{E}_t \subset \{\tau_{\delta} \leq t\}$. Applying Lemma \ref{lem:upper_bound_finite_time_decomposition} concludes the proof of the finite-time upper bound. Taking the limit $\delta \rightarrow 0$ yields that the instantiated CombGame meta-algorithm is an asymptotically optimal algorithm. 
\end{proof}

\subsection{Technical Arguments}   \label{technicalities}

In Appendix \ref{technicalities}, we prove technical arguments on C-Tracking (Appendix \ref{tracking_results_technicalities}) and on concentration events (Appendix \ref{concentration_results_technicalities}).

\subsubsection{Tracking Arguments}  \label{tracking_results_technicalities}

Let $\A_{|a} = \{A \in \A: a \in A\}$ be the set of actions containing the arm $a$ and $B_t = \text{supp}\left( \sum_{s=1}^{t} w_{s} \right)$. Sparse C-Tracking is defined as: $A_t \in \argmin_{A \in B_t} \frac{\tilde{N}_{t-1,a}}{\sum_{s=1}^{t} w_{s,A}}$ for all $t > n_{0}$. Lemma \ref{lem:c_tracking_control_empirical_counts} controls the deviation between the empirical count of sampled actions, $N_{t,A}$, and the cumulative sum of pulling proportions, $\sum_{s=1}^{t} w_{s,A}$. This is an adaptation of Lemma 7 in \citet{degenne_non-asymptotic_2019}.

\begin{lemma} \label{lem:c_tracking_control_empirical_counts}
Using sparse C-Tracking, we have: for all $t \geq n_{0}$ and for all $A \in \mathcal{A}$, and all $a \in [d]$,
\begin{align*}
    \sum_{s=1}^{t} w_{s,A} - (|\mathcal{A}|-1) &\leq N_{t,A}\leq 1 + \sum_{s=1}^{t} w_{s,A} \\
    \sum_{s=1}^{t} \tilde{w}_{s,a} - (|\A|-1) |\A_{|a}| &\leq \tilde{N}_{t,a} \leq |\A_{|a}| + \sum_{s=1}^{t} \tilde{w}_{s,a}
\end{align*}
\end{lemma}
\begin{proof}
If $A \notin B_t$, we have $N_{t,A} =0$ and $\sum_{s=1}^{t} w_{s,A} = 0$. Hence the first inequalities are immediate. Let $A \in B_t$ and $S_{t,A} = \sum_{s=1}^{t} w_{s,A}$. We will prove $N_{t,A}\leq 1 + S_{t,A}$ by induction. At $t = n_{0}$, the result is true based on the initialization: $S_{n_{0},A} = 1$ if $A \in B_{n_{0}}$ and $S_{n_{0},A} = 0$ else. Assume that $N_{s,A} \leq S_{s,A} + 1$ for all $A \in \mathcal{A}$ and all $s \leq t-1$. Let's prove that it holds at round $t$ too. If $A \neq A_t$, the induction property yields: $N_{t,A} = N_{t-1, A} \leq S_{t-1,A} + 1 \leq S_{t,A} + 1 $. Assume $A = A_t$, then: 
\begin{equation*}
	\frac{N_{t,A_t}}{S_{t,A_t}} = \frac{N_{t-1,A_t}}{S_{t,A_t}} + \frac{1}{S_{t,A_t}} = \frac{1}{S_{t,A_t}} + \min_{A \in \mathcal{A}} \frac{N_{t-1,A}}{S_{t,A}} \leq  \frac{1}{S_{t,A_t}} +  1
\end{equation*}
where the last inequality is shown by the absurd. If $\min_{A \in \mathcal{A}} \frac{N_{t-1,A}}{S_{t,A}} \leq  1$ doesn't hold, we have for all $A \in \mathcal{A}$, $N_{t-1,A} > S_{t,A}$. Summing these strict inequalities yields a contradiction: $t-1 = \sum_{A \in \mathcal{A}} N_{t-1,A} > \sum_{A \in \mathcal{A}} S_{t,A} = \sum_{s=1}^t \sum_{A \in \mathcal{A}} w_{s,A} = t$. Therefore, we have $N_{t,A_t}\leq 1 + S_{t,A_t}$. This concludes the induction.

Combining the previous upper bound and $t = \sum_{A \in \mathcal{A}} N_{t,A} = \sum_{A \in \mathcal{A}} S_{t,A}$ yield the lower bound:
\begin{equation*}
    N_{t,A} = t - \sum_{A' \neq A}N_{t,A'} \geq t - \sum_{A' \neq A}(S_{t,A'} + 1) = S_{t,A} - (|\mathcal{A}|-1)
\end{equation*}

Applying the linear map $W_{\A}$ on the previous inequalities yield the counterpart at the arms level:
\begin{align*}
   \sum_{s=1}^{t} \tilde{w}_{s,a} - (|\A|-1) |\A_{|a}| \leq \tilde{N}_{t,a} \leq |\A_{|a}| + \sum_{s=1}^{t} \tilde{w}_{s,a}
\end{align*}
\end{proof}

A better bound for C-Tracking was proven in Theorem 6 of \citet{degenne_structure_2020}. They obtain that for all $t\in \N$ and $A \in \A$,
\begin{equation*}
	- \ln (|\A|) \leq N_{t,A} - \sum_{s=1}^{t} w_{s,A} \leq 1
\end{equation*}

Lemma 8 from \citet{degenne_non-asymptotic_2019} is a technical lemma on summations.

\begin{lemma*}[Lemma 8 in \citet{degenne_non-asymptotic_2019}] \label{lem:lemma_8_degenne}
    For $t \geq t_0 \geq 1$ and $(x_s)_{s \in [t]}$ non negative real numbers such that $\sum_{s = 1}^{t_{0}-1} x_s > 0$,
    \begin{align*}
        &\sum_{s = t_{0}}^{t} \frac{x_s}{\sqrt{\sum_{r=1}^{s} x_r} } \leq  2 \sqrt{ \sum_{s = 1}^{t} x_s } - 2 \sqrt{ \sum_{s = 1}^{t_0 - 1} x_s}\\
        &\sum_{s = t_{0}}^{t} \frac{x_s}{\sum_{r=1}^{s} x_r} \leq \ln \left(\sum_{s = 1}^{t} x_s \right) - \ln \left(\sum_{s = 1}^{t_0 - 1} x_s \right)
    \end{align*}
\end{lemma*}

Lemma \ref{lem:summation_ratio_w_and_empirical_counts} controls the summation of ratios $\frac{\tilde{w}_{s,a}}{\sqrt{\tilde{N}_{s,a}}}$ and $\frac{\tilde{w}_{s,a}}{\sqrt{\tilde{N}_{s-1,a}}}$ over arms and time. This is an adaptation of Lemma 9 in \citet{degenne_non-asymptotic_2019} to our setting.

\begin{lemma} \label{lem:summation_ratio_w_and_empirical_counts}
    Let $(\tilde{w}_s)_{s \in \mathbb{N}} \in \tsimplex^{\mathbb{N}}$ and $\tilde{N}_t$ obtained with sparse C-Tracking. Then,
    \begin{align*}
        \sum_{a \in [d]} \sum_{s=n_{0}}^{t} \frac{\tilde{w}_{s,a}}{\sqrt{\tilde{N}_{s,a}}} \leq d  |\mathcal{A}|^2 + 2\sqrt{d t K} \quad &\text{ and } \quad \sum_{a \in [d]} \sum_{s=n_{0}+1}^{t} \frac{\tilde{w}_{s,a}}{\sqrt{\tilde{N}_{s-1,a}}} \leq d  |\mathcal{A}|^2 + 2\sqrt{2 d t K} \\
        \sum_{a \in [d]} \sum_{s=n_{0}}^{t} \frac{\tilde{w}_{s,a}}{\tilde{N}_{s,a}} \leq d |\mathcal{A}|^2 + d \ln \left( \frac{t K}{d} \right) \quad &\text{ and } \quad  \sum_{a \in [d]} \sum_{s=n_{0}+1}^{t} \frac{\tilde{w}_{s,a}}{\tilde{N}_{s-1,a}} \leq d |\mathcal{A}|^2 + 2d \ln \left( \frac{t K}{d} \right)
    \end{align*}
\end{lemma}
\begin{proof}
First, let's prove inequalities 1 and 3. Let $a \in [d]$ and $t_{0,a}$ be the first time such that: $\sum_{s=1}^{t_{0,a}-1} \tilde{w}_{s,a} > (|\mathcal{A}| - 1)|\mathcal{A}_{|a}| + 1$. Since $\tilde{w}_{t_{0,a} - 1,a} \leq 1$, we have $\sum_{s=1}^{t_{0,a}-1} \tilde{w}_{s,a} \leq (|\mathcal{A}| - 1)|\mathcal{A}_{|a}| + 2$. Since $\tilde{N}_{s,a} \geq 1$ for $s \geq n_{0}$, we obtain:
\begin{align*}
    \sum_{s=n_{0}}^{t} \frac{\tilde{w}_{s,a}}{\sqrt{\tilde{N}_{s,a}}} = \sum_{s=n_{0}}^{t_{0,a} - 1} \frac{\tilde{w}_{s,a}}{\sqrt{\tilde{N}_{s,a}}} + \sum_{s=t_{0,a}}^{t} \frac{\tilde{w}_{s,a}}{\sqrt{\tilde{N}_{s,a}}} &\leq \sum_{s=n_{0}}^{t_{0,a} - 1} \tilde{w}_{s,a} + \sum_{s=t_{0,a}}^{t} \frac{\tilde{w}_{s,a}}{\sqrt{\tilde{N}_{s,a}}} \\&\leq (|\mathcal{A}| - 1)|\mathcal{A}_{|a}| + 2 + \sum_{s=t_{0,a}}^{t} \frac{\tilde{w}_{s,a}}{\sqrt{\tilde{N}_{s,a}}} \\
    \sum_{s=n_{0}}^{t} \frac{\tilde{w}_{s,a}}{\tilde{N}_{s,a}} = \sum_{s=n_{0}}^{t_{0,a} - 1} \frac{\tilde{w}_{s,a}}{\tilde{N}_{s,a}} + \sum_{s=t_{0,a}}^{t} \frac{\tilde{w}_{s,a}}{\tilde{N}_{s,a}} &\leq \sum_{s=n_{0}}^{t_{0,a} - 1} \tilde{w}_{s,a} + \sum_{s=t_{0,a}}^{t} \frac{\tilde{w}_{s,a}}{\tilde{N}_{s,a}} \\
    &\leq (|\mathcal{A}| - 1)|\mathcal{A}_{|a}| + 2 + \sum_{s=t_{0,a}}^{t} \frac{\tilde{w}_{s,a}}{\tilde{N}_{s,a}}
\end{align*}

Combining Lemma \ref{lem:c_tracking_control_empirical_counts} and Lemma 8 in \citet{degenne_non-asymptotic_2019} for $x_s = \tilde{w}_{s,a}$, we obtain:
\begin{align*}
    \sum_{s=t_{0,a}}^{t} \frac{\tilde{w}_{s,a}}{\sqrt{\tilde{N}_{s,a}}}  &\leq \sum_{s=t_{0,a}}^{t} \frac{\tilde{w}_{s,a}}{\sqrt{\sum_{r=1}^{s} \tilde{w}_{r,a} -  (|\mathcal{A}| - 1)|\mathcal{A}_{|a}|}} \\
    &\leq 2 \sqrt{ \sum_{s = 1}^{t} \tilde{w}_{s,a} -  (|\mathcal{A}| - 1)|\mathcal{A}_{|a}| } - 2 \sqrt{ \sum_{s = 1}^{t_{0,a} - 1} \tilde{w}_{s,a} -  (|\mathcal{A}| - 1)|\mathcal{A}_{|a}| } \\    
    \sum_{s=t_{0,a}}^{t} \frac{\tilde{w}_{s,a}}{\tilde{N}_{s,a}} &\leq \sum_{s=t_{0,a}}^{t} \frac{\tilde{w}_{s,a}}{\sum_{r=1}^{s} \tilde{w}_{r,a} -  (|\mathcal{A}| - 1)|\mathcal{A}_{|a}|}\\
    &\leq \ln \left( \sum_{s = 1}^{t} \tilde{w}_{s,a} -  (|\mathcal{A}| - 1)|\mathcal{A}_{|a}| \right) - \ln \left( \sum_{s = 1}^{t_{0,a} - 1} \tilde{w}_{s,a} -  (|\mathcal{A}| - 1)|\mathcal{A}_{|a}| \right)
\end{align*}

Using that $\sum_{s=1}^{t_{0,a}-1} \tilde{w}_{s,a} -  (|\mathcal{A}| - 1)|\mathcal{A}_{|a}| > 1$, we obtain: $\sum_{s=t_{0,a}}^{t} \frac{\tilde{w}_{s,a}}{\sqrt{\tilde{N}_{s,a}}} \leq 2 \sqrt{ \sum_{s = 1}^{t} \tilde{w}_{s,a}}$ and $\sum_{s=t_{0,a}}^{t} \frac{\tilde{w}_{s,a}}{\tilde{N}_{s,a}} \leq \ln \left(  \sum_{s = 1}^{t} \tilde{w}_{s,a}\right)$. Combining the concavity of $x \mapsto \sqrt{x}$ and $x \mapsto \ln(x)$ and $\sum_{a \in [d]}\sum_{s = 1}^{t} \tilde{w}_{s,a} \leq t K$ yield by summation:
\begin{align*}
    \sum_{a \in [d]} \sum_{s=t_{0,a}}^{t} \frac{\tilde{w}_{s,a}}{\sqrt{\tilde{N}_{s,a}}}  &\leq 2 \sum_{a \in [d]} \sqrt{ \sum_{s = 1}^{t} \tilde{w}_{s,a}}  \leq 2 \sqrt{d t K}\\
    \sum_{a \in [d]} \sum_{s=t_{0,a}}^{t} \frac{\tilde{w}_{s,a}}{\tilde{N}_{s,a}} &\leq \sum_{a \in [d]} \ln \left(  \sum_{s = 1}^{t} \tilde{w}_{s,a}\right) \leq d \ln \left( \frac{t K}{d} \right)
\end{align*}

Therefore, we obtain: $\sum_{a \in [d]} \sum_{s=n_{0}}^{t} \frac{\tilde{w}_{s,a}}{\sqrt{\tilde{N}_{s,a}}} \leq d  |\mathcal{A}|^2 + 2\sqrt{d t K}$ and $\sum_{a \in [d]} \sum_{s=n_{0}}^{t} \frac{\tilde{w}_{s,a}}{\tilde{N}_{s,a}} \leq d  |\mathcal{A}|^2 + d \ln \left( \frac{t K}{d}\right) $. For all $s \geq n_{0}$, we have $N_{s-1,a} \geq 1$, hence $N_{s-1,a} \geq \frac{1}{2}N_{s,a}$. Plugging this inequality in the sum starting from $t_{0,a}$ yields: $\sum_{a \in [d]} \sum_{s=n_{0}+1}^{t} \frac{\tilde{w}_{s,a}}{\sqrt{\tilde{N}_{s-1,a}}} \leq d  |\mathcal{A}|^2 + 2\sqrt{2 d t K}$ and $\sum_{a \in [d]} \sum_{s=n_{0}+1}^{t} \frac{\tilde{w}_{s,a}}{\tilde{N}_{s-1,a}} \leq d |\mathcal{A}|^2 + 2d \ln \left( \frac{t K}{d} \right)$. 
\end{proof}

\subsubsection{Concentration Arguments}  \label{concentration_results_technicalities}

Let $t_{b} = t^{1/(1+b)} < t$, $L_{\M}$ the Lipschitz constant of $x \mapsto d_{KL}(x, y)$ ($\M$ bounded). The sequence of concentrations events $(\mathcal{E}_t)_{t \geq 1}$ for the exploration bonus $f$ with parameters $c>0$ and $b>0$ was defined as:
\begin{align*}
    \mathcal{E}_t &\defeq \left\{\forall s \leq t, \forall a \in [d], \quad \tilde{N}_{s,a} d_{\text{KL}}(\mu_{s,a}, \mu_a) \leq f\left(t_b\right)\right\}
\end{align*}

Lemma 14 in \citet{degenne_non-asymptotic_2019} controls the deviation $\left|d_{\text{KL}}(\mu_{s-1,a}, \lambda_{a}) - d_{\text{KL}}(\mu_{a}, \lambda_{a})\right|$. Its proof is similar to the beginning of the proof of Lemma \ref{lem:concentration_events_summed_kl_diff_upper_bound}.

\begin{lemma*}[Lemma 14 in \citet{degenne_non-asymptotic_2019}] \label{lem:concentration_events_kl_diff_upper_bound}
    Let $\M$ bounded. Under $\mathcal{E}_t$, for all $s \in [t]$, $a \in [d]$ any $\lambda \in \mathcal{M}$,
    \begin{equation*}
        \left|d_{\text{KL}}(\mu_{s-1,a}, \lambda_{a}) - d_{\text{KL}}(\mu_{a}, \lambda_{a})\right| \leq L_{\M} \sqrt{2\|\sigma^2\|_{\infty} \frac{f(t)}{\tilde{N}_{s-1,a}}}
    \end{equation*}
\end{lemma*}

Lemma \ref{lem:concentration_events_summed_kl_diff_upper_bound} controls the weighted sum of deviations, $\langle \tilde{N}_t, d_{\text{KL}}(\mu_t, \lambda) - d_{\text{KL}}(\mu, \lambda) \rangle$. This is an adaptation of Lemma 17 in \citet{degenne_non-asymptotic_2019}.

\begin{lemma} \label{lem:concentration_events_summed_kl_diff_upper_bound}
    Let $\M$ be bounded. Under $\mathcal{E}_t$, for any $\lambda \in \mathcal{M}$,
    \begin{equation*}
        \langle \tilde{N}_t, d_{\text{KL}}(\mu_{t}, \lambda)\rangle \geq \langle \tilde{N}_t, d_{\text{KL}}(\mu, \lambda)\rangle - L_{\M}\sqrt{2 t f(t) \|\sigma^2\|_{\infty} d K}
    \end{equation*}
\end{lemma}
\begin{proof}
Using the Lipschitz property of $x \mapsto d_{\text{KL}}(x,y)$, we have $d_{\text{KL}}(\mu_{t,a}, \lambda_{a}) - d_{\text{KL}}(\mu_{a}, \lambda_{a}) \geq - L_{\M} |\mu_{t,a} - \mu_{a}|$. The sub-Gaussian property when (a) or the direct formula for Gaussian when (b), implies that $|\mu_{t,a} - \mu_{a}| \leq \sqrt{2\sigma_{a}^2d_{\text{KL}}(\mu_{t,a}, \mu_{a})}$. Under $\mathcal{E}_t$, we have $d_{\text{KL}}(\mu_{t,a}, \mu_{a}) \leq \frac{f\left(t_b\right)}{\tilde{N}_{t,a}}$. Combining these inequalities, $f$ increasing and $\sigma_a^{2} \leq \|\sigma^2\|_{\infty}$, we obtain: $d_{\text{KL}}(\mu_{t,a}, \lambda_{a}) - d_{\text{KL}}(\mu_{a}, \lambda_{a}) \geq - L_{\M} \sqrt{2\|\sigma^2\|_{\infty}\frac{f(t)}{\tilde{N}_{t,a}}}$ for all $a \in [d]$. Summing with weights $\tilde{N}_t$ yields:
\begin{equation*}
    \langle \tilde{N}_t, d_{\text{KL}}(\mu_{t}, \lambda)\rangle - \langle \tilde{N}_t, d_{\text{KL}}(\mu, \lambda)\rangle  \geq  - L_{\M}\sqrt{2f(t) \|\sigma^2\|_{\infty}} \sum_{a \in [d]} \sqrt{\tilde{N}_{t,a}}
\end{equation*}

Since $x \mapsto \sqrt{x}$ is concave, we have $\sum_{a \in [d]} \sqrt{\tilde{N}_{t,a}} \leq \sqrt{d \sum_{a \in [d]} \tilde{N}_{t,a}} \leq \sqrt{dtK}$. This concludes the proof: $\langle \tilde{N}_t, d_{\text{KL}}(\mu_{t}, \lambda)\rangle \geq \langle \tilde{N}_t, d_{\text{KL}}(\mu, \lambda)\rangle - L_{\M}\sqrt{2 t f(t) \|\sigma^2\|_{\infty} d K}$.
\end{proof}

Lemma \ref{lem:lemma_15_16_degenne} controls one term of the slack appearing in the saddle-point property, the one linked to $r_s$:  $\sum_{s \geq n_{0}+1}^{t} \langle \tilde{w}_s, C_{s} \rangle$. This is an adaptation of Lemmas 15 and 16 in \citet{degenne_non-asymptotic_2019}.

\begin{lemma} \label{lem:lemma_15_16_degenne}
    Let $\M$ be bounded and 
    \begin{equation*}
        D_{s,a} = \max \left\{ 2L_{\M} \sqrt{2 \sigma_{a}^2 \frac{f(\max\{s-1, t_b\})}{\tilde{N}_{s-1,a}}}, \frac{f(\max\{s-1, t_b\})}{\tilde{N}_{s-1,a}}\right\}
    \end{equation*}
    Under the event $\mathcal{E}_t$, for all $s \in [t]$: $\sup_{\phi \in [\alpha_{s,a}, \beta_{s,a}]} (r_{s,a} - d_{\text{KL}}(\phi,\lambda_{s,a})) \leq D_{s,a}$. Let $C_{s,a} = r_{s,a} - d_{\text{KL}}(\mu_{s-1,a}, \lambda_{s,a})$, we obtain:
    \begin{align*}
        \sum_{s \geq n_{0}+1}^{t} \langle \tilde{w}_s, C_{s} \rangle &\leq  f(t) \left( d |\mathcal{A}|^2 + 2d \ln \left( \frac{t K}{d} \right) \right) + 2L_{\M} \sqrt{2\|\sigma^2\|_{\infty}f(t)}  \left( d  |\mathcal{A}|^2 + 2\sqrt{2 d t K} \right)
\end{align*}
\end{lemma}
\begin{proof}
We recall that: $r_{s,a} = \max \left\{ \frac{f(s-1)}{\tilde{N}_{s-1,a}}, \max_{\phi \in \{\alpha_{s,a}, \beta_{s,a}\}} d_{\text{KL}}(\phi, \lambda_{s,a}) \right\}$ for all $a \in [d]$. Assume $r_{s,a} = \frac{f(s-1)}{\tilde{N}_{s-1,a}}$. Since $d_{\text{KL}}$ is positive, $f$ is increasing and $s-1 \leq \max\{s-1,t_b\}$, we have:
\begin{equation*}
	\sup_{\phi \in [\alpha_{s,a}, \beta_{s,a}]} (r_{s,a} - d_{\text{KL}}(\phi,\lambda_{s,a})) \leq \frac{f(s-1)}{\tilde{N}_{s-1,a}} \leq \frac{f(\max\{s-1, t_b\})}{\tilde{N}_{s-1,a}} \leq D_{s,a}
\end{equation*}

Assume $r_{s,a} = d_{\text{KL}}(\phi_{s,a}, \lambda_{s,a})$ where $\phi_{s,a} = \argmax_{\phi \in \{\alpha_{s,a}, \beta_{s,a}\}} d_{\text{KL}}(\phi, \lambda_{s,a})$. By convexity of $x \mapsto d_{\text{KL}}(x,y)$, we have $d_{\text{KL}}(\phi_{s,a}, \lambda_{s,a}) = \max_{\phi \in [\alpha_{s,a}, \beta_{s,a}]} d_{\text{KL}}(\phi, \lambda_{s,a})$. Upper bounding yields: $\sup_{\phi \in [\alpha_{s,a}, \beta_{s,a}]} (r_{s,a} - d_{\text{KL}}(\phi,\lambda_{s,a})) \leq \sup_{\phi, \eta \in [\alpha_{s,a}, \beta_{s,a}]} |d_{\text{KL}}(\eta,\lambda_{s,a}) - d_{\text{KL}}(\phi,\lambda_{s,a})|$. The Lipschitz property of $x \mapsto d_{\text{KL}}(x,y)$ yields: $|d_{\text{KL}}(\eta,\lambda_{s,a}) - d_{\text{KL}}(\phi,\lambda_{s,a})| \leq L_{\M} |\eta - \phi|$. Under event $\mathcal{E}_t$, combining the sub-Gaussian property when (a) or the direct formula for Gaussian when (b) and $f$ increasing, we obtain:
\begin{equation*}
    \sup_{\phi \in [\alpha_{s,a}, \beta_{s,a}]} (r_{s,a} - d_{\text{KL}}(\phi,\lambda_{s,a})) \leq 2 L_{\M} \sqrt{2 \sigma_{a}^2 \frac{f(\max\{s-1, t_b\})}{N_{s-1,a}}}  \leq D_{s,a}
\end{equation*}

For the second part of the lemma, since $\mu_{s-1,a} \in [\alpha_{s,a}, \beta_{s,a}]$, we have $C_{s,a} \leq D_{s,a}$ and $\sum_{s} \langle \tilde{w}_s, C_{s} \rangle \leq \sum_{s} \langle \tilde{w}_s, D_{s} \rangle$. Applying Lemma \ref{lem:summation_ratio_w_and_empirical_counts} twice, we obtain:
\begin{align*}
    \sum_{s \geq n_{0}+1}^{t} \sum_{a \in [d]} \tilde{w}_{s,a} 2L \sqrt{2 \sigma_{a}^2 \frac{f(\max\{s-1, t_b\})}{N_{s-1,a}}}  &\leq 2L_{\M} \sqrt{2\|\sigma^2\|_{\infty}f(t)}\sum_{s \geq n_{0}+1}^{t} \sum_{a \in [d]} \frac{\tilde{w}_{s,a}}{\sqrt{N_{s-1,a}}} \\
    &\leq 2L_{\M} \sqrt{2\|\sigma^2\|_{\infty}f(t)}  \left( d  |\mathcal{A}|^2 + 2\sqrt{2 d t K} \right)   \\
    \sum_{s \geq n_{0}+1}^{t} \sum_{a \in [d]} \tilde{w}_{s,a}\frac{f(\max\{s-1, t_b\})}{N_{s-1,a}} &\leq f(t) \sum_{s \geq n_{0}+1}^{t} \sum_{a \in [d]} \frac{\tilde{w}_{s,a}}{N_{s-1,a}} \\
    &\leq f(t) \left( d |\mathcal{A}|^2 + 2d \ln \left( \frac{t K}{d} \right) \right) 
\end{align*}

Combining $\max(a + b) \leq a + b$ for $D_{s,a}$ and the previous inequalities concludes the proof.
\end{proof}

\section{Unbounded \texorpdfstring{$\M$}- for Gaussian Bandit} \label{appendix_unbounded_setting}

As already discussed in Appendix F of \citet{degenne_non-asymptotic_2019}, the boundedness assumption of $\M$ can be weakened. In particular, for Gaussian bandit where $y \mapsto d_{\text{KL}}(x,y) = \frac{(x-y)^2}{\sigma_{a}^{2}}$ is convex and symmetric, we can remove it completely using concentration events and explicit formulas. 

We sketch the ideas of the required adaptations, the full proof is omitted for the sake of space. The concentration arguments of Appendix \ref{concentration_results_technicalities} are replaced by weaker results: the deviation is controlled for a given $\lambda$, not for an arbitrary $\lambda \in \M$. Similarly, using explicit formulas, we can upper bound the optimistic reward and prove that $\tau_{\delta} < + \infty$. The adaptation is mainly technical and requires to be familiar with the detail of the proof of Theorem \ref{thm:finite_time_upper_bound}.

\paragraph{Bounded $\|\mu_t - \mu\|_{\infty}$} Under event $\mathcal{E}_t$, we have for all $s \leq t$ and all $a \in [d]$, $d_{\text{KL}}(\mu_{s,a}, \mu_a) \leq \frac{f(t)}{\tilde{N}_{s,a}}$. The concentration event yields that $\|\mu_s - \mu\|_{\infty} \leq \sqrt{2\|\sigma^{2}\|_{\infty}\frac{f(t)}{\min_{a \in [d]} \tilde{N}_{s,a}}}$. Hence, $\mu_s$ belongs to a bounded set around $\mu$. When all arms are sampled more than a logarithmic number of time, we have $\lim_{s \rightarrow \infty} \|\mu_s - \mu\|_{\infty} = 0$. 

\paragraph{Explicit formula} Let $I \in \I$, $\tilde{w} \in \tsimplex$ and $\phi \in \Theta_{I}$. Let $\lambda(\phi, I, \tilde{w}) \in \argmin_{\lambda \in \Theta_{I}^{\complement}} \langle \tilde{w}, d_{\text{KL}}(\phi,\lambda)\rangle$ and $\lambda(\phi, I, \tilde{w}, J) \in \argmin_{\lambda \in \Bar{\Theta}_{J}^{I}} \langle \tilde{w}, d_{\text{KL}}(\phi,\lambda)\rangle$. Using $\Theta_{I}^{\complement} = \bigcup_{J \neq I} \Bar{\Theta}_{J}^{I}$, there exists $J(I) \in \I$ such that $\lambda(\phi, I, \tilde{w}) = \lambda(\phi, I, \tilde{w}, J(I))$. Lemma \ref{lem:explicit_solutions_gaussian} proves an explicit formula for $\lambda(\phi, I, \tilde{w}, J)$, which implies an upper bound on $\|\phi - \lambda(\phi, I, \tilde{w})\|_{\infty}$. Let $\phi' \in \Theta_{I}$. When $\|\phi - \phi'\|$ is bounded, applying Lemma \ref{lem:explicit_solutions_gaussian} twice allows to control $\|\lambda(\phi,I,\tilde{w}) - \lambda(\phi',I,\tilde{w})\|_{\infty}$.

\begin{lemma} \label{lem:explicit_solutions_gaussian}
Assume $\M = \R^{d}$. Let $(I, J) \in \mathcal{I}^{2}$, such that $I \neq J$, $\phi \in \R^{d}$ and $\tilde{w} \in \tsimplex$. Let $\lambda(\phi, I, \tilde{w}, J) \in \argmin_{\lambda \in \Bar{\Theta}_{J}^{I}} \langle \tilde{w}, \frac{(\phi -  \lambda)^{2}}{\sigma^2} \rangle$. Then,
\begin{align*}
	\lambda(\phi, I, \tilde{w}, J) &= \begin{cases}
		\phi &\text{if } \phi \in \Bar{\Theta}_{J}^{I} \\
		\phi - (\mu_{a_0} - \alpha_{a_0})\delta_{a_0} &\text{if }a_{0} \in \text{supp}(\tilde{w})^{\complement} \cap I \triangle J \neq \emptyset \\
		\phi - \frac{\langle \1_{J} - \1_{I}, \phi \rangle}{\sum_{\tilde{a} \in I \triangle J} \frac{\sigma_{\tilde{a}}^{2}}{\tilde{w}_{\tilde{a}}}} \left(\frac{\sigma_{a}^{2}}{\tilde{w}_a}\left(\bm{1}_{a \in J} - \bm{1}_{a \in I}\right)\right)_{a \in [d]} &\text{else}
	\end{cases}
\end{align*}
where $\alpha_{a_{0}} = - \frac{(\bm{1}_{J \setminus \{a_{0}\}} - \bm{1}_{I \setminus \{a_{0}\}})^{\intercal} \phi}{\bm{1}_{a_{0} \in J} - \bm{1}_{a_{0} \in I}}$. 
\end{lemma}
\begin{proof}
The proof uses the fact that $\Bar{\Theta}_{J}^{I} = \{\lambda \in \R^{d}: \langle \1_{J} - \1_{I}, \lambda \rangle \geq 0\}$ and the KKT conditions.
\end{proof}

\paragraph{Adapted Lemma \ref{lem:concentration_events_summed_kl_diff_upper_bound}} This lemma is used in Appendix \ref{appendix_correct_candidate_answer} when $I^*(\mu) = I_{t'}$. We apply Lemma \ref{lem:explicit_solutions_gaussian} twice, for $\lambda(\mu_{t'-1}, I^*(\mu), \frac{\tilde{N}_{t'-1}}{t'-1})$ and $\lambda(\mu,I^*(\mu),\frac{\tilde{N}_{t'-1}}{t'-1})$ and we use that $\|\mu_{t'-1} - \mu\|_{\infty} \leq  \sqrt{2\|\sigma^{2}\|_{\infty} f(t)}$ by concentration. Therefore, we can control $\|\lambda(\mu_{t'-1}, I^*(\mu), \frac{\tilde{N}_{t'-1}}{t'-1}) - \lambda(\mu, I^*(\mu), \frac{\tilde{N}_{t'-1}}{t'-1})\|_{\infty}$.

\paragraph{Adapted Lemma \ref{lem:lemma_15_16_degenne}} This lemma is used in Appendix \ref{stopping_time_upper_bound}. Using the closed-form formula for $r_{s,a}$ in Lemma \ref{lem:gaussian_optimistic_reward}, we obtain: for all $s\leq t$ and all $a \in [d]$, $C_{s,a} = \frac{f(s-1)}{\tilde{N}_{s-1,a}} + \sqrt{\frac{2f(s-1)}{\sigma_{a}^2\tilde{N}_{s-1,a}}}|\mu_{s-1,a} - \lambda_{s,a}|$. Using Lemma \ref{lem:explicit_solutions_gaussian} for $\lambda_s  = \lambda(\mu_{s-1}, I_s, \tilde{w}_s)$, we can control $\|\mu_{s-1} - \lambda_{s}\|_{\infty}$.

\paragraph{Adapted Lemma 14 in \citet{degenne_non-asymptotic_2019}} This lemma is used in Appendix \ref{appendix_correct_candidate_answer}. Applying Lemma \ref{lem:explicit_solutions_gaussian} for $\lambda(\mu, I^*, \sum_{s = 1}^{t'-1} \tilde{w}_{s} )$ and using that $\|\mu_s - \mu\|_{\infty} \leq \sqrt{2\|\sigma^{2}\|_{\infty} f(t)}$ for all $s \leq t$, by concentration, we can control $\|d_{\text{KL}}(\mu, \lambda(\mu, I^*, \sum_{s = 1}^{t'-1} \tilde{w}_{s} )) - d_{\text{KL}}(\mu_{s-1}, \lambda(\mu, I^*, \sum_{s = 1}^{t'-1} \tilde{w}_{s} ))\|_{\infty}$. 

\paragraph{Adapted Lemma \ref{lem:loss_inf_norm_upper_lower_bound}} This Lemma is used in Appendix \ref{appendix_learner_cumulative_regret}. The closed-form formula for $r_{t,a}$ in Lemma \ref{lem:gaussian_optimistic_reward} is: $r_{t,a} = \frac{(\mu_{t-1,a} - \lambda_{t,a})^{2}}{2\sigma_{a}^2} + \frac{f(t-1)}{\tilde{N}_{t-1,a}} + \sqrt{\frac{2f(t-1)}{\sigma_{a}^2\tilde{N}_{t-1,a}}}|\mu_{t-1,a} - \lambda_{t,a}|$ for all $a \in [d]$.  Using Lemma \ref{lem:explicit_solutions_gaussian} for $\lambda_s  = \lambda(\mu_{s-1}, I_s, \tilde{w}_s)$, we obtain an upper bound on $\|r_t\|_{\infty}$, which will be used to bound $R_{t}^A$.

\paragraph{Adapted proof of $\tau_{\delta} < + \infty$} We use this result in the proof of Theorem \ref{thm:delta_pac} (Appendix \ref{appendix_proof_delta_pac}). Applying Lemma \ref{lem:explicit_solutions_gaussian} for $\lambda(\mu_{t-1}, I^*(\mu), \frac{\tilde{N}_{t-1}}{t-1})$ and using $\lim_{\infty}\langle \frac{\tilde{N}_{t-1}}{t-1}, d_{\text{KL}}(\mu_{t-1}, \lambda(\mu_{t-1}, I^*(\mu), \frac{\tilde{N}_{t-1}}{t-1})) \rangle$, we can conclude similarly.

\section{Implementation Details}  \label{appendix_implementation_details}

\paragraph{D-Tracking} D-Tracking tracks $w_{t}$ instead of $\sum_{s=1}^{t} w_s$ \citep{garivier_optimal_2016}. It can be used instead of C-Tracking. Sparse D-Tracking is defined as: $A_t \in \argmin_{A \in B_t} \frac{N_{t-1,A}}{w_{t,A}}$ where $B_t = \text{supp}(w_t)$. D-Tracking has been shown to empirically outperform C-Tracking \citep{degenne_non-asymptotic_2019, garivier_optimal_2016}. In our experiments, C-Tracking and D-Tracking have similar results, up to a few percent. Therefore, we omit C-Tracking from the graphs. 

In Appendix C of \citet{degenne_pure_2019}, the reason why D-Tracking might fail to converge is discussed. It stems from the fact that D-Tracking does not in general converge to the convex hull of the points it tracks. Due to the non-uniqueness of the optimal allocations, D-Tracking might also fail in our setting. For linear bandits \citet{degenne_gamification_2020} showed that D-Tracking is licensed theoretically in order to obtain asymptotically optimal algorithms. In lights of those facts, whether D-Tracking is theoretically validated in our setting remains open.

\paragraph{One learner} As in \citet{degenne_non-asymptotic_2019}, we consider only one learner $\A^{A}$ instead of partitioning the rounds according to the candidate answer $I_t$. Experimentally, the results when considering $|\I|$ learners are always within a few percent of the one learner implementation. Therefore, we omit them from the graphs.

When considering $|\I|$ learners, one might ask what is the number of called learners before stopping. Since a learner is not used until its corresponding answer is the candidate answer, we expect this number to be small in comparison to $|\I|$. Our experiments validate this intuition: the used learners are the one for $I^*$ and the ones for the most confusing alternatives. Considering a similar game-inspired algorithm, \citet{tirinzoni2020asymptotically} present a rigorous reason for using only one learner instead of $|\I|$ different ones.

\paragraph{Stylized stopping threshold and exploration bonus} As in \citet{degenne_non-asymptotic_2019}, we use stylized stopping threshold $\beta(t, \delta) = \ln \left( \frac{1 + \ln(t)}{\delta}\right)$ and exploration bonus $f(t) = \ln(t)$ instead of the ones licensed by the theory. Despite being unlicensed yet, they are both empirically conservative since the empirical error rate is order of magnitude lower than the theoretical confidence error $\delta$.

\paragraph{Sparsification} As shown in Table \ref{tab:comparison_learners}, the computational complexity of both OFW and LLOO can become a hurdle when $|B_t| \gg d$. This problem was mentioned and tackled in \citet{garber_linearly_2015}. To circumvent it, we use an offline sparsification procedure to obtain an approximation $\tilde{w}_{t,0}$ of $\tilde{w}_{t}$ with sparse support. 

Let $\tilde{w}_{t,0}$ be the approximate solution to the optimization problem $\min_{y \in \text{Im}(W_{\mathcal{A}})} \|y - \tilde{w}_{t}\|_{2}^{2}$ obtained thanks to Algorithm 2 in \citet{garber_linearly_2015}, up to precision $r^2$. By Theorem 2 in \citet{garber_linearly_2015}, this offline smooth and strongly convex optimization algorithm satisfies: $\|x_{s+1} - \tilde{w}_{t}\|_{2}^{2} \leq C \exp \left( - \frac{1}{4 \rho^2}s \right)$. The algorithm maintains a representation $w_{t,0} \in \simplex$. 

When $|B_t| \gg d$, we solve this optimization and use $(\tilde{w}_{t,0}, w_{t,0})$ instead of $(\tilde{w}_{t}, w_{t})$. The parameters of LLOO are modified accordingly to Lemma 10 in \citet{garber_linearly_2015}.

\paragraph{Doubling trick} The horizon $T$ corresponds to the stopping time $\tau_{\delta}$ which is unknown. Therefore, we need to convert the non-anytime learners, Hedge and LLOO, into anytime learners. The geometric doubling trick \citep{besson_what_2018} can be used for that purpose. It preserves the minimax bounds in $R_{t} = O(\sqrt{t})$. In our experiments, we use the geometric doubling trick sequence $\left(\lfloor T_{0} b^{i} \rfloor\right)_{i \in \N}$ where $T_{0} = 200$ and $b = \frac{3 +\sqrt{5}}{2}$ as advocated in \citet{besson_what_2018}.

\paragraph{Covering initialization} When considering a covering initialization, the sole requirement is to observe each arm at least once. Due to the combinatorial nature of the problem, numerous combinations of actions are valid initialization. Since our algorithms on the transformed simplex have a computational cost which is sensitive to $|B_{n_t}|$, we will consider an initialization such that the number of actions $n_{0}$ required to observe all arms is the smallest. When numerous choices achieve lowest $n_{0}$, we choose one arbitrarily. Alternatively one could sample randomly the actions without replacement till observing each arm at least once. This random covering initialization often damages simultaneously the sample complexity and the computational cost.

\paragraph{LLOO's parameters} We recall here the definitions of the geometric parameters for the polytope $\tsimplex$ used in \citet{garber_linearly_2015}. The diameter of $\tsimplex$ is $\diamtsimplex \defeq \max_{x,y \in \tsimplex} \| x - y \|_{2}$. The parameter $\mu_{\A}$ is defined as $\mu_{\A} \defeq \frac{\psi_{\A} \diamtsimplex}{\phi_{\A}}$ where $\psi_{\A}$ and $\phi_{\A}$ are also geometric parameters. A convex polytope admits a description with linear inequalities, $\tsimplex = \{x \in \mathbb{R}^d: A_1 x=b_1 \land  A_2 x \leq b_2\}$. $\phi_{\A}$ is defined as $\phi_{\A} \defeq \min_{A \in \A} \left\{ \min\{ b_{2}(j) - \langle A_{2}(j), \bm{1}_{A} \rangle : j \in [m],  b_{2}(j) > \langle A_{2}(j), \bm{1}_{A} \rangle \}\right\}$. It measures the deviation from equality constraints. $\psi_{\A}$ is defined as $\psi_{\A} \defeq \max_{M \in \mathbb{A}_{\A}}\|M\|$, where $\|.\|$ is the spectral norm, $r(A_2)$ is the row rank of $A_2$ and $\mathbb{A}_{\A}$ is the set of $r(A_2) \times d$ matrices whose rows are linearly independent vectors chosen from the rows of $A_2$. Computing $\psi_{\A}$ is computationally expensive for high dimensional polytope. In such case we use an approximate $\psi_{\A}$, computed with a greedy algorithm. The parameter $\mu_{\A}$ is invariant to translation, rotation and scaling.

\paragraph{GCB-PE} In the concurrent work of \citet{chen_combinatorial_2020}, GCB-PE aims at solving the best-action problem for partial linear feedback. \citet{chen_combinatorial_2020} use a different notion of sample complexity, which is defined as a time $T$ such that with probability $1 - \delta$, the algorithm returns the correct answer before time $T$. In our work, the sample complexity is the expected stopping time of the algorithm, which is required to be correct with probability $1 - \delta$.

The correspondence between our notations and theirs is: $M_{x} = S_{A} \defeq \left( \1_{(\tilde{a} = a)}\right)_{\tilde{a} \in A, a \in [d]}$, $\Bar{r}(I, \theta) = \langle \bm{1}_{I}, \theta \rangle$, $L_{p}= \sqrt{\max_{I \in \mathcal{I}}|I|}$. Since our experiments consider BAI with semi-bandit feedback, we need to adapt the Algorithm 1 of \citet{chen_combinatorial_2020}. The sole modification is to consider $\hat{I} = \argmax_{I \in \mathcal{I}} \Bar{r}(I, \hat{\theta}(n))$ and $\hat{I}^{-} = \argmax_{I \in \mathcal{I} \setminus \{\hat{I}\}} \Bar{r}(I, \hat{\theta}(n))$ instead of $\hat{A}$ and $\hat{A}^{-}$. 

The computational complexity of GCB-PE is sensitive to the choice of the global observer set. This choice corresponds to the random covering initialization in our setting, $\sigma = B_{n_{0}}$. Based on $\sigma$, they define a constant $\beta_{\sigma}$ which is used for the stopping rule. Unfortunately, $\beta_{\sigma}$ is the solution of the following NP-hard binary quadratic program:
\begin{align*}
		\beta_{\sigma}^2 = \max_{(\eta_i)_{i \in [|\sigma|]} \in [-1,1]^{m_{\sigma}}} \left|\left| \frac{1}{N_{n_{0}}} \odot \sum_{i=1}^{|\sigma|} S_{A_i}^{\intercal} \eta_i\right|\right|_{2}^{2} = \max_{\eta \in \{-1,1\}^{m_{\sigma}}} \eta^{\intercal} P_{\sigma} \eta
\end{align*}
where $m_{\sigma} = \sum_{i=1}^{|\sigma|} |A_{i}| \gg d$, $P_{\sigma} = M_{\sigma} \text{diag}\left( \frac{1}{N_{n_{0}}^{2}}\right) M_{\sigma}^{\intercal}$, $M_{\sigma}^{\intercal} = \begin{bmatrix} S_{A_1}^{\intercal}  \hdots  S_{A_{|\sigma|}}^{\intercal} \end{bmatrix}$ and $\odot$ denotes the component-wise multiplication. To our knowledge, there is no efficient solver for this optimization.

In our experiments on GCB-PE we will compute $\beta_{\sigma}$ by testing the $2^{m_{\sigma}}$ possibilities. This restricts our results to small examples since the computational cost is increasing exponentially.

\subsection{Experimental Results} \label{appendix_experimental_setting_details}

\begin{figure}
\centering
\scalebox{0.6}{\begin{tikzpicture}[node distance=1.8cm,bend angle=40,auto]
  \tikzstyle{place}=[circle,thick,draw=black!75,minimum size=10mm]
  \tikzstyle{source}=[circle,thick,draw=green!75,fill=green!20,minimum size=10mm]
  \tikzstyle{sink}=[circle,thick,draw=red!75,fill=red!20,minimum size=10mm] size=6mm]  
 
  	\begin{scope}
  	\node [source] (e0) at (0,3.6) {$s$};
    \node [place] (n1) [right of=e0] {};
    \node [place] (n3) [right of=n1] {};
    \node [place] (n6) [right of=n3] {};
    \node [place] (n2) [below of=e0] {};
    \node [place] (n4) [below of=n1] {};
    \node [place] (n7) [below of=n3] {};
    \node [place] (n10) [below of=n6] {};
    \node [place] (n5) [below of=n2] {};
    \node [place] (n8) [below of=n4] {};
    \node [place] (n11) [below of=n7] {};
    \node [place] (n13) [below of=n10] {};
    \node [place] (n9) [below of=n5] {};
    \node [place] (n12) [below of=n8] {};
    \node [place] (n14) [below of=n11] {};
  	\node [sink] (e15) [below of=n13] {$t$};
      
    \draw[->]
        (e0) edge (n1) (e0) edge (n2) (n14) edge (e15) (n13) edge (e15);
    \draw[->]
        (n1) edge (n3) (n1) edge (n4) 
        (n2) edge (n4) (n2) edge (n5) 
        (n3) edge (n6) (n3) edge (n7)
        (n4) edge (n7) (n4) edge (n8)
        (n5) edge (n8) (n5) edge (n9)
        (n6) edge (n10) (n10) edge (n13) (n13) edge (e15) (n12) edge (n14) (n14) edge (e15)
        (n7) edge (n10) (n7) edge (n11)
        (n8) edge (n11) (n8) edge (n12)
        (n11) edge (n13) (n11) edge (n14);
    \draw[->, red, very thick] (n9) -- (n12) node[midway,above] {$a^*$};
    \end{scope}
    
    \begin{scope}[xshift=8cm]
    
    \node [source] (e1) at (0,0) {$s$};
    \node [place] (n1) [right of=e1] {};
    \node [place] (n2) [above of=n1] {};
    \node [place] (n3) [right of=n1] {};
    \node [place] (n4) [above of=n3] {};
    \node [place] (n5) [right of=n3] {};
    \node [place] (n6) [above of=n5] {};
    \node [place] (n7) [right of=n5] {};
    \node [place] (n8) [above of=n7] {};
  	\node [sink] (e2) [right of=n7] {$t$};
      
    \draw[->]
        (e1) edge (n1) (e1) edge (n2) (n7) edge (e2) (n8) edge (e2);
    \draw[->]
        (n1) edge (n3) (n1) edge (n4) (n2) edge (n3) 
        (n3) edge (n5) (n3) edge (n6) (n4) edge (n5) (n4) edge (n6)
        (n5) edge (n7) (n5) edge (n8) (n6) edge (n7) (n6) edge (n8);
    \draw[->, red, very thick] (n2) -- (n4) node[midway,above] {$a^*$};
    \end{scope}
\end{tikzpicture}}
    \caption{Paths examples: (a) grid network with $n_s = 6$ and (b) line network with $(n_n, n_l) = (2,4)$}
    \label{fig:network_examples}
\end{figure}
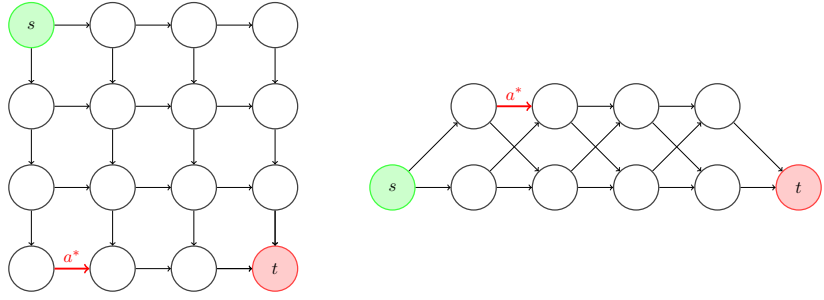

As illustrative examples we use the best-arm identification by sampling actions. The bandit is Gaussian, $\nu = \mathcal{N}(\mu, \sigma^{2} I_d)$. As regards the action set, we will consider:
\begin{itemize}
    \item uniform matroid, $\A = \left\{ A \subset [d]: |A| = k \right\}$, where the agent samples batches of size $k$. The batch setting is useful for real-world applications and admits an efficient oracle, the greedy algorithm.
    \item paths, $\A = \{A \subset [d]: A \in \text{path}(s,t,\mathcal{G})\}$, where the agent samples paths connecting $(s,t)$ in the graph $\mathcal{G}$. The path setting is omnipresent for network applications and admits efficient oracles, such as Dijkstra's algorithm. As a first illustrative example, we will consider a grid network with $n_s$ stages, also known as binomial bridges. A grid network with $n_s = 6$ is represented in Figure \ref{fig:network_examples}(a). Grid networks appear in real-world applications. They were also studied in \cite{kveton_tight_2015}. As a second illustrative example we will consider a line network with $n_l$ layers and redundancy $n_n$ (number of nodes per layer). A line network with $(n_n, n_l) = (2,4)$ is represented in Figure \ref{fig:network_examples}(b). Line networks appear in real-world applications. The redundancy ensures the system to be robust against failures.
    \item almost all sets, $\A = \left(\{I^*\} \cup \{A \in 2^d: I^* \not \subset A \}\right) \setminus \{\emptyset\}$, where the agent samples a set. This example is purely artificial. There is no efficient oracle. We designed it as an extreme needle-in-haystack problem where there is only one informative action among an exponential number of actions.
\end{itemize}

\begin{table}
\centering
\scalebox{0.85}{\begin{tabular}{|c || c | c| c| c| c|} 
 \hline
  &  $d$ & $|\mathcal{A}|$ & $|\mathcal{A}^*|$ & $\frac{|\mathcal{A}^*|}{|\mathcal{A}|}$  & $n_{0}$\\
 \hline\hline
 Uniform matroid & $d$ & $\binom{d}{k}$   &  $\binom{d-1}{k-1}$ & $\frac{k}{d}$ & $\lceil \frac{d}{k} \rceil$\\ 
 \hline
 Grid network &  $n_s(\frac{n_s}{2} + 1)$ & $\binom{n_s}{n_s/2}$   &  $1$ & $1 / \binom{n_s}{n_s/2} $& $n_s$ \\ 
 \hline
 Line network &  $2 n_{n} + (n_l - 1)n_n^2$ & $n_n^{n_l}$   &  $n_n^{n_l-2}$ & $1/n_n^2$ & $n_n^2$ \\ 
 \hline
 Almost all sets &  $d$  &  $2^{d-1}$   &  $1$ & $1/2^{d-1}$ & $2$\\ 
 \hline
\end{tabular}}
\caption{Central quantities}
\label{tab:central_quantities_for_examples}
\end{table}

The central quantities of interest are summarized in Table \ref{tab:central_quantities_for_examples}: the dimension $d$, the size of the action sets $|\A|$, the size of the informative action set (actions containing the best arm) $|\A^*|$ where $\A^* = \{A \subset [d]: I^* \subset A\}$, the ratio of informative actions $\frac{|\mathcal{A}^*|}{|\mathcal{A}|}$ and the minimal number of actions to perform a covering initialization $n_{0}$. Intuitively, the lower the ratio of informative actions is, the harder the problem is for naive algorithms. For example, uniform sampling fails drastically when $|\A^*|$ is low and $\sigma$ is high. When comparing learners on the simplex and the ones on the transformed simplex, the difference between the sizes of the respective initialization can have an important role, $|\A| - n_{0}$. The learners on $\tsimplex$ spend this additional budget on exploring relevant actions instead of merely sampling them all. The lower the noise, the more significant this difference is. In the no-noise setting, at most $n_{0}$ samples are necessary for the learners on the transformed simplex, while at most $|\A|$ samples are necessary for the ones on the simplex. The exact sample complexity depends on $|\A^*|$ and on the random draw of actions.

In the additional experiments, we will only compare AdaHedge and LLOO since they are the best instance in their family of learner (Figure \ref{fig:comparison_learners}).

\subsubsection{Uniform Matroid} \label{appendix_experimental_setting_details_uniform_matroid}

By increasing the dimension $d$, we observe the effect of an exponential increase of $|\A| = \binom{d}{k}$ while the ratio of informative actions $\frac{|\A^*|}{|\A|} = \frac{k}{d}$ is decreasing harmonically. Since $|\A^*| = \binom{d-1}{k-1}$ is also increasing with $d$ and $k$, we need to consider higher noise for $k=3$ than for $k=2$. Otherwise, the sampling rules using a full initialization will satisfy the stopping criterion before the end of the initialization.

For experiments on uniform matroids in Figures \ref{fig:comparison_learners} and \ref{fig:comparison_learners_uniform_matroid}, we consider $\mu^{(d)} \in \R^d$ for all $d \in \{5, 10, 15, 20, 25, 30, 35, 40, 45, 50\}$, such that $\mu^{(d)}_1 = 0.3$, $\mu^{(d)}_2 = 0.29$, $\mu^{(d)}_3 = 0.28$ and $\mu^{(d)}_{i} < 0.24$ for $i > 3$. Those values ensure that the best arm is always $I^*(\mu) = \{1\}$, while having two serious contenders $J \in \{\{2\}, \{3\}\}$. The rest of the arms are chosen ordered such as they are clearly suboptimal: $\mu^{(5)}_{4:5} = \{0.23, 0.2\}$, $\mu^{(10)}_{4:10} = \{0.232, 0.224, 0.207, 0.200, 0.192, 0.182, 0.176\}$, $\mu^{(15)}_{4:15} = \mu^{(10)}_{4:10} \cup \{0.214,  0.199, 0.195, 0.190, 0.164\}$, $\mu^{(20)}_{4:20} = \mu^{(15)}_{4:15} \cup \{0.185, 0.19, 0.195, 0.199, 0.214\}$, $\mu^{(25)}_{4:25} = \mu^{(20)}_{4:20} \cup \{0.158, 0.172, 0.211, 0.228, 0.244 \}$, $\mu^{(30)}_{4:30} = \mu^{(25)}_{4:25} \cup \{0.174, 0.18, 0.194, 0.202, 0.23, 0.242 \}$, $\mu^{(35)}_{4:35} = \mu^{(30)}_{4:30} \cup \{0.17, 0.178, 0.219, 0.222, 0.226 \}$, $\mu^{(40)}_{4:40} = \mu^{(35)}_{4:35} \cup \{0.197, 0.198, 0.201, 0.203, 0.205 \}$, $\mu^{(45)}_{4:45} = \mu^{(40)}_{4:40} \cup \{0.193, 0.206, 0.208, 0.21 \}$ and $\mu^{(50)}_{4:50} = \mu^{(45)}_{4:45} \cup \{0.188, 0.189, 0.191, 0.212, 0.213 \}$.

In Figures \ref{fig:comparison_learners_uniform_matroid}(a) and \ref{fig:comparison_learners_uniform_matroid}(b), we observe an identical behavior as in Figures \ref{fig:comparison_learners}(a) and \ref{fig:comparison_learners}(b). LLOO has competitive sample complexity for a low and almost constant computational cost compared to AdaHedge.\looseness=-1 

\begin{figure}
    \centering
    \includegraphics[scale=0.13]{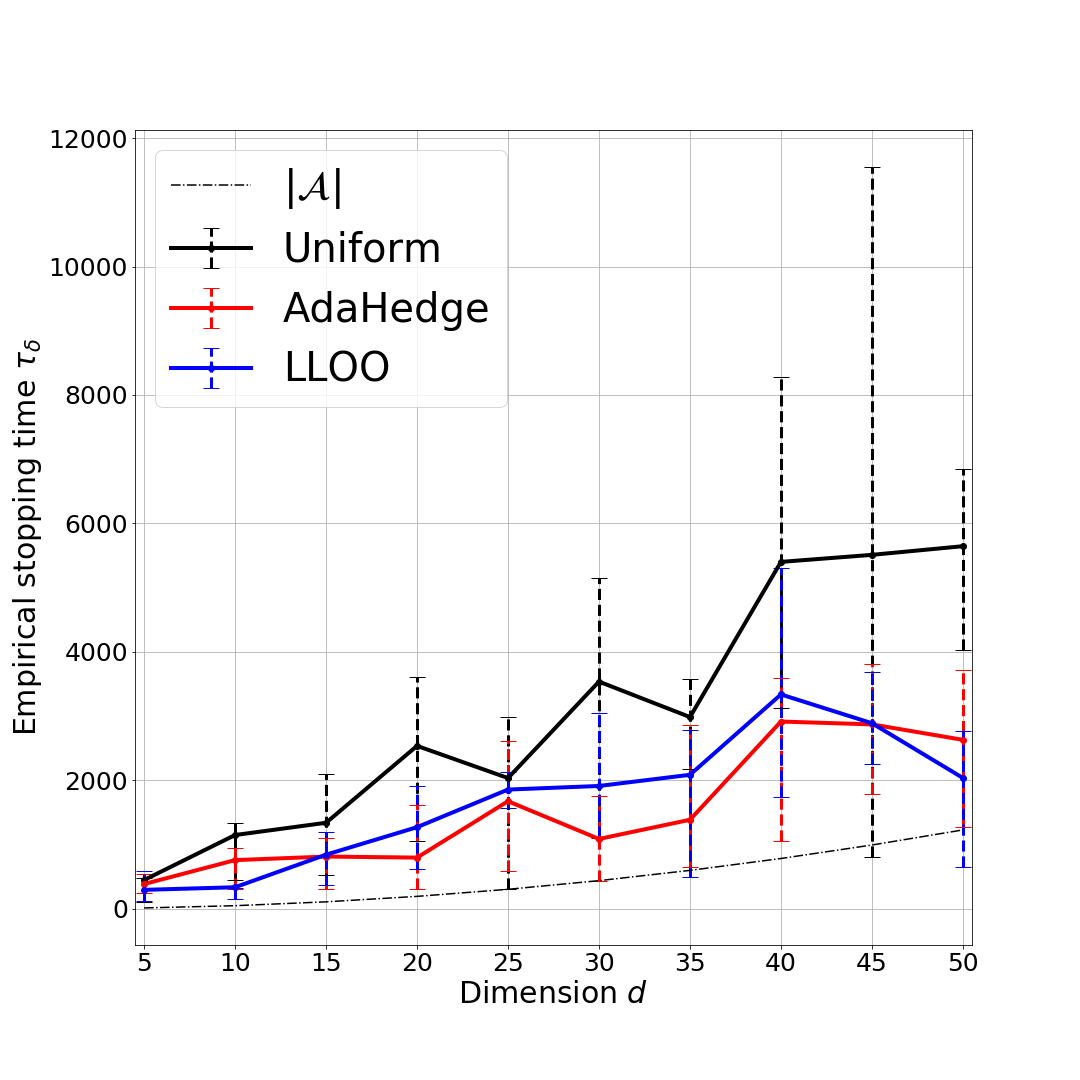}
    \includegraphics[scale=0.13]{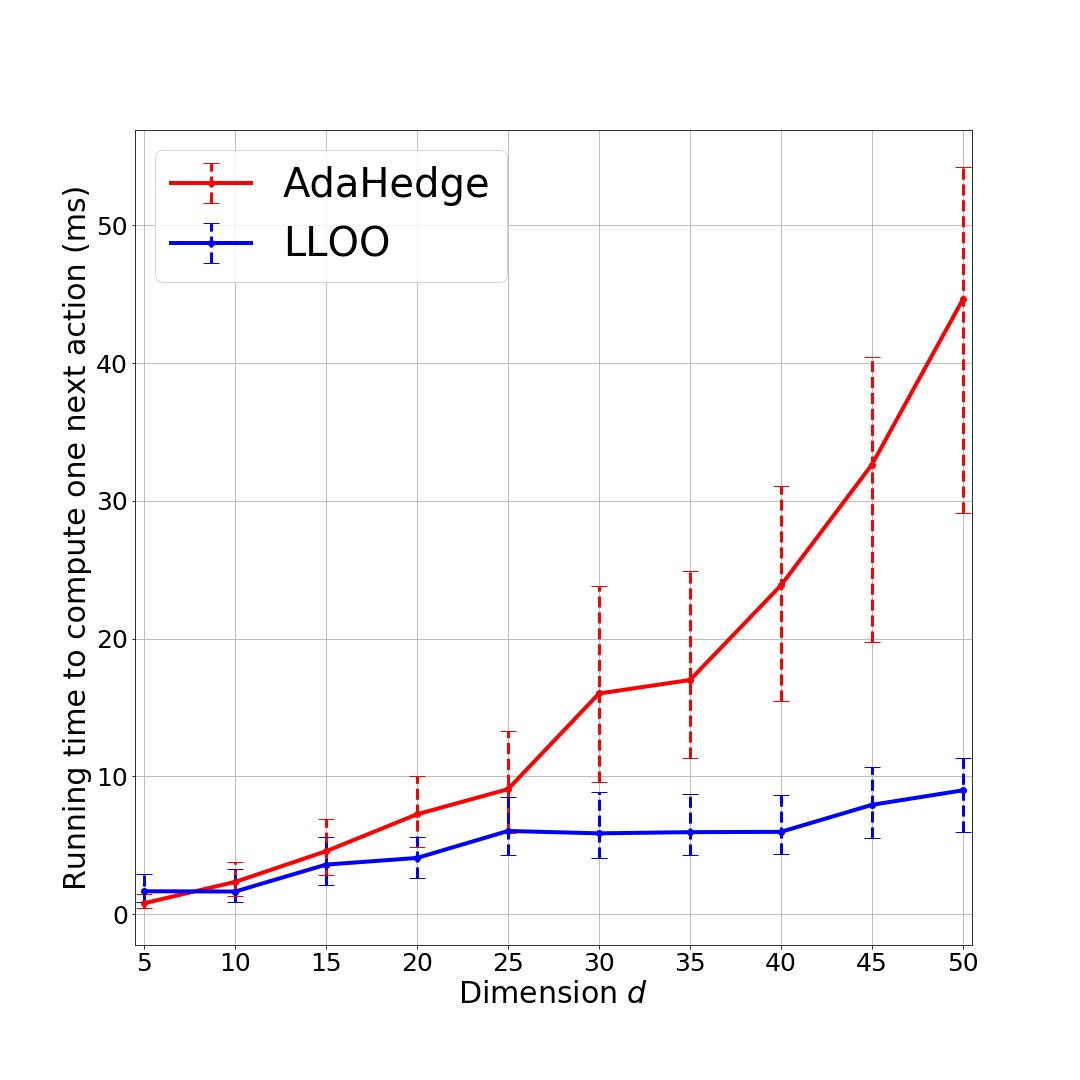}
    \includegraphics[scale=0.13]{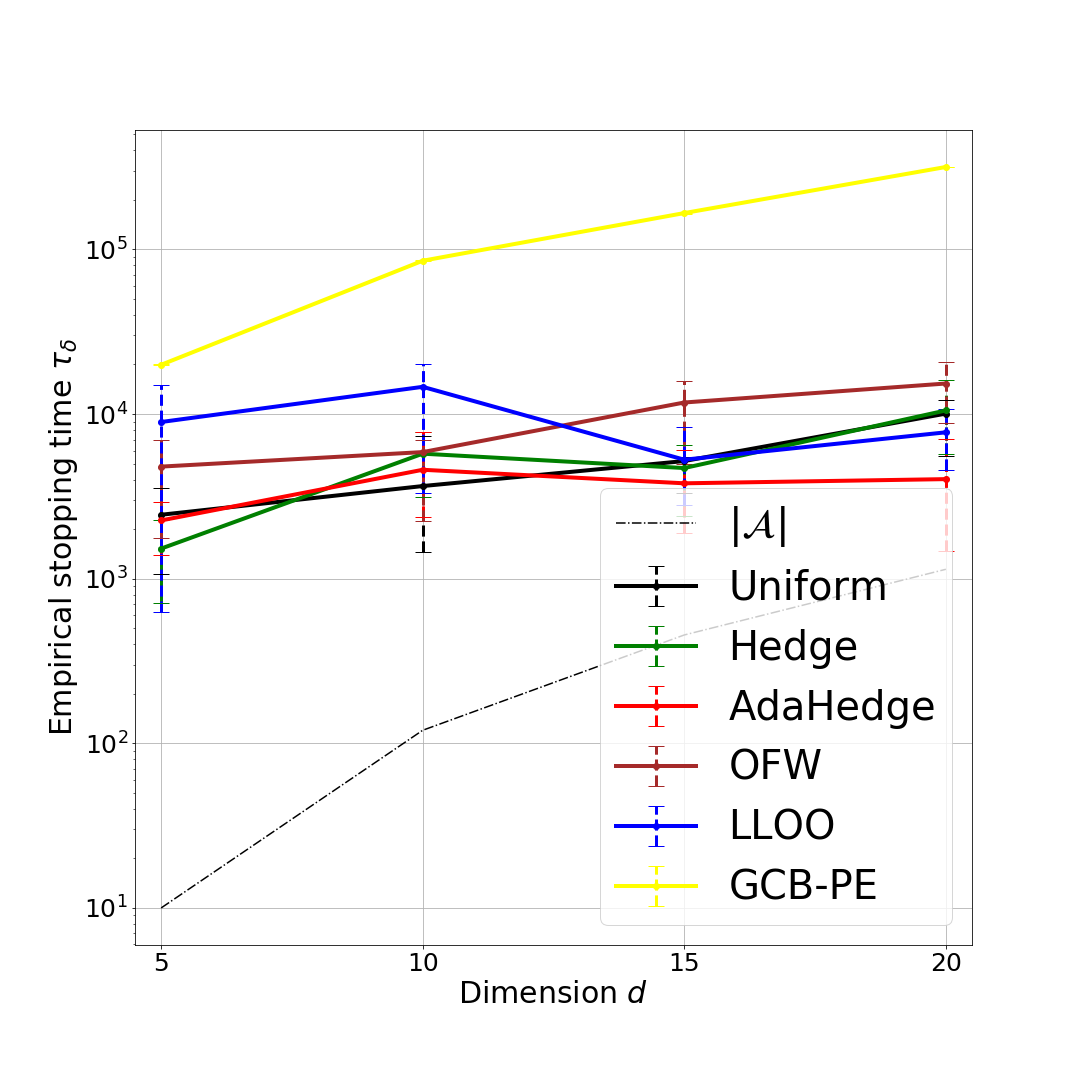}
    \caption{Uniform matroid, $k = 2$ (resp. $k=3$ for plot (c)), Gaussian bandit, $\nu = \mathcal{N}(\mu,\sigma^2 I_{d})$ with $\sigma = 0.035$ (resp. $\sigma = 0.1$). Influence of the dimension $d$ on: (a) (resp. (c)) the empirical stopping time $\tau_{\delta}$ and (b) the average running time to compute the next action.}
    \label{fig:comparison_learners_uniform_matroid}
\end{figure}

\subsubsection{Grid Network}

By increasing the number of stages $n_s$, we observe the effect of an exponential increase of $|\A| = \binom{n_s}{n_s/2}$ and an exponential decrease of $\frac{|\A^*|}{|\A|} = \frac{1}{\binom{n_s}{n_s/2}}$ since $|A^*| = 1$. 
 
For experiments on the grid networks in Figure \ref{fig:comparison_learners_grid_net}, we consider $\mu^{(n_l)} \in \R^d$ for all $n_s \in \{6, 8, 10, 12, 14,16\}$. The values for the parameters $\mu^{(n_l)}$ were obtained by random sampling with a Gaussian of mean $0.2$ and standard deviation $0.025$. After sorting, we increment $\mu^{(n_l)}_1$ by $0.025$ to ensure that $I^*(\mu) = \{1\}$ with a statistically significant gap.

The Figure \ref{fig:comparison_learners_grid_net}(a) highlights two important intuitive facts. First, the uniform sampling is highly inefficient in terms of samples when few informative actions are available, here $|A^*| = 1$. Second, the empirical performance of a learner on the simplex is limited by the initialization of size $n_{0} = |\A|$. The Figure \ref{fig:comparison_learners_grid_net}(b) highlights the lower computational cost of LLOO compared to AdaHedge. The slightly higher cost stems from the more expensive efficient oracle to solve the shortest path offline problem.

\begin{figure}
    \centering
    \includegraphics[scale=0.13]{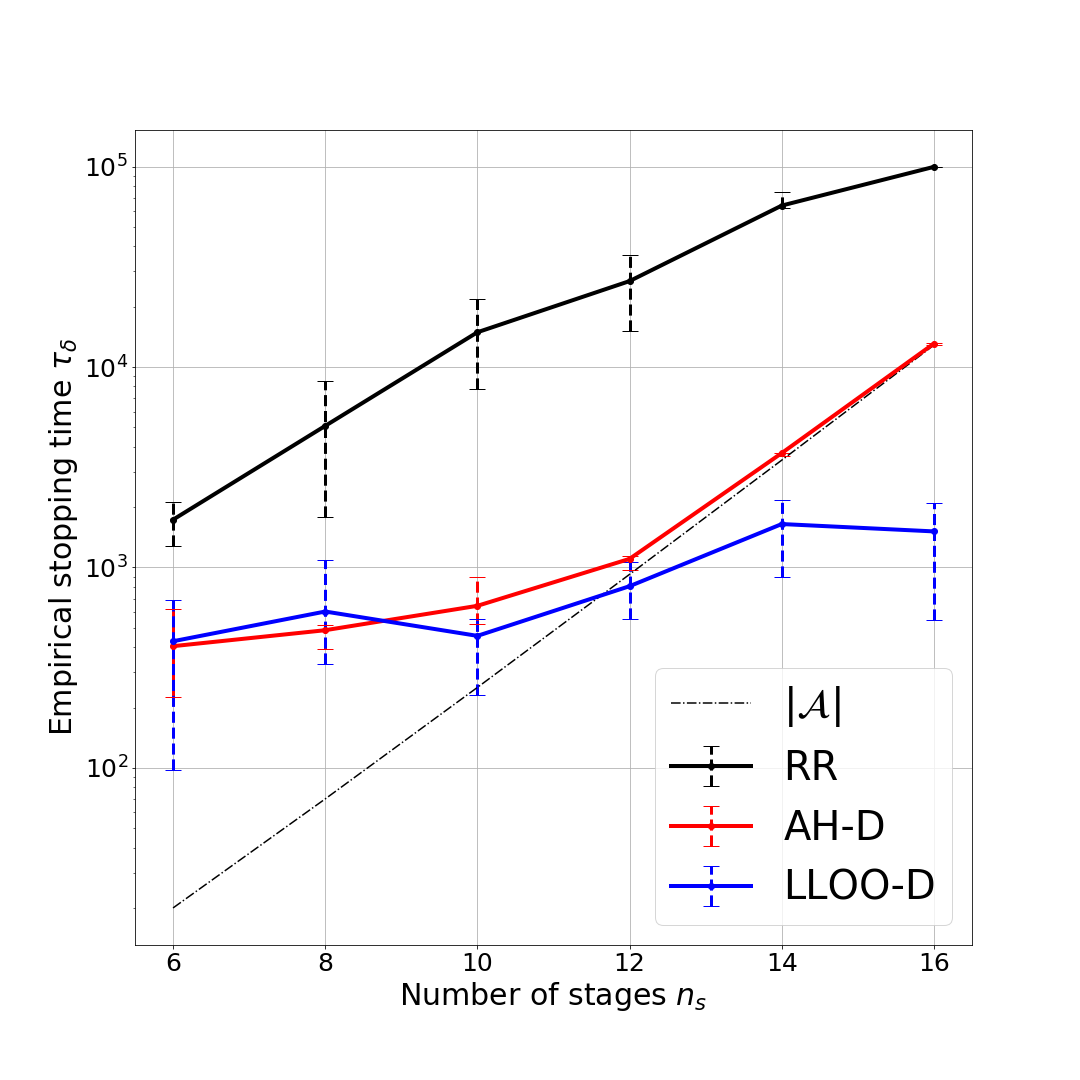}
    \includegraphics[scale=0.13]{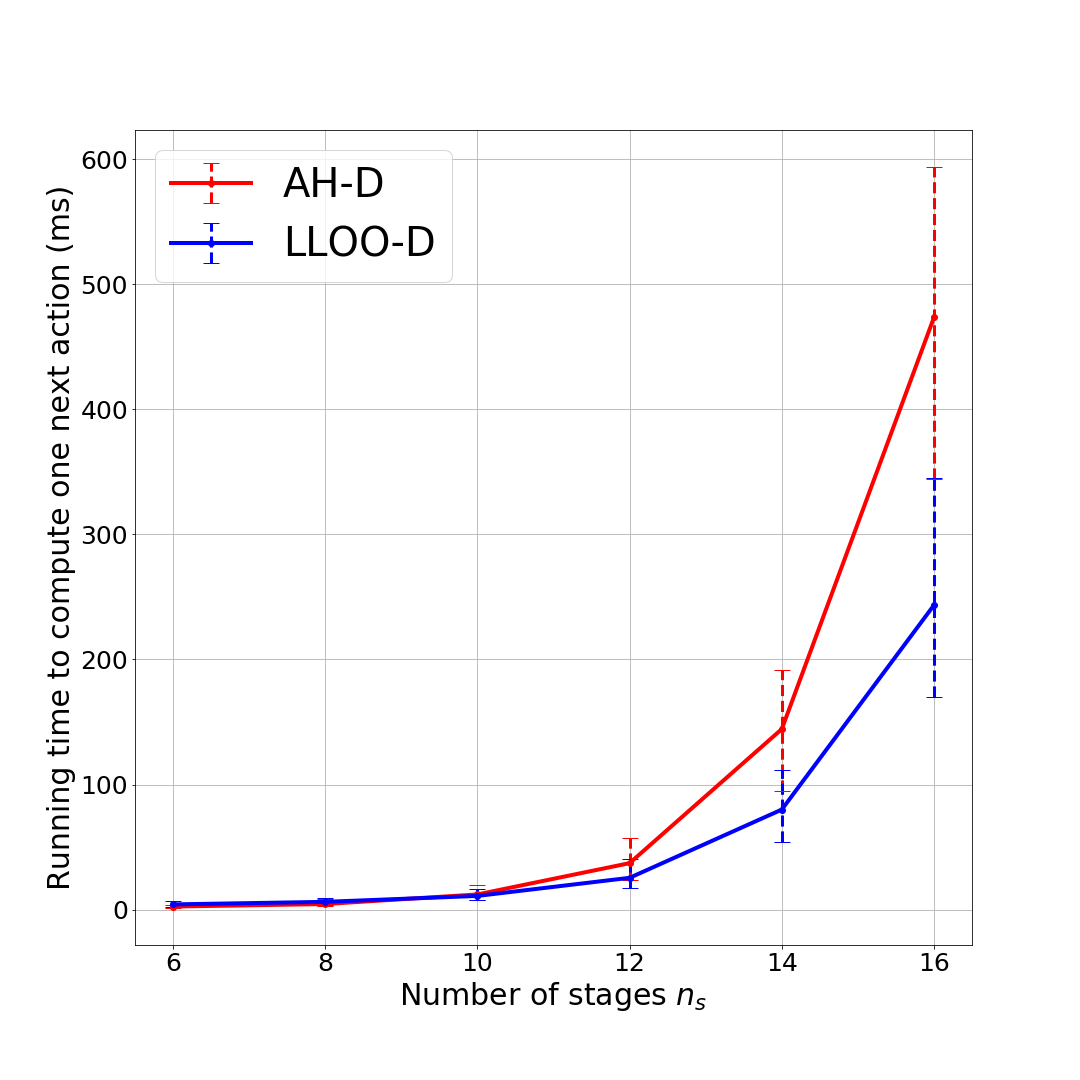}
    \caption{Grid network, Gaussian bandit, $\nu = \mathcal{N}(\mu,\sigma^2 I_{d})$ with $\sigma = 0.075$. Influence of the number of stages $n_s$ on: (a) the empirical stopping time $\tau_{\delta}$ and (b) the average running time to compute the next action.}
    \label{fig:comparison_learners_grid_net}
\end{figure}

\subsubsection{Line Network}

By increasing the number of layers $n_l$, we observe the effect of an exponential increase of $|\A| = n_n^{n_l}$ while the ratio of informative actions is decreasing as $\frac{|\A^*|}{|\A|} = \frac{1}{n_n^2}$. The number of informative actions $|\A^*| = n_n^{n_l-2}$ is also increasing with $n_l$, slowly for low $n_n$.

For experiments on the line networks in Figure \ref{fig:comparison_learners_line_net}, we consider $\mu^{(n_l)} \in \R^d$ for all $n_l \in \{5, \cdots ,12\}$ when $n_n = 2$ and $n_l \in \{4, \cdots ,8\}$ when $n_n = 3$. The values for the parameters $\mu^{(n_l)}$ were obtained by sampling randomly from a Gaussian with mean $0.2$ and standard deviation $0.025$. After sorting, we increment $\mu^{(n_l)}_1$ by $0.025$ to ensure that $I^*(\mu) = \{1\}$ with a statistically significant gap.

In Figure \ref{fig:comparison_learners_line_net}, the take-away message is similar as for uniform matroids. LLOO has competitive sample complexity for a low computational cost compared to AdaHedge.

\begin{figure}
    \centering
    \includegraphics[scale=0.13]{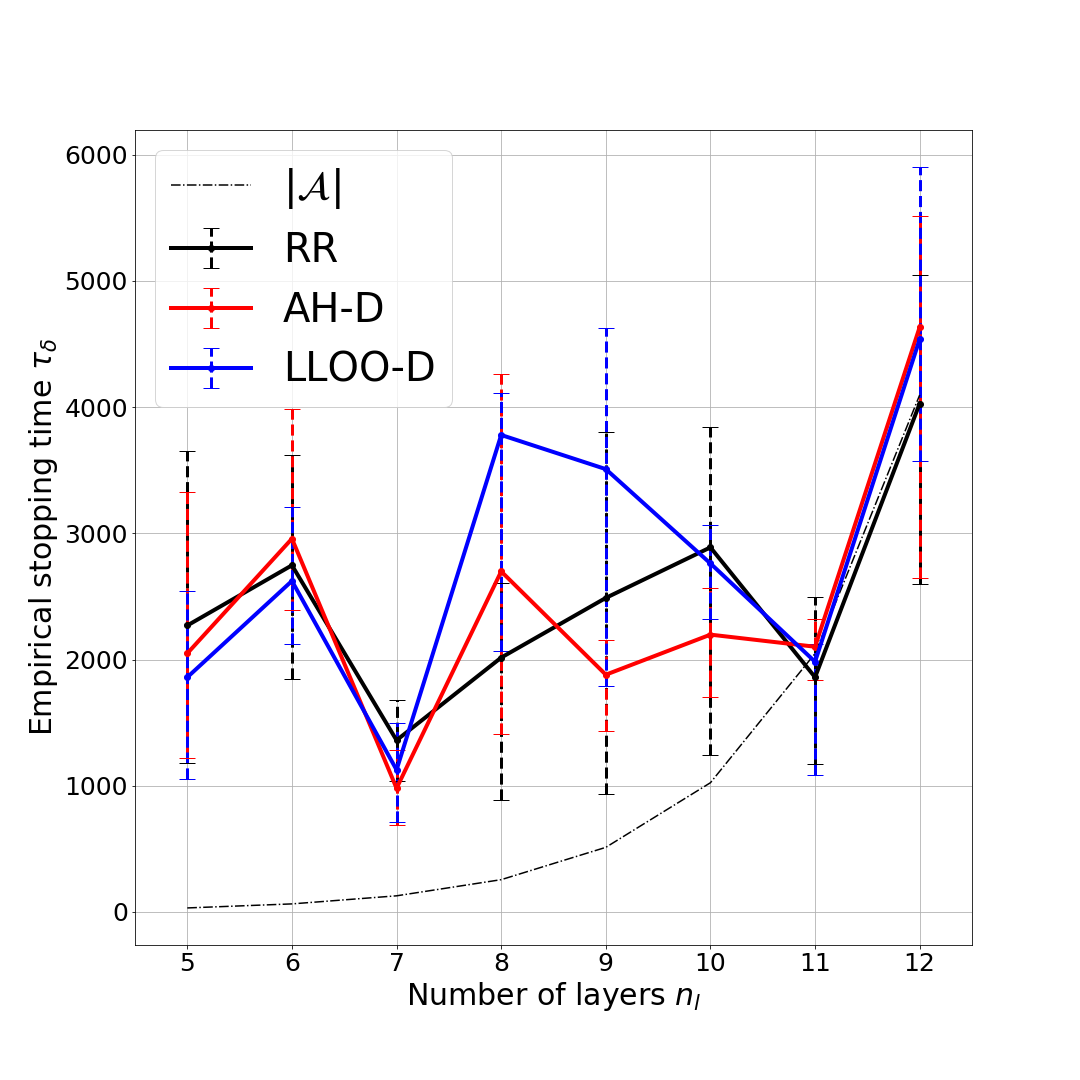}
    \includegraphics[scale=0.13]{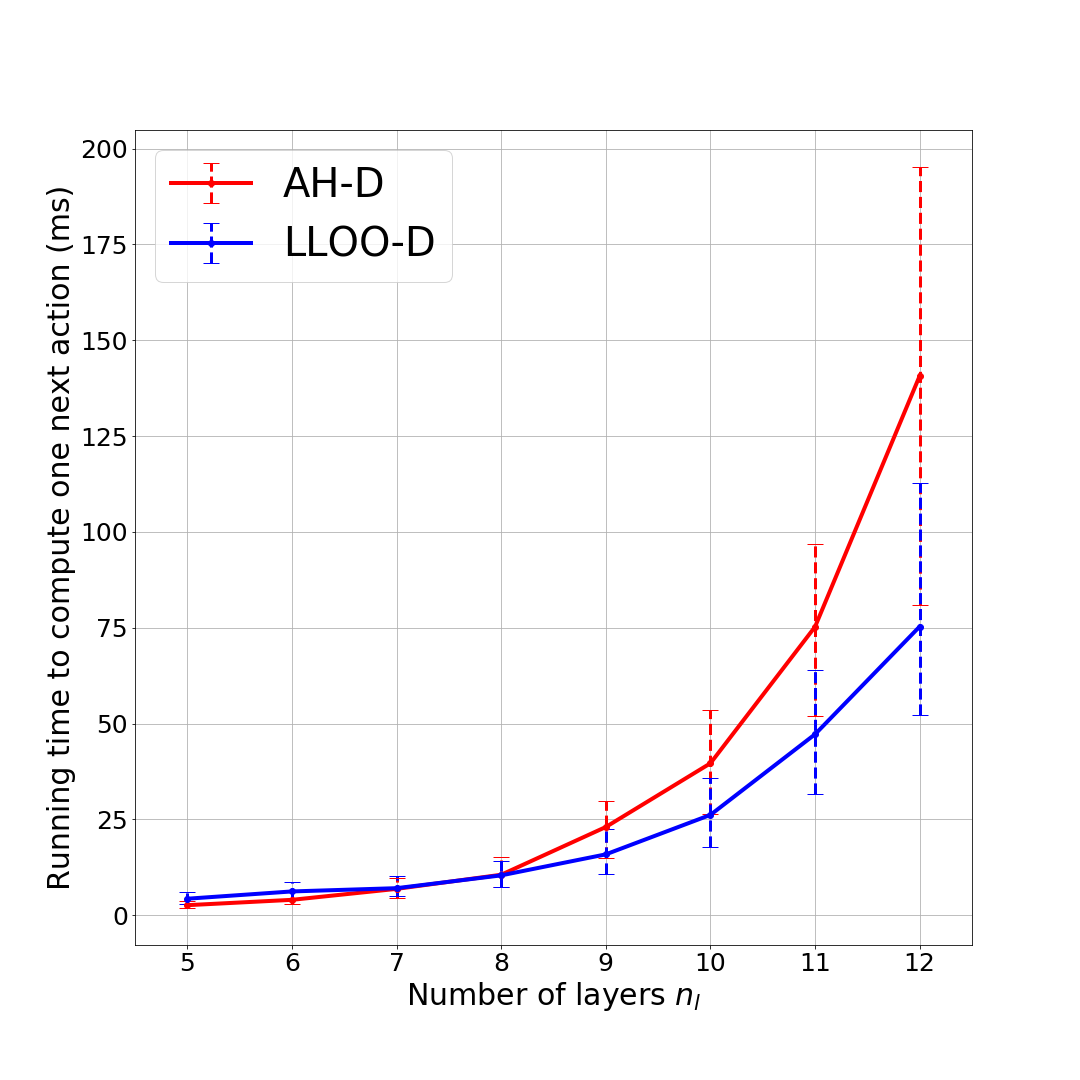}
    \\ 
    \includegraphics[scale=0.13]{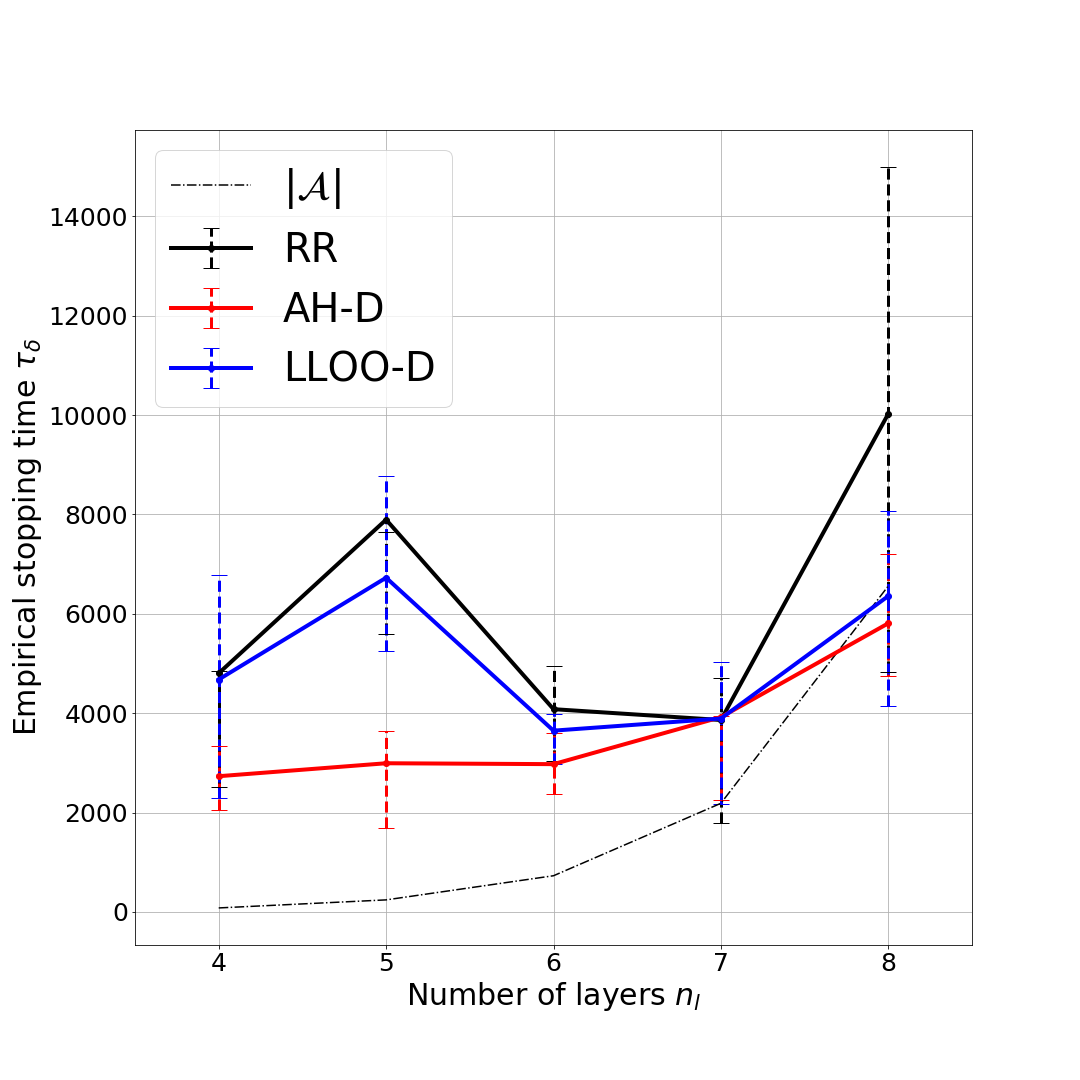}
    \includegraphics[scale=0.13]{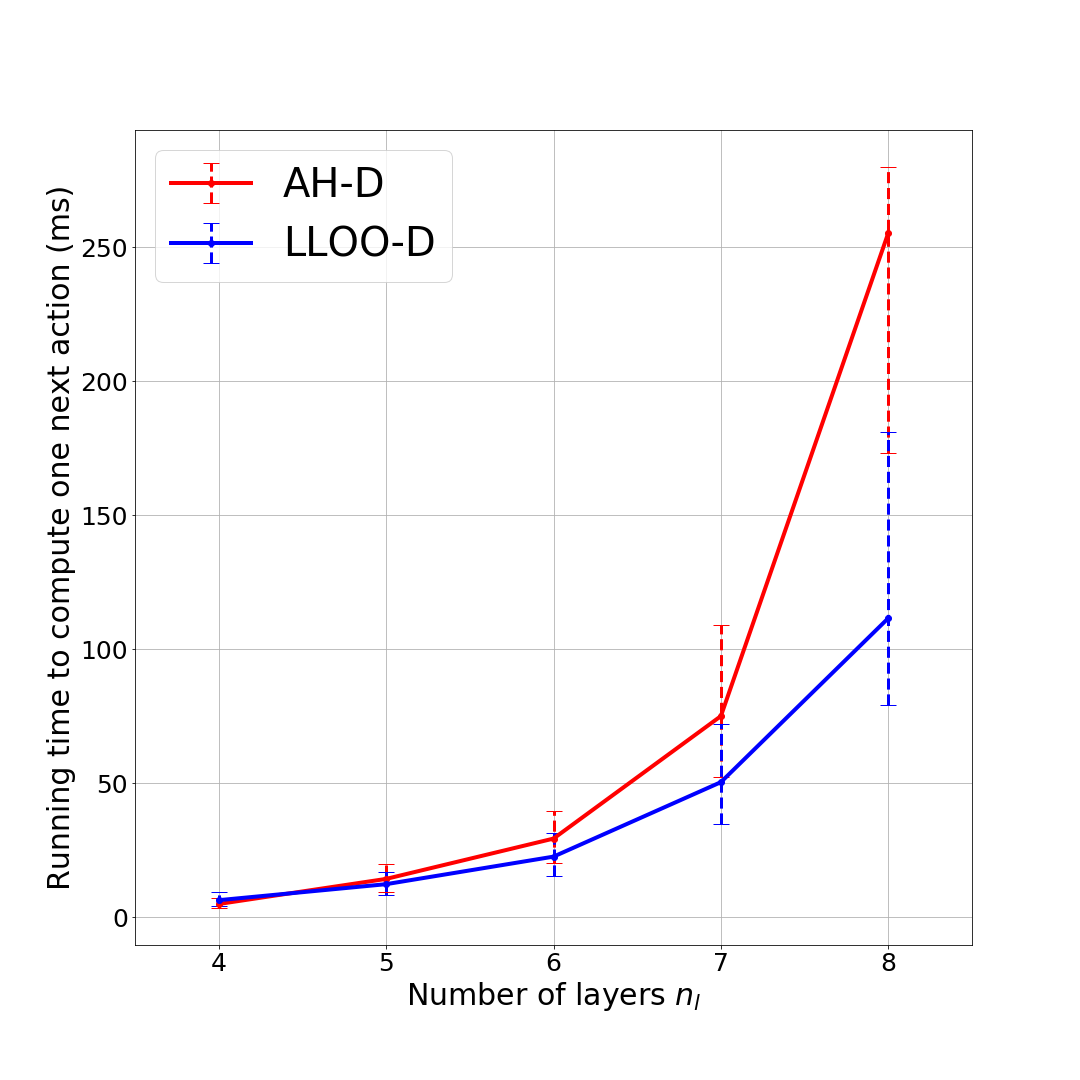}
    \caption{Line network, $n_n = 2$ (resp. $n_n = 3$ for the bottom plots), Gaussian bandit, $\nu = \mathcal{N}(\mu,\sigma^2 I_{d})$ with $\sigma = 0.2$. Influence of the number of layers $n_l$ on: (a) (resp. (c)) the empirical stopping time $\tau_{\delta}$ and (b) (resp. (d)) the average running time to compute the next action.}
    \label{fig:comparison_learners_line_net}
\end{figure}

\subsubsection{Almost all sets}

By increasing the dimension $d$, we observe the effect of an exponential increase of $|\A| = 2^{d-1}$ and an exponential decrease of $\frac{|\A^*|}{|\A|} = \frac{1}{2^{d-1}}$ since $|A^*| = 1$. 
 
For experiments on almost all sets in Figure \ref{fig:comparison_learners_low_ri}, we consider $\mu^{(d)} \in \R^d$ for all $d \in \{7, \cdots ,14\}$, such that $\mu^{(d)}_1 = 0.3$  and $\mu^{(d)}_{i} \leq 0.24$ for $i > 1$. Those values ensure that the best arm is always $I^*(\mu) = \{1\}$. The rest of the arms are chosen ordered such as they are clearly suboptimal: $\mu^{(7)}_{2:7} = \{0.24, 0.23, 0.22, 0.21, 0.2, 0.19\}$, $\mu^{(8)}_{2:8} = \mu^{(7)}_{2:7} \cup \{0.18\}$, $\mu^{(9)}_{2:9} = \mu^{(8)}_{2:8} \cup \{0.17\}$, $\mu^{(10)}_{2:10} = \mu^{(9)}_{2:9} \cup \{0.16\}$, $\mu^{(11)}_{2:11} = \mu^{(10)}_{2:10} \cup \{0.215\}$, $\mu^{(12)}_{2:12} = \mu^{(11)}_{2:11} \cup \{0.195\}$, $\mu^{(13)}_{2:13} = \mu^{(12)}_{2:12} \cup \{0.205\}$ and $\mu^{(14)}_{2:14} = \mu^{(13)}_{2:13} \cup \{0.185\}$.

In Figure \ref{fig:comparison_learners_low_ri}(a), the take-away message is similar as for the grid networks. Even though no efficient oracle exists, the computational cost of LLOO is still lower than the one of AdaHedge, see Figure \ref{fig:comparison_learners_low_ri}(b).

\begin{figure}
    \centering
    \includegraphics[scale=0.13]{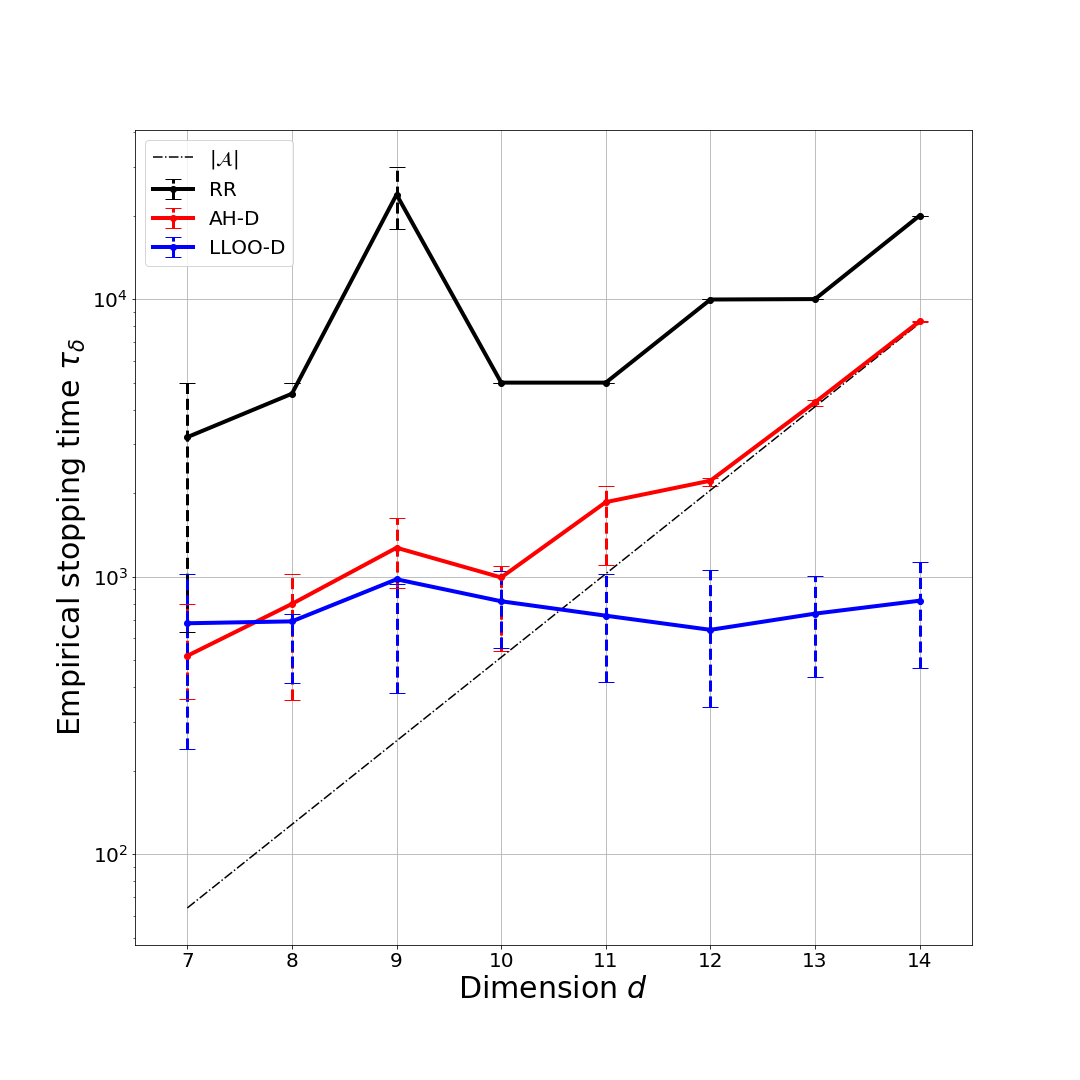}
    \includegraphics[scale=0.13]{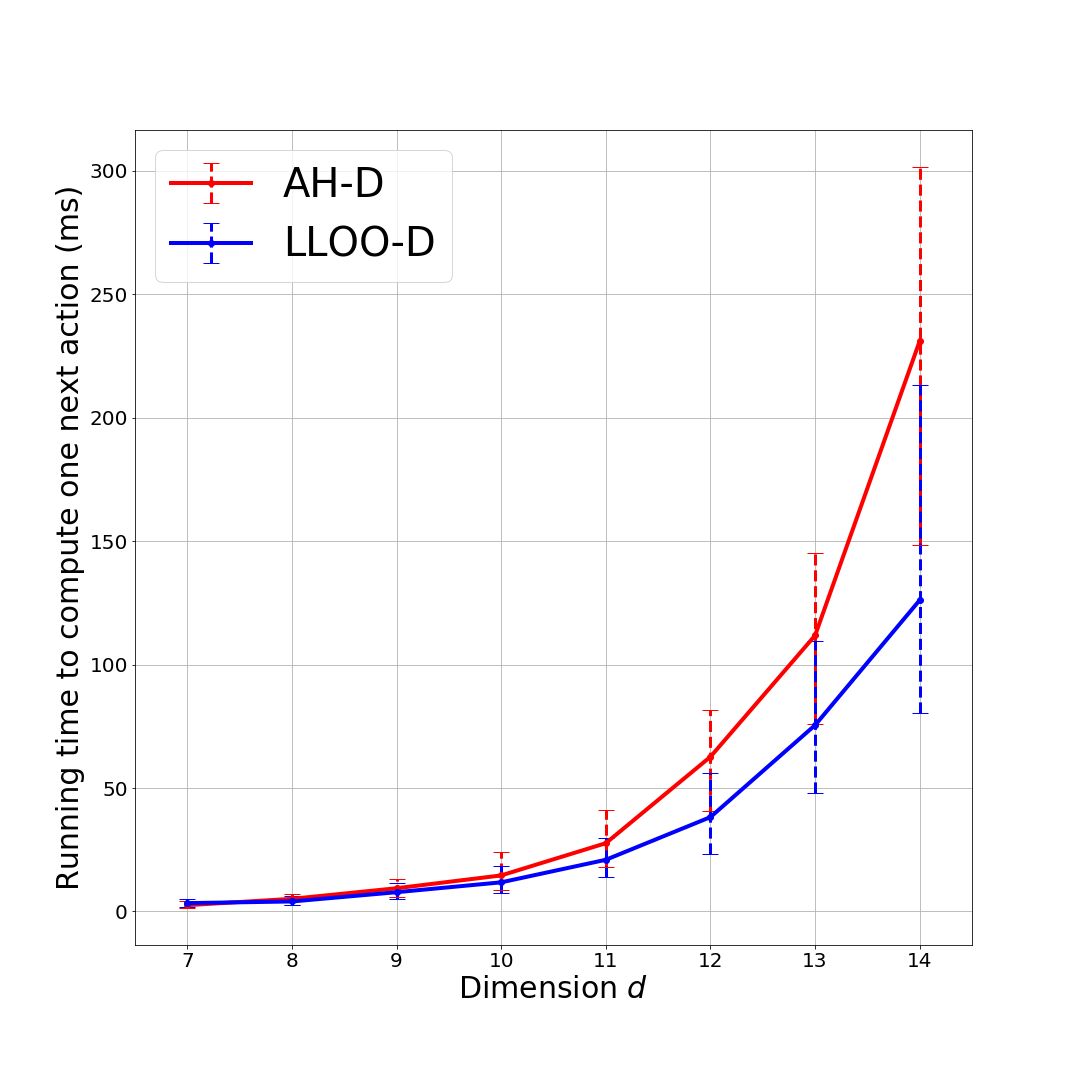}
    \caption{Almost all sets, Gaussian bandit, $\nu = \mathcal{N}(\mu,\sigma^2 I_{d})$ with $\sigma = 0.25$. Influence of the dimension $d$ on: (a) the empirical stopping time $\tau_{\delta}$ and (b) the average running time to compute the next action.}
    \label{fig:comparison_learners_low_ri}
\end{figure}

\end{document}